\documentclass[11pt]{article}
\usepackage[letterpaper,margin=1in]{geometry}
\usepackage{lmodern}
\usepackage[authoryear,round]{natbib}
\usepackage[T1]{fontenc}
\usepackage{microtype,amsmath,amssymb,amsthm,mathtools,mathrsfs}
\usepackage{booktabs,tabularx,longtable,xltabular,array,enumitem}
\usepackage{needspace}
\usepackage{placeins}
\usepackage{xcolor,graphicx}
\usepackage{hyperref,url}
\hypersetup{hidelinks}

\input{macros}

\newcommand{\iout}{{\rm out}}   %

\title{Where, When, and How Many Bits? \\
Scaling Laws for SGD with Stochastic Rounding}
\author{
\begin{minipage}{0.475\textwidth}
  \centering
  Zhichao Wang\\
  ICSI and Department of Statistics \\
  University of California, Berkeley, USA \\
  \texttt{zhichao.wang@berkeley.edu}
  \end{minipage}
  \begin{minipage}{0.475\textwidth}
    \centering
  Michael W. Mahoney \\
  ICSI, LBNL and Department of Statistics \\
  University of California, Berkeley, USA \\
  \texttt{mmahoney@stat.berkeley.edu}
  \end{minipage}
  \hfill
  }
\date{}

\begin{document}
\maketitle
\begin{abstract}
Low-precision training reduces computation but introduces additional rounding error.
We study how stochastic rounding (SR) affects prediction risk as model size,
training time and precision change. We consider a two-layer linear network
with a fixed random first layer, trained by mini-batch SGD on fresh Gaussian
data with power-law spectra. At each SGD step, the features used in forward and backward
computations are rounded independently. We prove matched upper and lower
bounds on prediction error for the final model and for uniform parameter
averaging, with general learning-rate and precision schedules chosen before training.
For fixed precisions and constant learning rates with parameter averaging,
we derive how the smallest prediction error decreases with compute and
identify model sizes and training times that achieve this optimal-compute scaling. Rounding only the forward features preserves the compute
scaling of exact arithmetic, while rounding backward features changes it.
Averaging parameters can change which computation needs higher
precision and when higher precision is useful. The laws also give bit
requirements for a given accuracy tolerance and show how backward rounding changes
learning-rate schedule design. At a fixed compute budget, lower precision can improve prediction when the benefit of a larger model or longer training outweighs
the additional SR-induced error. Model simulations illustrate these tradeoffs,
and experiments with an MLP and a small transformer examine rounding choices
beyond the theoretical model.
\end{abstract}

\section{Introduction}
\label{sec:intro}

Training large machine learning models in low precision can reduce memory usage and computational costs. For example,
mixed-precision training computes matrix multiplications of the forward and
backward passes in 16-bit arithmetic while keeping a high-precision copy
of the weights \citep{micikevicius2018mixed}. Recent LLMs are even pretrained
with 8-bit and 4-bit floating-point matrix products
\citep{micikevicius2022fp8,deepseek2024v3,tseng2025mxfp4,castro2025quartet,chmiel2025fp4,panferov2026quartet2}.
Meanwhile, quantization-aware training (QAT) simulates low-precision weights and
activations during training, making the trained model run in low
precision at inference \citep{jacob2018quantization,esser2020lsq,liu2024llmqat}.
A third setting saves training memory by quantizing the activations saved for backpropagation, while leaving the forward pass
unchanged \citep{chen2021actnn,liu2022gact}.
Separately, \citet{chen2020statistical} show extra gradient variance
caused by rounding in fully quantized training, where weights, activations and gradients are quantized.

In this paper, we focus on the first setting, in which the arithmetic of
training itself is rounded. Stochastic rounding (SR) rounds a number to
one of its neighboring representable values with probabilities that
preserve the input value in expectation. This unbiasedness can avoid update
stagnation at precisions where round-to-nearest stalls
\citep{gupta2015limited,zhang2017zipml,desa2018halp,zhao2024direct,ozkara2025sr}.
Existing SR theory provides probabilistic bounds on rounding errors in
numerical computations
\citep{higham2019new,higham2020sharper,connolly2021sr,croci2022sr,elarar2022sr,drineas2024stochastic,elarar2026survey}
and convergence guarantees for low-precision SGD
\citep{li2017training,desa2018halp,ozkara2025sr}.
Its benefit on small singular values has also been studied
\citep{boutsikas2024small,dexter2025stochastic,ma2026stochastic}.
These results do not further describe how prediction risk depends
jointly on model width, training time and precision, or how to allocate
precision bits across operations and training steps. So we aim to answer the following question in this paper:
\begin{center}
\emph{Where, when, and with how many bits should we use stochastic rounding during training?}
\end{center}

To answer this question, we prove a neural scaling law, which describes how loss asymptotically decreases with model size, training time and compute
\citep{kaplan2020scaling,hoffmann2022training,subramanian2026weather},
This law practically guides the allocation of compute resources during training. Recent empirical work extends these laws to low-precision training and model quantization \citep{kumar2024precision,sun2025floating,frantar2025compression,chen2025scaling}.
Typically, \citet{kumar2024precision,sun2025floating} empirically use fitted scaling laws to choose model size, training data and precision for a given compute budget.
These findings motivate a theoretical study of precision-dependent loss scaling and quantitative analyses of training dynamics with quantization \citep{zhang2025quantization,zhang2026precision}.

In the exact-arithmetic setting, theory of neural scaling law has been extensively studied \citep{bahri2024explaining,maloney2022solvable,bordelon2024dynamical,paquette2024phases,paquette2024homogenization,lin2024scaling,lin2025reuse,bordelon2026optimal,chen2026minibatch}.
Specifically, \citet{paquette2024phases} gives the baseline compute optimal scaling law. We follow these results and study a two-layer linear network with a fixed random first layer and train its output layer using online SGD with SR. We consider both the final model and a model formed by averaging parameters during training.
And the features used in the forward and backward computations are rounded independently, possibly at different precisions over the training. 
\citet{kim2026mixed} concurrently fit separate forward/backward precision effects and hardware-aware allocations, supported by a convex-optimization upper bound.
We theoretically derive binary-SR risk scaling laws and optimal compute allocations in low precision. 

\begin{figure}[htp]\centering
\includegraphics[width=.325\linewidth]{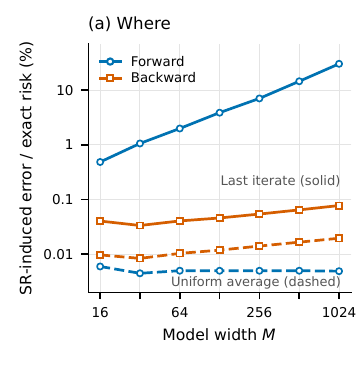}\hfill
\includegraphics[width=.325\linewidth]{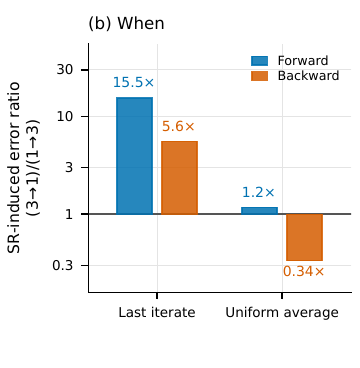}\hfill
\includegraphics[width=.325\linewidth]{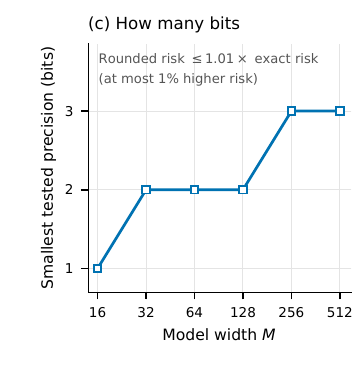}
\caption{\small
(a)~We round the features either at forward pass or backward pass in SGD dynamics \eqref{eq:sgd} with one-bit precision and compare the SR-induced error over the corresponding risk
without rounding. Solid lines evaluate the final model (last iterate); dashed lines evaluate the model obtained by averaging parameters over training
(uniform average). Forward rounding causes more error for last iterate, while backward causes more error for the average. 
(b)~Using three bits in the first half of training and one bit
in the second, or the reverse, with nearly equal average precision.
Each bar divides the SR-induced risk with higher precision early by that with
higher precision late. Values above one indicate higher precision at late phase of training outperforms the case with higher precision at early phase of training.
(c)~Each point selects the smallest tested
precision that keeps the final prediction risk no more than $1\%$ higher than
its value without rounding, at the same width and duration. Both locations
use the same precision from 1 to 4 bits. 
Appendix~\ref{app:experiments} gives experimental details.}
\label{fig:compute_law}
\end{figure}

\paragraph{Summary of main contributions.}
\begin{itemize}[leftmargin=*,itemsep=2pt,topsep=3pt]
\item \textbf{Neural scaling laws and compute scaling laws}
(Section~\ref{sec:dynamics}).
We prove general neural scaling laws, showing how prediction error changes with model size, training time and SR precision. For fixed precisions and constant learning rates with parameter averaging, we also derive how the smallest prediction error decreases with
compute and which model sizes and training times achieve this scaling under our
compute model.
\item \textbf{Where, when and how much to round}
(Sections~\ref{sec:where}--\ref{sec:howmuch}).
Using our scaling laws, we explain which computations need higher precision, when SR is useful during training, and how many needed bits meet an accuracy target. Figure~\ref{fig:compute_law}
illustrates how the answers change with different predictors, model size
and training time.
\item \textbf{Learning-rate design and compute allocation}
(Sections~\ref{sec:lr}--\ref{sec:compute-allocation}). Our scaling law covers general learning-rate and precision schedules.
We show how backward rounding changes learning-rate schedule design and compare
precision choices at a fixed compute budget. With a fixed compute budget, lower precision can improve prediction when the benefit of a larger model or longer training outweighs
the added rounding error; Figure~\ref{fig:frontier} illustrates this trade-off.
\end{itemize}
For more related work, see Appendix~\ref{app:related}. Appendices~\ref{app:actual-scaling}--\ref{app:rules} give the mathematical
foundations and proofs of our theorems. Appendix~\ref{app:experiments} shows the experimental details.

\section{Model setup}
\label{sec:setup}

\paragraph{Dataset and random features.}
We consider online SGD for Gaussian squared-loss regression with input dimension $d$ and feature width $M$.
For each sample $x\sim N(0,\Sigma_d)$ and label noise
$\epsilon\sim N(0,\sigma^2)$ with fixed $\sigma^2>0$, set the label to be
$y=x^\top\theta_*+\epsilon$.
Here $\Sigma_d=\diag(\lambda_j)\in\R^{d\times d}$ is the input covariance,
and $\theta_*\sim N(0,C_\theta)$ is a target parameter with covariance $C_\theta=\diag(\vartheta_j)\in\R^{d\times d}$.
We assume, for fixed constants $c_\lambda,C_\lambda,c_{\lambda\vartheta},C_{\lambda\vartheta}>0$,
\begin{equation}
\label{eq:population}
\begin{gathered}
d=\lceil c_dM\rceil,\quad c_d>1,\quad 1<b<a,\\
c_\lambda j^{-a}\le\lambda_j\le C_\lambda j^{-a},\qquad
c_{\lambda\vartheta}j^{-b}\le\lambda_j\vartheta_j
 \le C_{\lambda\vartheta}j^{-b}\quad(1\le j\le d).
\end{gathered}
\end{equation}
The exponents $a,b$ describe spectral and target-energy decay, relate to the source and capacity conditions \citep{paquette2024phases}. The constants in \eqref{eq:population} do not depend on $M$.
Let $S\in\R^{M\times d}$ be a random weight matrix with i.i.d. $N(0,1/M)$ entries,
drawn once and held fixed during training.
The sketch $S$, $\theta_*$, inputs $x$ and label noises are jointly independent.
The feature vector is $z=Sx$ and the predictor is the two-layer linear model
$v^\top z$. The population covariance of the feature vector $z$ and its largest eigenvalue are defined by
\begin{equation}
\label{eq:G}
 G=S\Sigma_dS^\top,\qquad g_1=\lambda_{\max}(G).
\end{equation}

\paragraph{Stochastic rounding (SR).}
For integer $p\ge0$ we use the binary grid
$\mathbb F_p=\{0\}\cup\{\pm2^e(1+k2^{-p}):e,k\in\mathbb Z,\ 0\le k<2^p\}$.
Here exponent $e$ is unbounded. For any $x\in\R$, let $\ell_p(x)$ and $u_p(x)$ be the neighboring grid points such that $\ell_p(x)<x<u_p(x)$. The stochastic rounding operator $Q_p(x)$ quantize $x$ to the upper point $u_p(x)$ with probability
$(x-\ell_p(x))/(u_p(x)-\ell_p(x))$ and the lower point $\ell_p(x)$ otherwise.
Grid points are fixed and rounding is unbiased: $\E[Q_p(x)\mid x]=x$. Besides, the variance of the rounding error is bounded by
\begin{equation}
\label{eq:cell}
\nu_p(x)=\E[(Q_p(x)-x)^2\mid x]
=(x-\ell_p(x))(u_p(x)-x)\le\tfrac14\,4^{-p}x^2.
\end{equation}
Denote the exact arithmetic by $Q_\infty(x)=x$. On vectors, $Q_p$ acts
coordinatewise independently.

\paragraph{The SGD training algorithm.}
Start from $v_0=0$ and take $T\ge1$ online SGD steps with fixed batch
size $B\ge1$ for mean squared loss of power-law random features. At update $u$, we draw $B$ fresh samples $\{x_{u\ell},y_{u\ell}\}_{\ell=1}^B$ and
\emph{independently} round features for prediction and gradient computation:
\begin{equation}
\label{eq:copies}
a_{u\ell}=Q_{p_{\fw,u}}(Sx_{u\ell}),\qquad
b_{u\ell}=Q_{p_{\bw,u}}(Sx_{u\ell}).
\end{equation}
Here $p_{\fw,u}$ and $p_{\bw,u}$ are the mantissa precisions of the two
independently rounded copies of $Sx_{u\ell}$.
The forward copy $a_{u\ell}$ is used to compute the prediction
$a_{u\ell}^\top v_u$; the backward copy $b_{u\ell}$ multiplies the prediction
residual in the gradient. We round only these two feature copies.
The SGD update is
\begin{equation}
\label{eq:sgd}
v_{u+1}=v_u-\frac{\eta_u}{B}\sum_{\ell=1}^B
 b_{u\ell}(a_{u\ell}^\top v_u-y_{u\ell}),\qquad 0\le u<T,
\end{equation}
where $\eta_u\ge0$ is the learning rate at step $u$. All other arithmetic is exact.
Write
\begin{equation}\label{eq:precision}
q_p=4^{-p},\qquad q_{r,u}=q_{p_{r,u}},\quad
r\in\{\fw,\bw\},\qquad q_\infty=0,
\end{equation}
\begin{equation}\label{eq:clock}
 \tau_t=\sum_{u<t}\eta_u,\qquad \rho_u=\eta_u/B,\qquad h_u=\eta_u^2/B.
\end{equation}
Here cumulative learning rate $\tau_t$ measures optimization time. We denote constant learning rate if $\eta_u=\eta$ for all $u$. In this case, define $\rho=\eta/B$ and $\tau_t=\eta t$.

\paragraph{Outputs and prediction risks.}
For SGD training in \eqref{eq:sgd}, we consider the last iterate and the uniform average outputs at final step $T$ as our predictors
\begin{equation}\label{eq:output-weights}
 v_T^{\rm last}=v_T,\qquad v_T^{\rm avg}=\frac1T\sum_{t=1}^Tv_t.
\end{equation}
Averaging can reduce SGD variance
\citep{polyak1992acceleration,yang2019swalp} in SGD dynamics.
Write $\|u\|_A^2=u^\top Au$. The population prediction excess risk can be written and simplified as
\begin{equation}
\label{eq:risk}
\mathcal E_S(v,\theta_*)=\frac{1}{2}\E_{x,\epsilon}[(y-v^\top Sx)^2|S,\theta_*]-\frac{\sigma^2}{2}=\tfrac12\|S^\top v-\theta_*\|_{\Sigma_d}^2.
\end{equation}
Evaluation uses exact arithmetic. 
Define its prediction risk and rounding-induced risk increase by
\begin{equation}
\label{eq:penalty}
\begin{aligned}
\mathscr R^{\iout}(T)&=
 \E[\mathcal E_S(v_T^{\iout},\theta_*)\mid S],\\
\Delta^{\iout}(T)&=
 \mathscr R^{\iout}(T)-\mathscr R_0^{\iout}(T),
\end{aligned}
\end{equation}
where index $\iout\in\{\mathrm{last},\mathrm{avg}\}$, expectation averages over the teacher, training data and SR,
conditional on $S$. The exact-arithmetic risk $\mathscr R_0^{\iout}(T)$ is the prediction risk with $p_{\fw,u},p_{\bw,u}=\infty$ for all $u$, under the same initialization, learning rates, batch size, time horizon and
output. 

\section{Scaling laws for stochastic-rounding SGD}
\label{sec:dynamics}\label{sec:limit}\label{sec:scaling}\label{sec:exact}

In this section, we give a neural scaling law of the prediction risks in \eqref{eq:penalty} for general learning-rate and precision
schedules, then evaluate it at constant learning rate. The law is decomposed into the error of exact-arithmetic training and the additional error due to
rounding. It applies to both outputs in \eqref{eq:output-weights}. The difference of their scaling laws
explain why they need different precision policies.

\subsection{Main result for general learning-rate and precision schedules}
\label{sec:general-schedules}

A rounding error introduced at update $u$ first affects $v_{u+1}$.
Its contribution to the final predictor depends on the subsequent
learning rates and the output rule. 
We introduce the following scalars to describe how rounding error at step $u$ contributes to prediction risk for last iterate and uniform average predictors. For $r=1,2$ and $0\le u<T$, define
\begin{align}
 K_r^{\rm last}(u)&=[1+2(\tau_T-\tau_{u+1})]^{\frac{1}{a}-r},\label{s3:eq:scalar-responses_last}\\
 K_r^{\rm avg}(u)&=\frac1{T^2}\sum_{v,w=u+1}^{T}
       [1+\tau_v+\tau_w-2\tau_{u+1}]^{\frac{1}{a}-r}.\label{s3:eq:scalar-responses_avg}
\end{align}
For the last iterate, an error introduced at update $u$ is propagated through
updates $u+1,\ldots,T-1$ before the final prediction. The cumulative learning
rate $\tau_T-\tau_{u+1}=\sum_{v=u+1}^{T-1}\eta_v$ measures the amount of
learning after that error occurs. For uniform averaging, the same error
contributes to every later iterate $v_{u+1},\ldots,v_T$. These contributions
are correlated because they originate from the same error. Squaring their
average therefore produces a term for every pair of later iterates, indexed
by $v,w$, and the factor $T^{-2}$ in \eqref{s3:eq:scalar-responses_avg}.
For exact-arithmetic training and $\iout\in\{\mathrm{last},\mathrm{avg}\}$, we define
\begin{equation}\label{s3:eq:scalar-reference}
\begin{aligned}
 \widehat B^{\rm last}&=\tfrac12(1+2\tau_T)^{-(b-1)/a},&
 \widehat B^{\rm avg}&=\frac1{2T^2}\sum_{v,w=1}^{T}
                              (1+\tau_v+\tau_w)^{-(b-1)/a},\\
 \widehat V^{\iout}&=\frac{\sigma^2}{2}
                    \sum_{u<T}h_uK_2^{\iout}(u),&
 h_u&=\eta_u^2/B,
\end{aligned}
\end{equation}
to characterize the bias and the variance from
minibatch sampling at step $T$.
For $0\le t\le T$, define
\begin{equation}\label{s3:eq:scalar-source}
 \widehat U_t=\frac{\tau_t^2}{(1+\tau_t)^{(a+b-1)/a}}
  +\sigma^2\sum_{u<t}h_u
       [1+2(\tau_t-\tau_{u+1})]^{1/a-1}.
\end{equation}
For $\iout\in\{\rm last,\rm avg\}$, define $\widehat N^{\iout}$ to control the rounding-induced risk increase:
\begin{equation}\label{s3:eq:scalar-rounding}
\begin{aligned}
 \widehat N^{\iout}=&\frac12\sum_{u<T}h_u\left[
  \frac{q_{\fw,u}\widehat U_u}{M}K_2^{\iout}(u)
  +\left(\frac{\sigma^2q_{\bw,u}}M+
      \frac{q_{\fw,u}q_{\bw,u}\widehat U_u}{M^2}\right)
                                           K_1^{\iout}(u)\right]\\
 &+\frac12\sum_{u<v<T}h_uh_v\left[
  \frac{\sigma^2q_{\bw,u}q_{\fw,v}}M K_2^{\iout}(v)+\frac{q_{\fw,u}q_{\bw,v}\widehat U_u}{M^2}
  [1+2(\tau_v-\tau_{u+1})]^{\frac{1}{a}-2}K_1^{\iout}(v)\right].
\end{aligned}
\end{equation}
For more explanation of these quantities, see Appendix~\ref{cl:finite}.

Fix $c_0\in(0,1)$. Appendix~\ref{app:forest} constructs positive thresholds
$\tau_*^{\rm sch}$ and $\rho_{\rm sch}$, depending only on the fixed model
quantities and $c_0$. For sufficiently large $M$, we require the learning rates
$\{\eta_u\}_{u=0}^{T}$ to satisfy
\begin{equation}\label{eq:scheduled-class}
 \tau_*^{\rm sch}\le\tau_T\le M^a,\qquad
 0\le\eta_ug_1\le c_0<1,\qquad \eta_u/B\le\rho_{\rm sch}.
\end{equation}
The thresholds
are independent of width, time horizon, batch size, schedules and the realized
sketch. The learning-rate and precision schedules are deterministically
chosen before the SR.

The following theorem gives matching upper and lower bounds for the error
already present in exact-arithmetic SGD and the additional error caused by
SR. It covers both the final iterate and uniform parameter averaging,
allowing learning rates and precisions to change during training.

\begin{theorem}[Scaling law for general SR and learning-rate schedules]
\label{thm:scheduled-main}
Recall \eqref{eq:penalty}. Consider any deterministic learning-rate schedule $\{\eta_u\}_{u=0}^{T}$ satisfying \eqref{eq:scheduled-class} and precision schedule $\{p_{\fw,u},p_{\bw,u}\}_{u=0}^{T}$. Under Assumption~\ref{as:assumption-schedules}, for each $D>0$, there exist some constants $c,K_D>0$ such that, for sufficiently large width $M$, with probability at least $1-6e^{-cM}-K_D M^{-D}$, the following holds. For $\iout\in\{\mathrm{last},\mathrm{avg}\}$, the prediction risk with SR at time $T$, $\mathscr R^{\iout}(T)=\mathscr R_0^{\iout}(T)+\Delta^{\iout}(T)$, satisfies
\begin{equation}\label{eq:scheduled-law}
\begin{aligned}
 \mathscr R_0^{\iout}(T)&\asymp
            M^{1-b}+\widehat B^{\iout}+\widehat V^{\iout}, \qquad
 \Delta^{\iout}(T)\asymp\widehat N^{\iout},
\end{aligned}
\end{equation}
where the constants are independent of width $M$, time horizon $T$, batch $B$, learning-rate and precision schedules, and $S$. 
\end{theorem}
Assumption~\ref{as:assumption-schedules} restricts the learning-rate and precision schedules so that earlier rounding errors do not excessively amplify the variance of later rounding errors. The proof of Theorem~\ref{thm:scheduled-main} mainly uses second-moment recursions in \eqref{eq:moment-sources} along the training trajectory of \eqref{eq:sgd} and is given in Appendix~\ref{cl:finite}.

\subsection{Special case: constant learning rate and precisions}\label{sec:constant-schedule}

We now specialize Theorem~\ref{thm:scheduled-main} to constant learning rates and precisions.
Take $\eta_u=\eta$, $p_{\fw,u}=p_{\fw}$, and $p_{\bw,u}=p_{\bw}$ throughout the training in \eqref{eq:sgd} for all $u=0,\dots,T$. At step $T$, $\rho_T =\eta/B = \rho$ and $\tau_T=\eta T =\tau$. Then  Theorem~\ref{thm:scheduled-main} implies that
\begin{align}
 \mathscr R_0^{\rm last}(T)&\asymp
          M^{1-b}+\tau^{-(b-1)/a}+\rho s_M,\label{eq:clean-law-last}\\
 \mathscr R_0^{\rm avg}(T)&\asymp
          M^{1-b}+\tau^{-(b-1)/a}
                    +\rho s_M\tau^{-(a-1)/a},\label{eq:clean-law-avg}
 \qquad s_M=\sigma^2+M^{1-b},
\end{align}
and their additional error induced by SR training are given by
\begin{align}
 \Delta^{\rm last}(T)\asymp{}&
  \frac{\rho q_\fw}{M}\tau^{\frac{a-b+1}{a}}
  +\frac{\rho q_\bw}{M}
                (s_M+\tau^{-\frac{b-1}{a}})\tau^{\frac{1}{a}}+\frac{\rho^2q_\fw q_\bw}{M}
                (s_M+\tau^{-\frac{b-1}{a}})\tau,\label{eq:law-last}\\
 \Delta^{\rm avg}(T)\asymp{}&
  \frac{\rho q_\fw}{M}\tau^{\frac{2-b}{a}}
  +\frac{\rho q_\bw}{M}
                (s_M+\tau^{-\frac{b-1}{a}})\tau^{\frac{1}{a}}
  +\frac{\rho q_\fw q_\bw}{M^2}\tau^{\frac{a-b+2}{a}}.\label{eq:law-avg}
\end{align}
For full statements and proofs, see Corollaries~\ref{ca:constant-reference} and~\ref{ca:constant-corollary}. In \eqref{eq:law-last} and \eqref{eq:law-avg} the terms in each SR-induced error are forward, backward and joint rounding.
Exact forward or backward arithmetic sets $q_\fw=0$ or $q_\bw=0$,
respectively.
Full risk $\mathscr R^{\rm last}(T)$ and $\mathscr R^{\rm avg}(T)$ are the sums of the corresponding terms in \eqref{eq:clean-law-last}, \eqref{eq:clean-law-avg} and \eqref{eq:law-last}, \eqref{eq:law-avg}.
As a remark, in this special case, a sufficient condition for Assumption~\ref{as:assumption-schedules} in Theorem~\ref{thm:scheduled-main} is, for some constant $\kappa_{\rm c}>0$,
\begin{equation}\label{s3:eq:constant-gain}
 \rho q_\fw q_\bw\tau/M\le\kappa_{\rm c}.
\end{equation}

\paragraph{How rounding changes the learning curve.}
For growing width $M$ and optimization time $\tau$, prediction risk $\mathscr R^{\rm avg}(T)$ can approach zero even when SR makes it decrease more slowly than its exact-arithmetic counterpart.
At fixed positive $\rho,\sigma^2$, with $\tau\to\infty$ and $\tau\le M^a$ from \eqref{eq:scheduled-class}, the exact-arithmetic risk decays as
$\mathscr R_0^{\rm avg}(T)\asymp\tau^{-(b-1)/a}$.
The last iterate instead retains an SGD-noise floor of order
$\rho\sigma^2$, so increasing width and training time alone cannot make
its risk vanish at fixed $\rho$.

For the average iterate, preserving the exact-arithmetic risk order requires
$\Delta^{\rm avg}(T)=O(\mathscr R_0^{\rm avg}(T))$; rounding error becomes negligible when
$\Delta^{\rm avg}(T)=o(\mathscr R_0^{\rm avg}(T))$.
With forward arithmetic exact and fixed finite backward precision,
one can easily check $\tau=o(M^a)$ implies \eqref{eq:law-avg} vanishing, whereas preserving the order of $\mathscr R_0^{\rm avg}(T)$ requires $\tau=O(M^{a/b})$.
Thus SR-induced excess risk in \eqref{eq:law-avg} can determine the decay rate even while $\mathscr R^{\rm avg}(T)=o(1)$. On the other hand, with SR for both forward and backward passes,
$\tau \asymp M$ with a sufficiently small proportionality constant
satisfies \eqref{s3:eq:constant-gain} which directly implies that the risk $\mathscr R^{\rm avg}(T)$ vanishes
and rounding error is negligible. Along this family, risk of the uniformly averaged predictor tends to zero as width and training time increase, while both mantissa precisions remain fixed.

\subsection{Compute-optimal scaling law with precision}
\label{sec:compute-scaling}\label{sec:precision-compute}

We now use the constant-learning-rate laws to determine how model width
and training duration should grow with a compute budget. We focus on the
uniformly averaged predictor.
We charge the forward and gradient feature computations against a compute budget $F$:
\begin{equation}\label{eq:declared-compute}
 f:=BM\sum_{u<T}[c(p_{\fw,u})+c(p_{\bw,u})]\le F,\qquad
 c(p)=\begin{cases}(p+p_0)/(p_{\rm ref}+p_0),&p<\infty,\\
 1,&p=\infty.
 \end{cases}
\end{equation}
Here sign/exponent uses $p_0>0$ bits and $p_{\rm ref}$ is a fixed large
mantissa bit count for normalization.
Each forward or gradient feature computation at precision $p$ costs
$BM\,c(p)$, so fewer bits lower the charge for the same updates. Empirical
scaling studies also include precision-dependent training costs
\citep{kumar2024precision,sun2025floating,kim2026mixed}. We use the above \eqref{eq:declared-compute} to model the compute cost of SGD training in \eqref{eq:sgd} with different precisions for simplicity. 

Consider the setup of Section~\ref{sec:constant-schedule} with constant $\eta,B$ and fixed precision pair $(p_\fw,p_\bw)$, independently of the budget $F$. At each budget, we choose width
$M$ and training duration $T$ to minimize the prediction risk $\mathscr R^{\rm avg}(T)$, subject to $f\le F$ and the assumptions of
Theorem~\ref{thm:scheduled-main}.
Table~\ref{tab:compute-precision} presents compute-optimal scaling laws (the order of the smallest achievable risk), and the orders of $M$ and $\tau=\eta T$ attaining it. The formal statements and proofs are given in Proposition~\ref{app:prop-fixed-bit-compute}. Besides, Proposition~\ref{r4:prop-variable-rho-average} justifies the same scaling laws but allowing $\eta,B$ to vary across budgets. Proposition~\ref{r4:prop-variable-rho-last} addresses similar laws for the last iterate predictor.

\begin{table}[htbp]
\centering\small
\caption{Compute-optimal prediction risk $\mathscr R^{\rm avg}(T)$ 
with fixed precisions, $\eta,B$, and compute budget $F$.
Precisions $p_\fw,p_\bw$ may differ when they are finite. Columns of $M$ and $\tau$ show the orders attaining the compute-optimal risk. Here we only focus on the scaling order as $F\to\infty$. $1<b<a$.}
\label{tab:compute-precision}
\begin{tabular}{@{}l|ccc@{}}
\toprule
Feature precision (forward, backward) & Risk & Width $M$ & Optimization time $\tau$\\
\midrule
Exact, exact
 & $F^{-\frac{b-1}{a+1}}$ & $F^{\frac1{a+1}}$ & $F^{\frac a{a+1}}$\\
Finite, exact
 & $F^{-\frac{b-1}{a+1}}$ & $F^{\frac1{a+1}}$ & $F^{\frac a{a+1}}$\\
Exact, finite
 & $F^{-\frac{b-1}{a+b}}$ & $F^{\frac b{a+b}}$ & $F^{\frac a{a+b}}$\\
Finite, finite
 & $F^{-\frac{b-1}{2a}}$ & $F^{1/2}$ & $F^{1/2}$\\
\bottomrule
\end{tabular}
\end{table}

Table~\ref{tab:compute-precision} shows that rounding location changes
both the risk exponent and the optimal allocations of width and training
time. For $v_T^{\rm avg}$, fixed finite forward precision preserves
the exact-arithmetic exponent when backward features remain exact with $\tau\asymp M^a$. However, the attaining allocation for backward rounding only uses $\tau\asymp M^{a/b}$; both finite precisions case have $\tau\asymp M$. These allocations use progressively more width
for a given optimization time, but their risk decays
successively more slowly with compute. Low precision therefore need not create an
accuracy floor when $M$ and training time increase together.
The exponents alone in Table~\ref{tab:compute-precision} do not determine which configuration has lower risk at a fixed budget: precision affects both exponents and risk coefficients. 
Section~\ref{sec:compute-allocation} develops a more detailed analysis of compute allocation mechanisms and finite-budget comparisons.

\section{Implications for low-precision training}
\label{sec:answers}
In this section, we apply scaling laws in Section~\ref{sec:dynamics} to answer three precision questions: where rounding is
least harmful, when training needs higher precision, and how many bits
meet an accuracy target. We then specifically consider learning-rate design and
compute allocation. Appendix~\ref{app:rules} gives the full
statements and proofs of the claims in this section.

\subsection{Where: rounding location for different predictors}
\label{sec:where}

Which location in training updates \eqref{eq:sgd} needs more precision? We compare two different SGD dynamics with different rounding locations: forward-only and backward-only SR in \eqref{eq:copies} using the same mantissa bit $p$. Let $q=4^{-p}>0$. For $\iout\in\{{\rm last},{\rm avg}\}$, let $\Delta^{\iout}(q,0)$ and $\Delta^{\iout}(0,q)$ denote $\Delta^{\iout}(T)$ in \eqref{eq:law-last} and \eqref{eq:law-avg} with $(q_\fw,q_\bw)=(q,0)$ and $(q_\fw,q_\bw)=(0,q)$, respectively; the other location is exact at each step. Consider $\tau\asymp M^a$ and setup of Section~\ref{sec:constant-schedule}.
Proposition~\ref{prop:regime} derives the following ratios for SR-induced risk increase:
\begin{equation}\label{eq:location-ratio}
 \frac{\Delta^{\rm last}(q,0)}{\Delta^{\rm last}(0,q)}
\asymp\frac{M^{a-b}}{\sigma^2+M^{1-b}},\qquad
 \frac{\Delta^{\rm avg}(q,0)}{\Delta^{\rm avg}(0,q)}
\asymp\frac{M^{1-b}}{\sigma^2+M^{1-b}}.
\end{equation}
With fixed label noise, the first ratio grows as $M^{a-b}$ and the second decreases as $M^{1-b}$. Thus, forward-only rounding causes a larger SR-induced error for the last iterate, whereas backward-only rounding leads to a larger error for the uniform average output, for sufficiently large $M$.

From \eqref{eq:law-last} to \eqref{eq:law-avg} with $\tau\asymp M^a$, averaging makes the forward-only SR-induced error smaller by a factor of order $M^{a-1}$,
whereas backward-only errors have the same order for both outputs.
Averaging also reduces exact-arithmetic risk in \eqref{eq:clean-law-avg}, so
backward rounding error becomes relatively dominant.
Figure~\ref{fig:compute_law}(a) illustrates this reversal with one-bit
rounding at one location. The practical lesson is to choose between the last iterate and the uniform average outputs before allocating precision across computations; under the conditions above, we need higher forward precision for the last iterate and higher backward precision for the uniform average. As a remark, \eqref{eq:location-ratio} is only for single precision location comparison, hence when both locations are rounded, the interaction terms related to $q_{\fw}q_{\bw}$ change the overall error behavior. Figure~\ref{fig:rules}(a) illustrates joint allocation at a fixed total bit count: the best tested split is $(p_\fw,p_\bw)=(4,1)$ for the last iterate and $(2,3)$ for the uniform average.

\subsection{When: precision allocation during training}
\label{sec:when}

Which updates deserve higher precision under compute budget? Fix width $M$,
batch $B$, horizon $T$, and constant learning rate $\eta$. We allocate the compute budget in \eqref{eq:declared-compute} across training updates.
First keep one feature computation exact throughout training and use
finite precision at the other. For the rounded location
$r\in\{\fw,\bw\}$, Theorem~\ref{thm:scheduled-main} shows
\begin{equation}\label{eq:single-schedule-objective}
\begin{gathered}
 \Delta^{\iout}(T)\asymp\widehat N^{\iout}
      =\sum_{u<T}A_{r,u}4^{-p_{r,u}},\quad
 A_{\fw,u}=\frac{\eta^2\widehat U_uK_2^{\iout}(u)}{2BM},\quad
 A_{\bw,u}=\frac{\eta^2\sigma^2K_1^{\iout}(u)}{2BM}
\end{gathered}
\end{equation}
where $K_1^{\iout}(u),K_2^{\iout}(u),\widehat U_u$ are defined in \eqref{s3:eq:scalar-responses_last}, \eqref{s3:eq:scalar-responses_avg}, and \eqref{s3:eq:scalar-source}. 
These coefficients $A_{r,u}$ are fixed before precision is allocated.
Given fixed compute budget \eqref{eq:declared-compute} and fixed $M,T$, we have a total precision constraint: $\sum_{u<T}p_{r,u}\le P$ for a fixed $P>0$. 
Proposition~\ref{prop:when-full} finds the minimizer of $\widehat N^{\iout}$ as
\begin{equation}\label{eq:lagrange}
 p_{r,u}=\begin{cases}
 \max\{0,\log_4(A_{r,u}/\lambda)\},&A_{r,u}>0,\\
 0,&A_{r,u}=0.
 \end{cases}
\end{equation}
Here $\lambda>0$ is a constant so that $\sum_{u<T}p_{r,u}=P$.
We assign more bits when $\widehat N^{\iout}$ reduces most.

The allocation in~\eqref{eq:lagrange} assigns more bits to larger coefficients.
For backward rounding at constant learning rate,
$K_1^{\rm last}(u)$ increases with $u$, whereas $K_1^{\rm avg}(u)$
decreases. This allocation gives nondecreasing backward
precision for the last iterate and nonincreasing backward precision
for the uniform average.
The practical lesson is that, at the same precision budget, switching
from the last iterate to a uniform average can reverse when higher
backward precision is useful. Figure~\ref{fig:compute_law}(b) compares
one- and three-bit precision in opposite time orders at nearly equal
average precision: using higher precision late during training gives smaller last-iterate error
at both locations, whereas using higher backward precision early during training gives smaller
uniform-average error. 
When both locations are rounded, interaction term changes
their precision schedules. Appendix~\ref{app:joint-phase-allocation}
and Proposition~\ref{app:prop-joint-phase-allocation} study the joint
allocation of precision over training.

\subsection{How many bits: limiting the risk increase from rounding}
\label{sec:howmuch}

\begin{table}[t]
\centering\small
\caption{Following Section~\ref{sec:constant-schedule} with $\tau\asymp M^a$.
Scaling definitions in \eqref{eq:accuracy-components} with $s_M$ defined in \eqref{eq:clean-law-avg}. }
\label{tab:where}
\begin{tabular}{@{}lcc@{}}
\toprule
 & Last iterate & Uniform average\\
\midrule
Exact-arithmetic risk scale $D^{\iout}$
 & $M^{1-b}+\rho s_M$ & $M^{1-b}+\rho s_M M^{1-a}$\\
Forward coefficient $A_\fw$
 & $\rho M^{a-b}$ & $\rho M^{1-b}$\\
Backward coefficient $A_\bw$
 & $\rho(s_M+M^{1-b})$ & $\rho(s_M+M^{1-b})$\\
Joint coefficient $A_\times$
 & $\rho^2(s_M+M^{1-b})M^{a-1}$ & $\rho M^{a-b}$\\
\bottomrule
\end{tabular}
\end{table}

How many forward and backward mantissa bits are needed to keep the risk
close to that of exact-arithmetic SGD? Consider the setting and results in Section~\ref{sec:constant-schedule}. We require the risk increase
from rounding to be at most a fraction $\vartheta$ of the exact-arithmetic risk:
\begin{equation}\label{eq:tolerance}
 \Delta^{\iout}(T)\le\vartheta\mathscr R_0^{\iout}(T),
 \qquad\vartheta>0.
\end{equation}
At $\tau\asymp M^a$, for $\iout\in\{{\rm last},{\rm avg}\}$, results in \eqref{eq:clean-law-last}-\eqref{eq:law-avg} can be simplified as
\begin{equation}\label{eq:accuracy-components}
  \mathscr R_0^{\iout}(T)
 \asymp
  D^{\iout}, \quad \Delta^{\iout}(T) 
 \asymp
  A_\fw q_\fw+A_\bw q_\bw+A_\times q_\fw q_\bw,
\end{equation}
where exact-arithmetic risk scale $D^{\iout}$ and three rounding coefficients $A_\fw,A_\bw,A_\times$ are defined in Table~\ref{tab:where}, which only depend on $M,\sigma,\eta,B$.
Corollary~\ref{cor:bits-full} then further proves that to ensure \eqref{eq:tolerance}, it suffices to satisfy the following bit requirements, for some constant $C>0$.
\begin{equation}\label{eq:accuracy-bit-thresholds}
\begin{aligned}
 p_\fw&\ge\left[\log_4
       \frac{C A_\fw}{\vartheta D^{\iout}}\right]_+,
 &p_\bw&\ge\left[\log_4
       \frac{C A_\bw}{\vartheta D^{\iout}}\right]_+,\\
 p_\fw+p_\bw&\ge\left[\log_4
       \frac{C A_\times}{\vartheta D^{\iout}}\right]_+,
 &[x]_+&=\max\{x,0\}.
\end{aligned}
\end{equation}
The first two bounds in \eqref{eq:accuracy-bit-thresholds} control precisions for forward-only or backward-only rounding, respectively, like Section~\ref{sec:where}. For joint rounding at two locations, the precisions must jointly satisfy all the bounds in \eqref{eq:accuracy-bit-thresholds}, though unequal forward and backward precisions are allowed.

For example, consider uniform averaging with exact forward computation and
fixed positive $\rho,\sigma^2,\vartheta$ at $\tau\asymp M^a$.
The lower bound on SR-induced error and upper bound on exact-arithmetic
risk in Corollary~\ref{cor:bits-full}, together with Table~\ref{tab:where},
imply the \emph{necessary} condition for \eqref{eq:tolerance},
$
 p_\bw\ge\left[(b-1)\log_4 M+
           \log_4\frac{c\rho\sigma^2}{\vartheta}\right]_+,
$
where $c>0$ is a fixed constant. Thus backward precision must
grow at least logarithmically with width to retain the prescribed relative accuracy.
Figure~\ref{fig:compute_law}(c) uses equal forward and backward precisions
for the last iterate. Among one to four tested bit counts, the smallest
that keeps additional SR-induced risk within $1\%$ of exact arithmetic risk is two bits at $M=128$
and three at $M=256$. In practice, we should choose precision alongside width and the output: fixed backward bits can prevent the model from keeping its improving exact-arithmetic accuracy as width grows.
When both computations are rounded, we should carefully consider their interaction in addition to their separate contributions before comparing compute costs
in Section~\ref{sec:compute-allocation}.

\begin{figure}[t]\centering
\includegraphics[width=.325\linewidth]{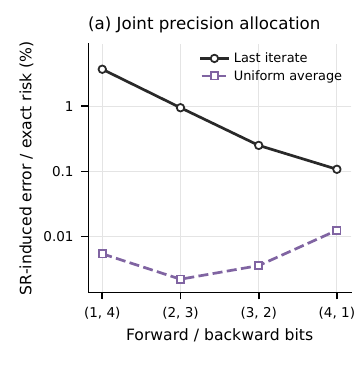}\hfill
\includegraphics[width=.325\linewidth]{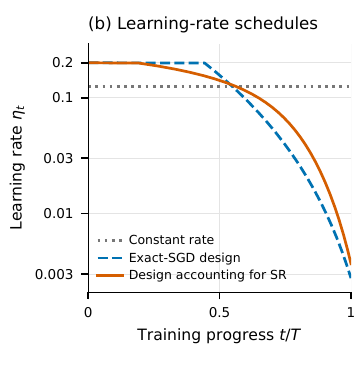}\hfill
\includegraphics[width=.325\linewidth]{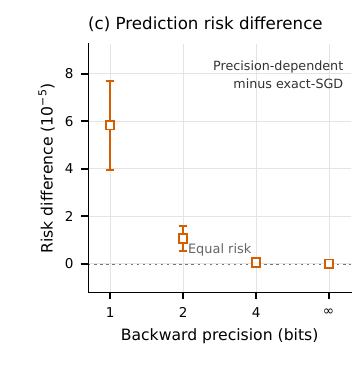}
\caption{\small
(a)~With five mantissa bits in total, more forward bits reduce last-iterate
error; averaged-output error is smallest at $(p_{\fw},p_{\bw})=(2,3)$.
Errors are relative to the corresponding exact-arithmetic risks.
(b)~The orange schedule accounts for backward-rounding error using
\eqref{eq:lr-rounded-schedule}; the blue schedule uses the exact-arithmetic
design \eqref{eq:lr-exact-schedule}. Both use one-bit backward SR;
the dotted line is a constant learning rate.
(c)~Full last-iterate prediction risk under the precision-dependent design
minus that under the exact-SGD design, at the same backward precision.
Positive values favor the exact-SGD design. In (c), points average eight
independent regression problems. Forward features are exact in (b,c).
Appendix~\ref{app:exp-figure2} gives details.}
\label{fig:rules}
\end{figure}

\subsection{How precision changes learning-rate design}
\label{sec:lr}

How should learning-rate decay change in low precision training?
Fix width $M$, batch size $B$, horizon $T$. We keep forward features exact and consider a constant backward precision $p_\bw$ and last-iterate output. Consider learning-rate upper bound $\eta_{\max}$ defined in Appendix~\ref{t3:sub-lr}, chosen before the sketch is drawn.
For either schedule below, we use continuous time $t\in[0,T]$ and write
$\tau_t=\int_0^t\eta(s)\,ds$, with $\tau=\tau_T$.
The constants $\lambda_0$ and $\lambda_\bw>0$ are chosen
separately so that the total learning rate accumulates to $\tau$.

Following the variational approach of \citet{bordelon2026optimal},
minimizing our continuous scalar risk objective at fixed $\tau$ gives the exact-arithmetic schedule
\begin{equation}\label{eq:lr-exact-schedule}
 \eta_{\rm exact}(t)=\min\left\{\eta_{\max},\,
 \sqrt{\lambda_0}\,[1+2(\tau-\tau_t)]^{1-1/(2a)}\right\}.
\end{equation}
With backward rounding, minimizing the corresponding continuous scalar objective leads to
\begin{equation}\label{eq:lr-rounded-schedule}
 \eta_{\rm SR}(t)=\min\left\{\eta_{\max},\,
 \frac{\sqrt{\lambda_\bw}\,[1+2(\tau-\tau_t)]^{1-1/(2a)}}
 {\sqrt{1+4^{-p_\bw}[1+2(\tau-\tau_t)]/M}}\right\}.
\end{equation}
Here if $p_\bw=\infty$, the denominator equals one and the normalization
$\lambda_\bw=\lambda_0$ recovers \eqref{eq:lr-exact-schedule}.
After optimizing the total learning rate and averaging over update intervals,
these formulas give discrete rates whose last-iterate risk is within a fixed
factor of the best admissible schedule at the same $M,B,T,p_\bw$
(Theorem~\ref{c2:arc-main}).

The precision-dependent term
$4^{-p_\bw}[1+2(\tau-\tau_t)]/M$ in \eqref{eq:lr-rounded-schedule} is largest at the start of training and decreases as training progresses.
If $4^{-p_\bw}(1+2\tau)/M\ll1$, this correction is small and the exact-SGD design in \eqref{eq:lr-exact-schedule} remains appropriate.
When the rate is below $\eta_{\max}$, $\tau-\tau_t$ is large, and this
term is much larger than one, the time-domain power-law approximation
has exponent $(a-1)/(a+1)$, compared with $2a-1$ for exact arithmetic.
The practical lesson is to tune precision and learning-rate
schedulers together. Lower precision does not universally imply a
lower peak or an earlier start of decay for learning rate.
Figure~\ref{fig:rules}(b) illustrates how the learning rate schedule changes with the one-bit backward rounding, compared to the exact-SGD learning rate schedule.
Figure~\ref{fig:rules}(c) shows that the precision-dependent learning rate design has slightly higher prediction risk; the difference is smaller at higher precisions rounding.

\subsection{Model width, training duration, and precision at fixed compute}
\label{sec:compute-allocation}

Given compute budget $F$ in \eqref{eq:declared-compute}, how should model width, training duration, and precision be chosen, and when can lower precision even improve prediction risk? We now use the uniformly averaged iterate with fixed learning rate $\eta$, batch $B$ and label noise $\sigma^2$.
We consider the setting of Section~\ref{sec:constant-schedule}, accuracy tolerance $\vartheta$ in \eqref{eq:tolerance} and compute budget $F$ in \eqref{eq:declared-compute}. 

\paragraph{Fixed precisions: balancing width and training duration.}
For fixed precisions $p_{\fw},p_{\bw}$ over training, the compute budget in \eqref{eq:declared-compute} bounds $M\tau$, so increasing width shortens the affordable training duration. In \eqref{eq:clean-law-avg}, larger width reduces approximation error, while longer training reduces optimization error.
Lower precision reduces compute cost, allowing a wider model and longer training. These additional training steps and larger model reduce approximation and optimization error in \eqref{eq:clean-law-avg}. But they improve full prediction risk only when those gains exceed the additional SR-induced error $\Delta^{\rm avg}(T)$. Specifically, \eqref{r4:eq:compute-width-requirements} in Appendix~\ref{app:rules} gives sufficient conditions on $M,\tau$ to ensure the accuracy tolerance in \eqref{eq:tolerance}.
In practice, this means we should choose width $M$ and training duration $T$ jointly like the result in Table~\ref{tab:compute-precision}. Lower precisions can provide more choices
of model sizes and training time given a fixed compute budget.

\paragraph{Growing precisions: maintaining accuracy as width increases.}
How should precision change when model width and training duration
increase?  For $\tau\asymp M^a$, Proposition~\ref{app:prop-matched-compute} simplifies \eqref{eq:accuracy-bit-thresholds} and proves
$p_\bw\ge\frac{b-1}{2}\log_2M-O(1)$ and
$p_\fw+p_\bw\ge\frac{a-1}{2}\log_2M-O(1)$ to ensure risk tolerance in \eqref{eq:tolerance}.
The $O(1)$ terms only depend on the $\sigma$, $\eta$, $\vartheta$ and constants. Backward precision must grow at least logarithmically. Forward
precision can remain fixed if backward precision grows enough to meet the total-bit condition. Total mantissa length of order $\log M$ implies compute
$f\asymp M^{a+1}\log M$ which improves the decay exponent in the last row of Table~\ref{tab:compute-precision} by
\begin{equation}\label{r4:eq:growing-finite-risk}
 \mathscr R^{\rm avg}(T)
 \asymp(F/\log F)^{-(b-1)/(a+1)}.
\end{equation}
This scaling only has an additional $\rm{polylog}(F)$ factor comparing with optimal order $F^{-(b-1)/(a+1)}$ from exact SGD training.  Proposition~\ref{app:prop-matched-compute} proves \eqref{r4:eq:growing-finite-risk}. Hence, increasing precisions logarithmically as model size increases can
prevent SR from becoming the accuracy bottleneck as training
scales.

\begin{figure}[t]\centering
\includegraphics[width=.325\linewidth]{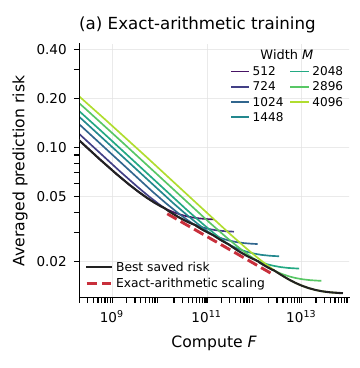}\hfill
\includegraphics[width=.325\linewidth]{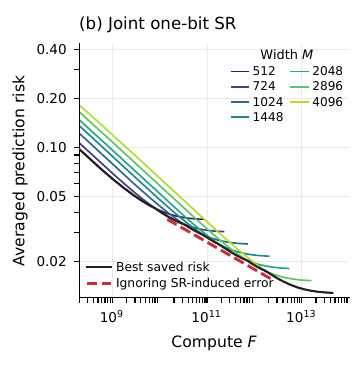}\hfill
\includegraphics[width=.325\linewidth]{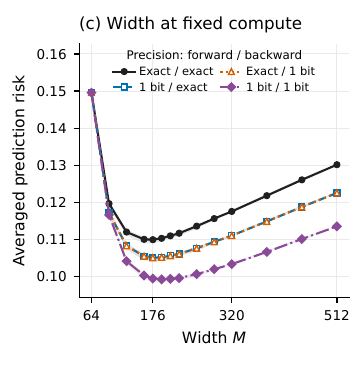}
\caption{\small Prediction risk of uniformly averaged outputs under a compute budget.
(a,b)~Exact-arithmetic and joint one-bit SR training ($p_\fw=p_\bw=1$) at several widths.
Black curves show the smallest prediction risks among tested widths and
training durations. Red dashed curves use the optimal-compute scaling in Table~\ref{tab:compute-precision} for exact-arithmetic training.
(c)~At a fixed budget $F$, joint SR permits a wider model and longer training
with lower prediction risk. We average over
three sketches. Appendix~\ref{app:exp-compute} gives details.}
\label{fig:frontier}
\end{figure}

\paragraph{When does lower precision improve risk at a fixed budget?}
For fixed width, time, and learning rate, SR cannot reduce
the expected risk of the uniformly averaged iterate. Its benefit
requires cheaper updates to permit a better training configuration.
Compare lower-cost ($L$) and higher-cost ($H$) precision choices and fix other training setup the same. Let $T_L,T_H$ be time horizons satisfying
the same budget in \eqref{eq:declared-compute}, and $\Delta_L,\Delta_H$ their averaged SR-induced risk
increases $\Delta^{\rm avg}(T)$.  Then the lower precision choice has smaller risk if and only if
\begin{equation}\label{eq:precision-time-balance}
 \Delta_L-\Delta_H
 <\mathscr R_0^{\rm avg}(T_H)-\mathscr R_0^{\rm avg}(T_L).
\end{equation}
Appendix~\ref{app:avg-compute-example} presents an example in which \eqref{eq:precision-time-balance} holds. This indicates that lower precision training can outperform exact-arithmetic training at a fixed compute budget by training longer or with a wider model.
At fixed width and training duration, lower precision adds rounding error
without using the computation it saves. At a fixed compute budget, those
savings can instead pay for more training updates or a wider model.
Lower precision can therefore reduce total prediction risk when the benefit
of the additional training or model capacity exceeds the rounding error.

\section{Experiments}
\label{sec:experiments}

In this section, we examine the compute tradeoffs in Section~\ref{sec:compute-allocation},
then study data reuse and neural-network computations beyond our theoretical analysis.
Consider $a=2$, $b=3/2$,
$\sigma^2=0.1$, $B=32$ and $d=4M$. In simulations we may train beyond $\tau_T=M^a$, so these comparisons do not follow all assumptions of Theorem~\ref{thm:scheduled-main}.
Compute follows \eqref{eq:declared-compute}. Appendix~\ref{app:experiments} shows the experimental details.

\paragraph{Width and training duration at fixed compute.}
Figure~\ref{fig:frontier} illustrates the joint choice of width, duration
and precision in Section~\ref{sec:compute-allocation} for uniformly
averaged outputs. Figure~\ref{fig:frontier}(a,b) compare prediction risks across widths.
Both red dashed curves use the exact-arithmetic scaling
$F^{-1/6}$ in Table~\ref{tab:compute-precision}. Figure~\ref{fig:frontier}(b) has better optimal scaling than the joint-SR scaling $F^{-1/8}$ in Table~\ref{tab:compute-precision},
since simulation here does not impose all conditions of Theorem~\ref{thm:scheduled-main}.
In Figure~\ref{fig:frontier}(c), joint one-bit SR permits a wider model and longer training,
reducing mean prediction risk by $9.7\%$ at the same compute budget among
the tested configurations. The gain from a wider model and longer training
outweighs the SR-induced error: precision should be selected together
with width and time by comparing full risks at a fixed compute budget.

\paragraph{Data reuse at fixed compute.}
Figure~\ref{fig:reuse-compute} asks whether SR updates remain useful
when training revisits existing examples during training. At fixed width and compute cap
$F=2^{27}$, with fresh data or either reused dataset, lower precision allows longer training  and wider model. The uniform average output then achieves about 5\% lower prediction risk than exact-arithmetic training. Figure~\ref{fig:reuse-compute}(b) shows lower mean prediction risk for low-precision training at every tested compute budget. Additional training offsets the SR-induced error, as in \eqref{eq:precision-time-balance}. The more data reused, the lower the prediction risk in Figure~\ref{fig:reuse-compute}(c). Thus the compute benefit can persist with data
reuse. We leave a data-reuse scaling law with SR for future work.

\begin{figure}[t]\centering
\includegraphics[width=.325\linewidth]{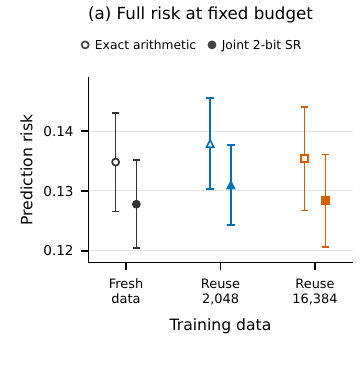}\hfill
\includegraphics[width=.325\linewidth]{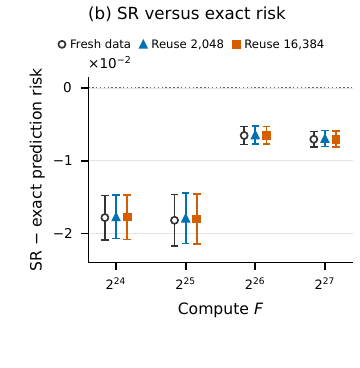}\hfill
\includegraphics[width=.325\linewidth]{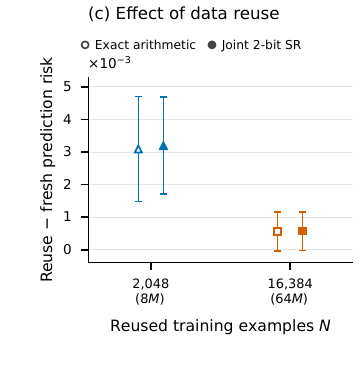}
\caption{\small Prediction risk under a fixedcompute budget with fresh and reused data.
Uniformly averaged outputs at $M=256$, using exact arithmetic or joint
two-bit feature SR. (a)~Full prediction risks at $F=2^{27}$.
(b)~The difference between SR and exact prediction risks across budgets; negative values indicate SR is beneficial. (c)~Reuse-minus-fresh risk differences at $F=2^{27}$ for different data sizes. For more details, see Appendix~\ref{app:exp-fixed-budget}.}
\label{fig:reuse-compute}
\end{figure}

\paragraph{Precision choices in neural networks.}
Figure~\ref{fig:nn} extends the location question in Section~\ref{sec:where}
to a six-layer transformer trained on 200M FineWeb-Edu tokens and
a CIFAR-10 MLP. The combined rounding configuration has higher transformer
validation loss than activation-only rounding in Figure~\ref{fig:nn}(a). Our theory only considered independent rounding of features in \eqref{eq:copies}. Sharing introduces dependence between rounding errors. Figure~\ref{fig:nn}(b) compares this independent rounding with shared rounding, which reuses the same rounded activation for the forward computation and the weight-gradient computation. Remarkably, it shows that sharing the rounded activation between forward and weight-gradient computations also changes performance.
At one bit, sharing lowers the validation losses relative to independent rounding for both models. Compared with training without SR, shared rounding gives lower loss
for MLP but higher loss for transformer. 
Figure~\ref{fig:nn}(c) separates rounding during training from rounding at
test time. The one-bit forward-trained MLP has test loss $1.624$ without SR during test-time
predictions and $1.672$ with forward-activation SR at test time. Precision comparisons should
therefore specify all rounded computations and whether activations are
shared, and compare models with SR disabled at test time before
assessing low-precision inference.

\begin{figure}[htp]\centering
\includegraphics[width=.325\linewidth]{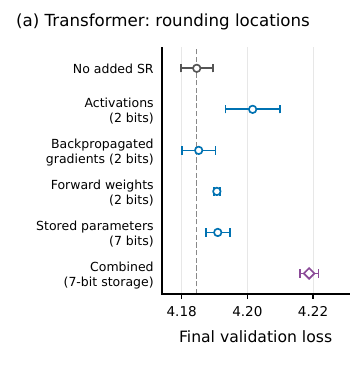}\hfill
\includegraphics[width=.325\linewidth]{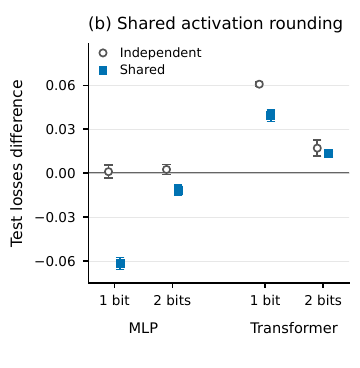}\hfill
\includegraphics[width=.325\linewidth]{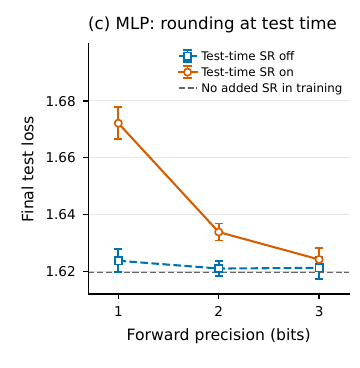}
\caption{\small Rounding locations, shared activations and test-time prediction in neural networks.
(a)~Final transformer validation loss; the combined case rounds both
activation computations, incoming gradients and forward weights to two
bits, with seven-bit parameter storage.
(b)~Shared SR reuses the same forward-rounded activation in the weight-gradient
computation. Points show final loss differences from same-seed
training with SR and without SR. Negative values indicate lower loss than this baseline.
(c)~Each trained MLP is tested twice with its parameters fixed: with
forward-activation SR disabled, and with SR enabled at its training precision.
Both tests use the same examples. Appendix~\ref{app:exp-nn} gives details and final performance values.}
\label{fig:nn}
\end{figure}

\FloatBarrier
\section{Conclusion}
\label{sec:discussion}

We derived scaling laws for stochastic rounding in random-feature regression.
Precision, model size and training time should be chosen together.
Under our compute model, lower precision can improve prediction at a fixed
budget when the benefits of a larger model or longer training outweigh the
added rounding error. The learning rates and the choice to average model
parameters also affect where and when higher precision is needed.
Extending the theory to data reuse, more quantization locations, feature learning regimes \citep{ba2022high}, and measuring savings on hardware, remain future work.
\label{sec:main-text-end}

\paragraph{Use of AI assistance.}
The authors used ChatGPT to develop proof candidates, especially for tracking constants in the proofs, and Claude and Codex to set up initial experiment codes and figures. The authors are responsible for the results, their stated scope and the final manuscript. 
\paragraph{Acknowledgements.}
ZW and MWM would like to acknowledge the NSF and the DARPA DIAL and DARPA AIQ programs for partial support of this work.

\bibliographystyle{iclr2027_conference}
\bibliography{references}

\appendix

\makeatletter
\setcounter{theorem}{0}
\@addtoreset{theorem}{section}
\renewcommand{\thetheorem}{\Alph{section}.\arabic{theorem}}
\def\theHtheorem{\Alph{section}.\arabic{theorem}}
\makeatother
\section{Related work}
\label{app:related}
This appendix reviews work on scaling laws, quantization and stochastic
rounding, and explains how our model and results relate to these studies.

\paragraph{Empirical scaling laws.}
Scaling laws try to analyze how loss changes as model size, sample size and compute increasing \citep{kaplan2020scaling,hoffmann2022training}, with extensions to
continual weather training \citep{subramanian2026weather}.
\citet{kumar2024precision,sun2025floating} use fitted laws to choose model
size, the number of data points and precision for a given compute budget. Related work studies
scaling in quantization-aware training \citep{chen2025scaling} and joint
quantization and sparsity \citep{frantar2025compression}.
\citet{kim2026mixed} extend fitted laws to different forward and backward
precisions and hardware-dependent compute costs.

\paragraph{Theory of scaling law.}
One of the solvable models for scaling law is the random feature model
\citep{mei2022rf,bach2023projections,lejeune2022sketched,cheng2024ridge,defilippis2024deterministic}.
Linear models with power-law random features (PLRF) have been applied to theoretically understand the scaling law of prediction risk
\citep{bahri2024explaining,maloney2022solvable,bordelon2024dynamical,bordelon2022learning,paquette2024phases,paquette2024homogenization,lin2024scaling}.
These results are also extended to the setting of data-reuse SGD \citep{lin2025reuse}, mini-batch control
\citep{chen2026minibatch} and general learning-rate schedules
\citep{bordelon2026optimal}. 
Power-law spectra have also been observed and analyzed in trained weights of networks
\citep{hodgkinson2025models,MM19_HTSR_ICML,MM20_SDM,MM20a_trends_NatComm}. 

\paragraph{Quantization in linear regression.}
\citet{zhang2025quantization} derive prediction-risk upper bounds for linear
regression with quantized data, labels, parameters, activations and gradients.
\citet{zhang2026precision} study scaling laws with model size and training data in
sketched linear regression, giving upper and lower bounds for their
quantization models. Both use constant learning rates and parameter averaging.
Their general upper bounds allow unbiased quantization, while their scaling results make additional assumptions on rounding-error
covariance. Our quantization from SR does not generally follow their multiplicative or additive quantization assumptions.
Our laws also cover last-iterate predictor, and general
learning-rate and precision schedules.

\paragraph{Low-precision training.}
QAT uses low-precision weights and activations
during training for efficient inference
\citep{jacob2018quantization,esser2020lsq,liu2024llmqat}.
Low-precision training methods include BinaryConnect
\citep{courbariaux2015binaryconnect}, quantized gradients
\citep{zhou2016dorefa}, SR training \citep{gupta2015limited}, and direct
low-precision weight updates \citep{zhao2024direct}. Theoretical works can be found by \citet{li2017training,desa2018halp}.
Practical LLMs use FP8, MXFP4, and FP4 in training
\citep{micikevicius2022fp8,deepseek2024v3,tseng2025mxfp4,castro2025quartet,chmiel2025fp4,panferov2026quartet2}.
\citet{zhang2017zipml} use independent rounding in \eqref{eq:copies} for low-precision training. 
Learned precision \citep{yu2022ldp} and learning-rate schedules \citep{wolfe2024schedules} for low-precision training studied how precision and optimization choices work together. 

\paragraph{Numerical analysis for SR.}
\citet{higham2019new,higham2020sharper,connolly2021sr,elarar2022sr,croci2022sr} analyze the bound of the rounding error for SR in numerical computations.
Other work also studies SR with few random bits \citep{fitzgibbon2025stochastic},
computational complexity \citep{drineas2024stochastic}, and its effects on
small singular values of random matrices \citep{boutsikas2024small,dexter2025stochastic,ma2026stochastic}.
We refer to \citet{elarar2026survey} as a comprehensive review of SR theory, hardware and applications.
\section{Stochastic rounding and sketch geometry}
\label{app:actual-scaling}\label{sec:rev3_actual_scaling}
This appendix develops the rounding identities and random-feature geometry
used to prove Theorem~\ref{thm:scheduled-main} in Appendix~\ref{cl:finite}.
Lemmas~\ref{arithmetic:lem:scalar}--\ref{arithmetic:lem:joint} bound Gaussian
rounding moments, and Proposition~\ref{arithmetic:prop:exact} derives the
training second-moment recursions. Proposition~\ref{as:prop-event} combines
the spectral, teacher and rounding-coordinate bounds into one
high-probability sketch event used in the subsequent analysis. 

\begingroup
\let\E\relax\newcommand{\E}{\mathbb E}
\let\tr\relax\newcommand{\tr}{\operatorname{tr}}
\let\diag\relax\newcommand{\diag}{\operatorname{diag}}
\let\norm\relax\newcommand{\norm}[1]{\left\|#1\right\|}
\theoremstyle{plain}
\subsection{Shared notation and second-moment comparison}
\label{app:model-details}
We define the oracle predictor and second moments used in the training
recursions of Proposition~\ref{arithmetic:prop:exact}.
Throughout this appendix, unless a sketch probability is stated,
expectations condition on $S$ and average over the teacher, training data
and rounding variables.
\paragraph{Oracle predictor and rounding geometry.}
For random sketch $S$, use $G$ from \eqref{eq:G} and define
\begin{equation}\label{app:eq:oracle-geometry}
\begin{gathered}
 G=\sum_{j=1}^M g_j\phi_j\phi_j^\top,\qquad
 v_*^S=L_S\theta_*,\quad L_S=G^{-1}S\Sigma_d, \quad
 V=\E[v_*^S(v_*^S)^\top\mid S]=L_SC_\theta L_S^\top.
\end{gathered}
\end{equation}
The eigenvectors $\phi_j$ are orthonormal and
$g_1\ge\cdots\ge g_M>0$. The matrix $V$ need not commute with $G$.
Write the oracle predictor as $v_*=v_*^S$ and $\Sigma=\Sigma_d$ for simplicity. Define
\begin{equation}\label{eq:rounding-geometry}
 \theta_G=\tr G,\qquad J=\frac{M}{\theta_G}\diag(G_{11},\ldots,G_{MM}).
\end{equation}
Here $\theta_G$ is total feature variance and $J$ records normalized
coordinate variances in the rounding basis. It need not commute with $G$.

\paragraph{Approximation error and residual noise.}
Define
\begin{equation}\label{eq:vstar}
\begin{gathered}
 \Sigma_\perp=\Sigma_d-\Sigma_dS^\top G^{-1}S\Sigma_d,\\
 A_S=\tfrac12\theta_*^\top\Sigma_\perp\theta_*,\qquad
 \bar A(S)=\tfrac12\tr(C_\theta\Sigma_\perp),\qquad
 s(S)=\sigma^2+2\bar A(S).
\end{gathered}
\end{equation}
Then $y=z^\top v_*^S+\widetilde\epsilon$, where
$\widetilde\epsilon=x^\top(\theta_*-S^\top v_*^S)+\epsilon$.
Conditional on $S,\theta_*$, this residual is centered Gaussian and
independent of $z$, because their cross-covariance is
$S\Sigma_d\theta_*-Gv_*^S=0$. Its variance is $\sigma^2+2A_S$;
averaging over the teacher gives $s(S)$. Write
$\bar A=\bar A(S)$ and $s=s(S)$.
The normal equation $Gv_*^S=S\Sigma_d\theta_*$ splits the risk in
\eqref{eq:risk} as
\begin{equation}\label{app:eq:risk-decomposition}
 \mathcal E_S(v,\theta_*)=A_S+\tfrac12\|v-v_*^S\|_G^2
\end{equation}
where $A_S$ is the approximation error and the second term is the
estimation error. Define
\begin{equation}\label{eq:estimation-risk}
 E_0^{\iout}=\mathscr R_0^{\iout}(T)-\bar A(S),\qquad
 \mathscr R^{\iout}(T)=\bar A(S)+E_0^{\iout}+\Delta^{\iout}(T).
\end{equation}

\paragraph{Output weights and second moments.}
For outputs in \eqref{eq:output-weights}, write
$v_T^{\iout}=\sum_{t=1}^T w_t^{\iout}v_t$, with
\begin{equation}\label{app:eq:output-weights}
 w_t^{\rm last}=\mathbf 1\{t=T\},\qquad
 w_t^{\rm avg}=1/T\quad(1\le t\le T).
\end{equation}
Both choices have nonnegative weights summing to one. Define
\begin{equation}\label{eq:moment-sources}
\begin{gathered}
 C_t=\E[(v_t-v_*)(v_t-v_*)^\top\mid S],\qquad
 W_t=\E[v_tv_t^\top\mid S],\\
 R_t=s+\tr(GC_t^0),\qquad U_t=\tr(JW_t^0),
\end{gathered}
\end{equation}
where $C_t^0,W_t^0$ are the corresponding second moments for
exact-arithmetic SGD. These are full second moments, not centered
covariances. The sources $R_t$ and $U_t$ measure its residual variance
and parameter magnitude in rounding coordinates, respectively. 

In addition to the second moments in \eqref{eq:moment-sources}, define
\begin{equation}\label{eq:moment-gap}
 X_t=C_t-C_t^0,\qquad x_t=\tr(GX_t),\qquad y_t=\tr(JX_t).
\end{equation}
Proposition~\ref{arithmetic:prop:exact} will show that
$X_t=W_t-W_t^0\succeq0$. Thus $X_t$ is a difference of second moments;
it is not the second moment of $v_t-v_t^0$.

\subsection{Gaussian SR moments and training recursions}
\label{app:binary-sr}
We prove the rounding-moment bounds in
Lemmas~\ref{arithmetic:lem:scalar}--\ref{arithmetic:lem:joint} and use them
to derive the recursions in Proposition~\ref{arithmetic:prop:exact}.
Use the binary grid, stochastic rounding rule and variance
identity in Section~\ref{sec:setup} and \eqref{eq:cell}.
At representable values put $\nu_p(x)=0$, and set $\nu_\infty=0$. 

With the precision scale \eqref{eq:precision}, binary rounding also obeys
\begin{equation}\label{arithmetic:eq:cell}
 \nu_p(-x)=\nu_p(x),\qquad\nu_p(x/2)=\nu_p(x)/4,
 \qquad |Q_p(x)|\le2|x|.
\end{equation}
\begin{lemma}[Gaussian scalar and matrix moments]\label{arithmetic:lem:scalar}
With $q_p$ defined in \eqref{eq:precision}, for $X\sim N(0,g)$, $g>0$,
and $p\in\mathbb N_0$, put
$m_p(g)=\E\nu_p(X)$ and $n_p(g)=\E[\nu_p(\sqrt g Z)Z^2]$,
where $Z\sim N(0,1)$. Then
\[
 q_pg/192\le m_p(g)\le q_pg/4,\qquad
 n_p(g)\ge2\log2\,m_p(g)\ge m_p(g),\qquad
 n_p(g)\le3q_pg/4.
\]
If $G\succ0$, $z\sim N(0,G)$ and $\gamma_i=G_{ii}$, then
\begin{equation}\label{arithmetic:eq:weighted}
 \frac{q_p\gamma_i}{192}G\preceq
 K_{p,i}:=\E[\nu_p(z_i)zz^\top]
 \preceq\frac34q_p\gamma_iG.
\end{equation}
\end{lemma}
\begin{proof}
Write $X=\sqrt gZ$, let $\epsilon=2^{-p}\le1$, and denote the
$N(0,g)$ density by $f_g$.
Let $\ell<u$ be consecutive positive grid points, and choose $e\in\mathbb Z$
so that $2^e\le\ell<2^{e+1}$. Their rounding cell $[\ell,u]$ has width
$\Delta=u-\ell=\epsilon2^e$. By \eqref{eq:cell},
\[
 \int_\ell^u\nu_p(x)f_g(x)\,dx
 \ge f_g(u)\Delta^3/6
 \ge\frac{\epsilon^2}{6}
       \int_\ell^u 2^{2e} f_g((1+\epsilon)x)\,dx.
\]
Here $u\le(1+\epsilon)\ell$, and $f_g$ decreases on the
positive half-line. Summing over cells, using reflection and
$2^{2e}\ge x^2/4$ on each cell, gives
\[
 m_p(g)\ge\frac{q_p}{24}\int_\mathbb R
 x^2 f_g((1+\epsilon)x)\,dx
 =\frac{q_pg}{24(1+\epsilon)^3}\ge q_pg/192.
\]
Tonelli applies because the integrands are nonnegative. The upper bounds
on $m_p,n_p$ follow from \eqref{eq:cell} and $\E Z^4=3$.
For the lower bound on $n_p$, let $\mu$ have density
$\nu_p(x)f_g(x)/m_p(g)$. Binary scaling gives
$m_p(g/4)=m_p(g)/4$. Together with
$f_{g/4}(x)=2f_g(x)e^{-3x^2/(2g)}$, this yields
$\E_\mu e^{-3X^2/(2g)}=1/8$.
Jensen's inequality implies $\E_\mu(X^2/g)\ge2\log2$, or
$n_p(g)\ge2\log2\,m_p(g)$.

For the matrix comparison, Gaussian regression gives
$z=(Ge_i/\gamma_i)z_i+\xi$, where $\xi$ is independent of $z_i$ and has
covariance $G-Ge_ie_i^\top G/\gamma_i$. Hence
$K_{p,i}=m_p(\gamma_i)G+  \frac{n_p(\gamma_i)-m_p(\gamma_i)}{\gamma_i}Ge_ie_i^\top G$.
Whitening by $G^{-1/2}$ gives eigenvalues $m_p(\gamma_i)$ on the orthogonal
complement of the unit vector $G^{1/2}e_i/\sqrt{\gamma_i}$ and $n_p(\gamma_i)$
on that vector. The scalar bounds on these eigenvalues prove
\eqref{arithmetic:eq:weighted}.
\end{proof}

\begin{lemma}[Joint cell variance]\label{arithmetic:lem:joint}
For centered jointly Gaussian $(X,Y)$ with variances $g,h>0$, arbitrary
correlation in $[-1,1]$ and finite precisions $p,p'$, one has
\[
 c_\times q_pq_{p'}gh\le\E[\nu_p(X)\nu_{p'}(Y)]
 \le\tfrac3{16}q_pq_{p'}gh,
 \qquad
 c_\times=\tfrac38\left(\frac{\delta(1-\delta)}{64}\right)^2,
 \quad\delta=2^{-8}.
\]
Exact precisions or zero variances give zero on both sides.
\end{lemma}
\begin{proof}
For $X\sim N(0,g)$, let $E_X$ be the event that
$\sqrt g/4\le|X|\le4\sqrt g$ and $|X|$ lies in the middle
$[\delta,1-\delta]$ fraction of its cell.
The folded Gaussian density is at most $4/(5\sqrt g)$.
The cells achieving this magnitude range have total width at most
$8\sqrt g$, so their excluded edges have probability at most
$(64/5)\delta=1/20$. The two magnitude exclusions have probabilities
at most $1/5$ and $1/16$, by the same density bound and Markov's
inequality. Hence $\Pr(E_X)\ge11/16$.
On $E_X$, the cell width satisfies
$\Delta\ge2^{-p}\sqrt g/8$, and \eqref{eq:cell} gives $\nu_p(X)\ge\delta(1-\delta)\Delta^2
 \ge\frac{\delta(1-\delta)}{64}q_pg.$
The analogous event for $Y$ intersects $E_X$ with probability at
least $3/8$, without any independence assumption. Multiplying the
variance bounds on this intersection proves the lower estimate.
For the upper estimate, use \eqref{eq:cell} and the Gaussian identity
$\E X^2Y^2=gh(1+2\operatorname{Corr}(X,Y)^2)\le3gh$.
Exact precisions and zero variances are immediate.
\end{proof}

For $z\sim N(0,G)$ and $r\in\{\fw,\bw\}$, define the diagonal
conditional rounding covariance
$D_{r,t}(z)=\diag(\nu_{p_{r,t}}(z_i))$ and the linear maps
\begin{align}
 \mathcal F_t(W)&=\E[zz^\top\tr(D_{\rm fw,t}(z)W)],\label{arithmetic:eq:maps}\\
 \mathcal B_t(C,s)&=\E[D_{\rm bw,t}(z)(z^\top Cz+s)],\quad
 \mathcal I_t(W)&=\E[D_{\rm bw,t}(z)\tr(D_{\rm fw,t}(z)W)].\nonumber
\end{align}
These maps are positive for $C,W\succeq0$, $s\ge0$.
With $J$ from \eqref{eq:rounding-geometry}, $U=\tr(JW)$ and
$R=s+\tr(GC)$, the preceding lemmas give 
\begin{equation}\label{arithmetic:eq:map-bounds}
 f_t^-UG\preceq\mathcal F_t(W)\preceq f_t^+UG,\quad
 \beta_t^-RJ\preceq\mathcal B_t(C,s)\preceq\beta_t^+RJ,\quad
 m_t^-UJ\preceq\mathcal I_t(W)\preceq m_t^+UJ,
\end{equation}
where the upper coefficients are
\begin{equation}\label{eq:rounding-source-coefficients}
 f_t^+=\frac{3\theta_G}{4M}q_{\fw,t},\qquad
 \beta_t^+=\frac{3\theta_G}{4M}q_{\bw,t},\qquad
 m_t^+=\frac{3\theta_G^2}{16M^2}q_{\fw,t}q_{\bw,t},
\end{equation}
and the lower coefficients are
\begin{align}
 f_t^-&=q_{\rm fw,t}\theta_G/(192M),\qquad
 \beta_t^-=q_{\rm bw,t}\theta_G/(192M), \quad 
 m_t^-=c_\times q_{\rm fw,t}q_{\rm bw,t}\theta_G^2/M^2.\label{arithmetic:eq:coefficients}
\end{align}
Indeed $\mathcal F_t(W)=\sum_iW_{ii}K_{p_{\rm fw,t},i}$, the $i$th diagonal
entry of $\mathcal B_t$ is $\tr(CK_{p_{\rm bw,t},i})+s m_{p_{\rm bw,t}}(G_{ii})$,
and that of $\mathcal I_t$ is
$\sum_kW_{kk}\E[\nu_{p_{\rm bw,t}}(z_i)\nu_{p_{\rm fw,t}}(z_k)]$.
Since $W\succeq0$, every $W_{kk}$ is nonnegative. Applying the preceding
moment bounds and using
$\sum_kG_{kk}W_{kk}=\theta_G U/M$ and $G_{ii}=\theta_G J_{ii}/M$
proves \eqref{arithmetic:eq:map-bounds}. 

\begin{proposition}[Actual moments and matched-reference comparison]
\label{arithmetic:prop:exact}\label{arithmetic:lem:onestep}
Consider deterministic precision
schedules and nonnegative learning rates in \eqref{eq:sgd}, with $h_t$ as in \eqref{eq:clock}. Write
\begin{equation}\label{arithmetic:eq:products}
 H_t=I-\eta_tG,\qquad
 P_{u\to t}=\prod_{v=u}^{t-1}H_v,\qquad P_t=P_{0\to t},
 \quad P_{t\to t}=I.
\end{equation}
The matrices $H_t$ commute, so each product $P_{u\to t}$ is symmetric.
With the moments \eqref{eq:moment-sources}, gap \eqref{eq:moment-gap} and
rounding maps \eqref{arithmetic:eq:maps}, all finite-time moments exist, and $C_0=V,$
\begin{align}
 C_{t+1}&=H_tC_tH_t+h_t\{(s+\tr(GC_t))G+GC_tG+
 \mathcal F_t(W_t)+\mathcal B_t(C_t,s)+\mathcal I_t(W_t)\},\label{arithmetic:eq:recursion}\\
 W_t&=C_t+V-P_tV-VP_t.\label{arithmetic:eq:W}
\end{align}
For exact-arithmetic SGD, the same recursion holds with all three
rounding maps set to zero. Moreover,
$X_t=W_t-W_t^0=C_t-C_t^0$ satisfies $X_0=0,$
\begin{equation}\label{arithmetic:eq:innovation}
 X_{t+1}=H_tX_tH_t+h_tQ_t,\quad
 Q_t=GX_tG+x_tG+\mathcal F_t(W_t)+\mathcal B_t(C_t,s)+\mathcal I_t(W_t),
\end{equation}
and $X_t,Q_t\succeq0$ at every finite time. 
\end{proposition}

\begin{proof}
The bound in \eqref{arithmetic:eq:cell} implies
\[
 \|v_{t+1}\|\le(1+\eta_t B^{-1}\textstyle\sum_\ell\|b_{t\ell}\|\|a_{t\ell}\|)
 \|v_t\|+\eta_t B^{-1}\textstyle\sum_\ell\|b_{t\ell}\||y_{t\ell}|.
\]
Here $a_{t\ell},b_{t\ell}$ are the rounded features from
\eqref{eq:copies}. Repeatedly applying this inequality bounds each
finite-time state by a polynomial in finitely many Gaussian magnitudes,
including the teacher. All its moments therefore exist.

Let $\mathscr F_t$ contain the teacher and the history through update
$t-1$, with $S$ still fixed. Given this history, put $e=v_t-v_*$ and use the prediction residual
$\widetilde\epsilon$ defined in Appendix~\ref{app:model-details}.
Conditional on $\mathscr F_t$, $z$ and $\widetilde\epsilon$ are
jointly centered Gaussian with cross-covariance
$S\Sigma\theta_*-Gv_*=0$, hence are independent. Write
$\widetilde s=\sigma^2+2A_S$ for the residual variance at this conditioning stage.
Write $a=z+\xi$, $b=z+\zeta$,
$r=z^\top e-\widetilde\epsilon$ and $\Gamma=b(r+\xi^\top v_t)$.
Conditional on $(\mathscr F_t,x,\epsilon)$, the rounding errors
$\xi,\zeta$ are centered and independent, with covariances
$D_{\rm fw,t}(z),D_{\rm bw,t}(z)$. Thus
\[
 \E\Gamma=zr,\qquad
 \E[\Gamma\Gamma^\top]=(zz^\top+D_{\rm bw,t}(z))
                (r^2+v_t^\top D_{\rm fw,t}(z)v_t).
\]
The Gaussian fourth-moment identity gives
$\E[zz^\top(z^\top e)^2]=(e^\top Ge)G+2Gee^\top G$.
The residual cross term vanishes, so averaging the last display gives
\[
 \E[\Gamma\mid\mathscr F_t]=Ge,\quad
 \E[\Gamma\Gamma^\top\mid\mathscr F_t]
 =(\widetilde s+e^\top Ge)G+2Gee^\top G+
 \mathcal F_t(v_tv_t^\top)+\mathcal B_t(ee^\top,\widetilde s)+\mathcal I_t(v_tv_t^\top).
\]
The batch samples are conditionally independent. Consequently,
$\E[e_{t+1}\mid\mathscr F_t]=H_te$, and the conditional second moment
equals $H_tee^\top H_t$ plus $h_t$ times the conditional covariance
of $\Gamma$. This covariance is the preceding second moment of
$\Gamma$ minus $Gee^\top G$.
Averaging over the teacher uses
$\E[\widetilde s\mid S]=s$ and linearity of the maps, proving
\eqref{arithmetic:eq:recursion}. Iterating the conditional mean gives
\begin{equation}\label{arithmetic:eq:two-time}
 \E[e_u e_w^\top\mid S]=P_{w\to u}C_w\quad(u\ge w),\qquad
 \E[v_t\mid S,\theta_*]=(I-P_t)v_*.
\end{equation}
\label{arithmetic:cor:twotime}
Since $e_0=-v_*$, the mean recursion also gives
$\E[e_tv_*^\top\mid S]=-P_tV$. Expanding
$v_tv_t^\top=(e_t+v_*)(e_t+v_*)^\top$ proves
\eqref{arithmetic:eq:W}; no commutation between $V$ and $P_t$ is needed.
Subtracting the exact recurrence proves \eqref{arithmetic:eq:innovation}.
Positivity of the maps and induction from zero prove $X_t,Q_t\succeq0$.
\end{proof}

To compare the rounded and exact-arithmetic runs, for $Z\succeq0$ define
the order-preserving affine maps
\begin{equation}\label{core:comparison-innovations}
 Q_t^\pm(Z)=GZG+[\tr(GZ)+f_t^\pm(U_t+\tr(JZ))]G
 +[\beta_t^\pm(R_t+\tr(GZ))+m_t^\pm(U_t+\tr(JZ))]J.
\end{equation}
Then~\eqref{arithmetic:eq:map-bounds} gives
$Q_t^-(X_t)\preceq Q_t\preceq Q_t^+(X_t)$.
Starting from $X_0^\pm=0$, let
$X_{t+1}^\pm=H_tX_t^\pm H_t+h_tQ_t^\pm(X_t^\pm)$.
The map bounds and induction in the positive semidefinite order give
\begin{equation}\label{arithmetic:eq:sandwich}
 0\preceq X_t^-\preceq X_t\preceq X_t^+,
 \qquad Q_t^-(X_t^-)\preceq Q_t\preceq Q_t^+(X_t^+).
\end{equation}

\subsection{One sketch event and constants}
\label{core:geometry}\label{sec:rev3_witness_port}
We now prove Proposition~\ref{as:prop-event}, constructing one event,
independent of the learning-rate and precision schedules, on which all
geometric comparisons hold.
Throughout this construction, $c,C>0$ may change from line to line and
depend only on the fixed population family in \eqref{eq:population}.
The numerical thresholds used later to define the admissible schedules
are recorded explicitly.

Split the population coordinates after $n_H=\lfloor\kappa_{\rm geo}M\rfloor$:
$S=(S_H,S_R)$, with matching diagonal blocks $\Sigma_H,\Sigma_R$ of $\Sigma$.
Choose the fixed constants
\begin{equation}\label{as:geometry-constants}
\begin{gathered}
 \kappa_{\rm geo}=\tfrac14\min(1,c_d-1),\quad
 \delta_{\rm geo}=\tfrac34(c_d-1),\\
 M_{\rm geo}=\lceil\kappa_{\rm geo}^{-1}\rceil,\quad c_H=\tfrac14,\\
 r_{\rm geo}=\tfrac12(\sqrt{1+\delta_{\rm geo}}-1),\quad p_{\rm geo}=\min(1/32,r_{\rm geo}^2/2),\\
 K_{0,\rm geo}=\sqrt{c_d+1}+2,\quad
 K_{R,\rm geo}=\sqrt{c_d+1}+1+r_{\rm geo}.
\end{gathered}
\end{equation}
Choose a fixed minimum width $M_{\rm w}$ satisfying
\begin{equation}\label{as:width-threshold}
 M_{\rm w}\ge M_{\rm geo}\ge4.
\end{equation}
For $M\ge M_{\rm geo}$ define event
\begin{equation}\label{as:gaussian-event}
 \mathcal E_M=\{\|S\|\le K_{0,\rm geo},\ s_{\min}(S_H)\ge c_H,
 \ r_{\rm geo}\le s_{\min}(S_R^\top)
       \le s_{\max}(S_R^\top)\le K_{R,\rm geo}\}.
\end{equation}

\begin{lemma}[Spectral and approximation geometry]
\label{gaussian:geo:headtail}\label{gaussian:geo:spectraltheorem}
\label{gaussian:geo:approxthm}
The event $\mathcal E_M$ satisfies
$\Pr(\mathcal E_M^c)\le6e^{-p_{\rm geo}M}$.
On it, with $\mathsf T=S_R\Sigma_RS_R^\top$,
\begin{equation}\label{gaussian:geo:spectralconstants}
 \begin{gathered}
 cM^{-a}I\preceq\mathsf T\preceq CM^{-a}I,\quad
 \|\Sigma_R\|\le CM^{-a},\quad
 cj^{-a}\le g_j\le Cj^{-a},\\
 c\le\theta_G\le C,\quad \|L_S\|\le C,\quad
 c_AM^{1-b}\le\bar A\le C_AM^{1-b},\qquad
 c_A=\tfrac12c_{\lambda\vartheta}(c_d-1)(c_d+1)^{-b}.
 \end{gathered}
\end{equation}
The lower approximation bound holds for every full-row-rank $S$.
Moreover $L_SS^\top=I$ and $G^{-1/2}S\Sigma^{1/2}$ is a coisometry.
\end{lemma}
\begin{proof}
The split satisfies $n_H\asymp M$, $1\le n_H\le M/4$ and
$d-n_H\ge(1+\delta_{\rm geo})M$.
Apply \citet[Corollary~5.35]{vershynin2011nonasymptotic} to
$\sqrt M S^\top$, $\sqrt M S_H$ and $\sqrt M S_R^\top$, with deviations
$\sqrt M$, $\sqrt M/4$ and $r_{\rm geo}\sqrt M$.
The union bound gives the stated failure probability.
On $\mathcal E_M$, the population envelopes and tail singular-value
bounds give $\Sigma_R\asymp M^{-a}I$, $\mathsf T\asymp M^{-a}I$
and $G\succeq\mathsf T\succ0$.

On the orthogonal complement of the first $j-1$ columns of $S$,
$u^\top Gu\le Cj^{-a}\|u\|^2$. For $j\le n_H$, the first $j$ columns
have smallest singular value at least $c_H$, so their range satisfies
$u^\top Gu\ge cj^{-a}\|u\|^2$.
For $j>n_H$, use $G\succeq\mathsf T$ and $j>\kappa_{\rm geo}M$.
Min--max therefore gives $g_j\asymp j^{-a}$, and summation gives
$\theta_G\asymp1$.
Put $D_H=\Sigma_H^{-1}+S_H^\top\mathsf T^{-1}S_H\succeq cM^aI$.
The identity
$G^{-1}S_H\Sigma_H=\mathsf T^{-1}S_HD_H^{-1}$ bounds the head block
of $L_S$, while
$\|G^{-1}S_R\Sigma_R\|\le\|G^{-1}\|\|S_R\|\|\Sigma_R\|\le C$
bounds its tail block. Thus $\|L_S\|\le C$.
The right-inverse and coisometry identities follow directly from
$G=S\Sigma S^\top$.

Put $P=I-\Sigma^{1/2}S^\top G^{-1}S\Sigma^{1/2}$ and
$D_\theta=\diag(\lambda_j\vartheta_j)$.
For every full-row-rank $S$, $P$ is a rank-$(d-M)$ orthogonal projector,
so
$2\bar A=\tr(PD_\theta)\ge c_{\lambda\vartheta}d^{-b}(d-M)
\ge2c_AM^{1-b}$.
The Woodbury identity gives
\[
 P_{HH}=\Sigma_H^{-1/2}D_H^{-1}\Sigma_H^{-1/2}
 \preceq CM^{-a}\Sigma_H^{-1},
\]
where $P_{HH}$ is the head block of $P$.
Using $P_{jj}\le1$ on the tail and $1<b<a$,
\[
 2\bar A\le CM^{-a}\sum_{j\le n_H}j^{a-b}
       +C\sum_{j>n_H}j^{-b}\le CM^{1-b}.
\]
Fix one such upper constant $C_A$ uniformly over widths and sketches
on the event.
\end{proof}

\begin{corollary}[Deterministic effective-noise scale]\label{risk:noise-comparison}
On the event of Lemma~\ref{gaussian:geo:approxthm},
\begin{equation}\label{eq:noise-scale}
 \sigma^2+2c_A M^{1-b}\le s(S)\le\sigma^2+2C_A M^{1-b}.
\end{equation}
Use $s_M$ from \eqref{eq:clean-law-avg} and put
$c_s=\min(1,2c_A)$, $C_s=\max(1,2C_A)$. Then
\begin{equation}\label{risk:noise-comparison-bounds}
 c_ss_M\le s(S)\le C_ss_M,\qquad
 0\le s(S)/\sigma^2-1\le2C_AM^{1-b}/\sigma^2.
\end{equation}
The same constants compare
$s(S)+\tau^{-(b-1)/a}$ and $s_M+\tau^{-(b-1)/a}$ for every $\tau>0$.
\end{corollary}
\begin{proof}
Add $\sigma^2$ to twice the approximation bounds and use $c_s\le1\le C_s$.
Subtracting $\sigma^2$ and dividing by it proves the second inequality;
adding $\tau^{-(b-1)/a}$ preserves the first comparison.
\end{proof}

\begin{lemma}[Full-matrix teacher source]\label{source:thm:main}
Suppose $S$ has full row rank, $\|S\|\le K_{\rm sk}$, $\|L_S\|\le K_{\rm proj}$,
and $0<\omega<1$. If positive constants $\chi_-,\chi_+$ satisfy
$\chi_-\Sigma^{-\omega}\preceq C_\theta\preceq\chi_+\Sigma^{-\omega}$,
then
\begin{equation}\label{source:eq:main}
 \chi_-K_{\rm sk}^{-2(1-\omega)}G^{-\omega}\preceq V
 \preceq\chi_+K_{\rm proj}^{2(1-\omega)}G^{-\omega}.
\end{equation}
For \eqref{eq:population}, $\omega=\frac{a-b}{a}$ and
$\chi_-=c_{\lambda\vartheta}C_\lambda^{-b/a}$,
$\chi_+=C_{\lambda\vartheta}c_\lambda^{-b/a}$.
Thus $cG^{-(a-b)/a}\preceq V\preceq CG^{-(a-b)/a}$ on $\mathcal E_M$.
\end{lemma}
\begin{proof}
Consider the analytic matrix families
$F(z)=G^{z/2}L_S\Sigma^{-z/2}$ and
$R(z)=\Sigma^{z/2}S^\top G^{-z/2}$ on $0\le\Re z\le1$.
Both are bounded on this strip.
The boundary bounds for $F$ are $K_{\rm proj}$ at $\Re z=0$ and one
at $\Re z=1$; those for $R$ are $K_{\rm sk}$ and one, respectively.
Indeed, imaginary powers are unitary, $F(1)=G^{-1/2}S\Sigma^{1/2}$
is the coisometry from Lemma~\ref{gaussian:geo:headtail}, and
$R(1)=F(1)^*$.
Apply the scalar three-lines theorem to matrix coefficients
\citep[Theorem~7.13]{tropp2022matrix}: it gives
$\|F(\omega)\|\le K_{\rm proj}^{1-\omega}$ and
$\|R(\omega)\|\le K_{\rm sk}^{1-\omega}$.
Since $F(\omega)R(\omega)=I$,
$\|u\|^2\le\|F(\omega)^*u\|\|R(\omega)u\|$ implies
$K_{\rm sk}^{-2(1-\omega)}I\preceq F(\omega)F(\omega)^*
\preceq K_{\rm proj}^{2(1-\omega)}I$.
Since
$F(\omega)F(\omega)^*=G^{\omega/2}L_S\Sigma^{-\omega}L_S^\top G^{\omega/2}$,
congruence by $G^{-\omega/2}$ and the assumed bounds on $C_\theta$
prove \eqref{source:eq:main}.
Finally $\vartheta_j\lambda_j^\omega
=(\lambda_j\vartheta_j)\lambda_j^{-b/a}\in[\chi_-,\chi_+]$
proves the model specialization without commuting $V$ and $G$.
\end{proof}

\begin{lemma}[Simultaneous rounding-coordinate geometry]
\label{haar:lem:measurable}\label{haar:thm:haar}\label{haar:thm:J}
Almost surely $G$ has simple positive spectrum. Its ordered orthonormal
frame $F(G)=(\phi_1,\ldots,\phi_M)$, with each column's first nonzero
coordinate positive and extended by $I$ off this set, is Borel.
For independent Rademacher signs $\epsilon_i$, also independent of $S$,
$O=F(G)\diag(\epsilon)$ is Haar,
independent of the spectrum and every rotation-invariant Borel sketch statistic.
For $O\in O(M)$ define the exact tests
\begin{equation}\label{haar:eq:Haarevent}
 \mathcal H_M=\bigcap_{i,j}
 \left\{\frac1{2M}\le\sum_kO_{ki}^2O_{kj}^2\le\frac{16}{M},\quad
                 \sum_kO_{ki}^2O_{kj}^4\le\frac{16}{M^2}\right\}.
\end{equation}
The event $\Omega_M^H=\{G\text{ simple positive},F(G)\in\mathcal H_M\}$
is Borel. For each $D>0$ it fails with probability at most $K_DM^{-D}$.
On it, for every positive spectrum,
\begin{equation}\label{haar:eq:Jbounds}
 1/2\le\phi_j^\top J\phi_j\le16,\qquad
 \phi_j^\top J^2\phi_j\le16\quad(1\le j\le M).
\end{equation}
\end{lemma}
\begin{proof}
A leading $M$-column minor of $S$ and
the spectral discriminant of $G$ are nonzero polynomials, as shown by
$S=[\diag(\sqrt{j/\lambda_j})_{j\le M}\ 0]$.
Hence $G$ has simple positive spectrum almost surely.
On the simple positive set, the continuous spectral projector
$P_j(A)=\prod_{i\ne j}(A-g_i(A)I)/(g_j(A)-g_i(A))$
gives the Borel column $P_j(A)e_k/\sqrt{P_j(A)_{kk}}$, choosing its first
positive diagonal entry $k$. For each orthogonal $U$,
$F(UGU^\top)=UF(G)D_U(G)$ for diagonal signs $D_U$.
The Gaussian density is left-rotation invariant. Averaging the column
signs gives $(Z(S),UO)\overset d=(Z(S),O)$ for every invariant $Z$.
Integrating $U$ against Haar probability proves the Haar law and independence
\citep[Theorem~1.4]{meckes2019classical}.

In the Gaussian Gram--Schmidt construction
\citep[Section~1.2]{meckes2019classical}, distinct Haar columns are
$u=x/\|x\|$, $v=(y-\gamma x)/\|y-\gamma x\|$, for independent
standard Gaussian vectors $x,y$. Set
$A_{rs}=M^{-1}\sum_kx_k^ry_k^s$, $\gamma=A_{11}/A_{20}$ and
$b_\perp=A_{02}-A_{11}^2/A_{20}$.
Let $\mathcal P_M$ be the event that
$|A_{rs}-(\E Z^r)(\E Z^s)|\le1/100$ for every $r+s\le6$, $Z\sim N(0,1)$.
For any integer $r\ge1$ and independent copies of a centered $Y_1$
with $\E|Y_1|^{2r}<\infty$, expansion and H\"older give
\begin{equation}\label{haar:lem:moment}
 \E\left|M^{-1}\sum_kY_k\right|^{2r}
 \le(2r)^{2r}\E|Y_1|^{2r}M^{-r}.
\end{equation}
In the expansion, singleton indices vanish; at most $(2r)^{2r}M^r$
terms remain, each bounded by $\E|Y_1|^{2r}$ by H\"older.
Gaussian monomials have every finite moment. Markov and a union bound over
28 such monomials therefore give arbitrarily large polynomial decay and
\begin{equation}\label{haar:eq:KD}
 K_D=1+\sup_{n\ge2}n^{D+2}\Pr(\mathcal P_n^c)<\infty.
\end{equation}
On $\mathcal P_M$, $A_{20},A_{02}\in[.99,1.01]$,
$|\gamma|\le1/99$ and $b_\perp\in[98/99,1.01]$. The exact ratios are
\begin{align*}
 M\sum_ku_k^2v_k^2&=
 \frac{A_{22}-2\gamma A_{31}+\gamma^2A_{40}}{A_{20}b_\perp},\\
 M^2\sum_ku_k^2v_k^4&=
 \frac{A_{24}-4\gamma A_{33}+6\gamma^2A_{42}
                    -4\gamma^3A_{51}+\gamma^4A_{60}}{A_{20}b_\perp^2}.
\end{align*}
Substituting the bounds defining $\mathcal P_M$ shows that the first
ratio lies in $(1/2,2)$ and the second is below $4$.
For a repeated column, the ratios are
$A_{40}/A_{20}^2\in(1/2,4)$ and $A_{60}/A_{20}^3<16$.
Thus the corresponding pair tests hold. A union bound over $M^2$ pairs
with \eqref{haar:eq:KD} proves the failure bound; for $M=1$ the tests
hold identically. They are invariant under column signs, so the event is
Borel in the sketch alone.

With $\pi_i=g_i/\theta_G$,
$\phi_j^\top J\phi_j=M\sum_i\pi_i\sum_kO_{ki}^2O_{kj}^2$.
Jensen also gives
$\phi_j^\top J^2\phi_j\le M^2\sum_i\pi_i\sum_kO_{kj}^2O_{ki}^4$.
The pair tests prove \eqref{haar:eq:Jbounds}.
Because the eigenframe is independent of rotation-invariant sketch
statistics, the same probability bound holds conditional on any such
statistic, including the event $\mathcal E_M$.
\end{proof}

Define  
\begin{equation}\label{as:events}
 \begin{gathered}
 \Omega_M=\mathcal E_M\cap\Omega_M^H, \quad
 \mathcal J_M^\pm=\left\{(S,\epsilon):
 \begin{gathered}
 S\in\mathcal E_M,\quad G\text{ simple positive},\\
 F(G)\diag(\epsilon)\in\mathcal H_M,\quad\epsilon\in\{-1,1\}^M
 \end{gathered}\right\}.
 \end{gathered}
\end{equation}
All tests are invariant under column signs. Thus the consequences of
the Haar calculation hold for the actual eigenvectors of every fixed
$S\in\Omega_M$. The training proof introduces no additional random frame.

\begin{definition}[Fixed constants]
\label{as:def-witness}\label{wp:def:witnessclass}
For the fixed population family and $M_{\rm w}$, let $\mathcal W$ be the class of
positive tuples
$W=(c_g,C_g,C_\theta^{\rm tr},c_V,C_V,C_A^{\rm sch},c_J,C_J)$,
with $C_\theta^{\rm tr},C_J\ge1$, depending only on the fixed population
family. A tuple is valid when, for every $M\ge M_{\rm w}$
and every $(S,\epsilon)\in\mathcal J_M^\pm$,
\begin{equation}\label{as:U1}
 \begin{gathered}
 c_gj^{-a}\le g_j\le C_gj^{-a},\quad\theta_G\le C_\theta^{\rm tr},\\
 c_VG^{-(a-b)/a}\preceq V\preceq C_VG^{-(a-b)/a},\\
 0\le\bar A\le C_A^{\rm sch}M^{1-b},\quad
 c_J\le\phi_j^\top J\phi_j\le C_J,\quad
 \phi_j^\top J^2\phi_j\le C_J,\\
 \phi_j=F(G)\diag(\epsilon)e_j.
 \end{gathered}
\end{equation}
Fix one valid tuple $W\in\mathcal W$ before choosing the learning-rate
and precision schedules, and write
$\Theta=C_\theta^{\rm tr}$. The constants below refer to this chosen tuple.
\end{definition}
The preceding lemmas show that $\mathcal W$ is nonempty: combine their
spectral and teacher bounds with $c_J=1/2$, $C_J=16$, increasing the
upper constants to at least one where required. The geometric constant
$C_A$ and the selected tuple component $C_A^{\rm sch}$ need not be equal.

\begin{proposition}[Common sketch event]
\label{as:prop-event}\label{wp:prop:transport}
For every $M\ge M_{\rm w}$, $\Omega_M$ in \eqref{as:events} is independent of the
learning-rate and precision schedules. Its definition does not depend on
$D$, and every valid tuple satisfies \eqref{as:U1} on it.
For each fixed $D>0$,
\begin{equation}\label{as:event-probability}
 \Pr(\Omega_M^c)\le6e^{-p_{\rm geo}M}+K_DM^{-D}
 \le[K_D+6(D/(ep_{\rm geo}))^D]M^{-D}.
\end{equation}
\end{proposition}
\begin{proof}
The Borel property follows from the preceding constructions. A union
bound combines their failure probabilities; independence of the two
events is unnecessary. The final inequality follows from
$\sup_{x>0}x^De^{-p_{\rm geo}x}=(D/(ep_{\rm geo}))^D$.
\end{proof}
Finally, the unconditional lower approximation estimate and the chosen
upper comparison give
\begin{equation}\label{as:approx-lower}
 c_AM^{1-b}\le\bar A\le C_A^{\rm sch}M^{1-b}.
\end{equation}
\endgroup

\section{General learning-rate and precision schedules}
\label{app:actual-schedules}\label{sec:rev3_schedules}
\label{sec:rev3_schedule_interface}\label{sec:rev3_schedule_closure}
\label{sec:rev3_native_bridge}
This appendix bounds the second moments that determine prediction risk
for general learning-rate and precision schedules.
Theorem~\ref{as:theorem-clean} gives risk identities and exact-arithmetic
moment bounds, and Lemma~\ref{sc:th:primary} controls the additional state
moments caused by rounding. Theorem~\ref{sc:th:output} and
Proposition~\ref{as:prop-native-lower} give matching bounds for the
SR-induced output risk, which Appendix~\ref{cl:finite} converts into
Theorem~\ref{thm:scheduled-main}. Appendix~\ref{app:forest} constructs the
admissibility thresholds used throughout.
In this appendix, condition on a sketch $S$ in the event of
Proposition~\ref{as:prop-event}.
Write $\Theta=C_\theta^{\rm tr}$ and $\tau_*=\tau_*^{\rm sch}$.
The symbols $c,C$ denote positive constants depending only on the fixed
model, the comparison tuple of Definition~\ref{as:def-witness}, 
$c_0$, the minimum cumulative learning rate $\tau_*$ and the positive noise variance.
They do not depend on width,
horizon, batch size or the learning-rate and precision schedules. 
Their construction is given in Appendix~\ref{app:forest}.

\begingroup
\emergencystretch=3em
\let\E\relax\newcommand{\E}{\mathbb E}
\let\tr\relax\newcommand{\tr}{\operatorname{tr}}
\let\Conv\relax\newcommand{\Conv}{\operatorname{Conv}}
\let\rem\relax\newcommand{\rem}{\operatorname{Rem}}
\let\Abar\relax\newcommand{\Abar}{\bar A}
\let\T\relax\newcommand{\T}{\mathsf T}
\let\N\relax\newcommand{\N}{\mathcal N}
\let\Worig\relax\newcommand{\Worig}{W}
\let\Lop\relax\newcommand{\Lop}{\mathsf L}
\theoremstyle{plain}
\subsection{Schedules, output responses and admissibility}\label{app:model}
We define the output responses and feedback bounds used in
Assumption~\ref{as:assumption-schedules}.
Use the model and independence assumptions of Section~\ref{sec:setup}
and the auxiliary matrices and moments in \eqref{app:eq:oracle-geometry},
\eqref{eq:rounding-geometry}, \eqref{eq:moment-sources} and
\eqref{eq:moment-gap}.
Define the exponents and propagation products along the eigenvectors of $G$ by
\begin{gather}
 \omega_V=(a-b)/a,\qquad \alpha=1/a,\qquad \nu=(a-b+1)/a,
 \label{eq:source-exponents}\\
 \pi_j(t,u)=\prod_{v=u}^{t-1}(1-\eta_vg_j),\qquad \pi_j(t,t)=1.
 \label{as:products}
\end{gather}
The parameter range $1<b<a$ gives $0<\omega_V<1$ and
$0<\alpha<\nu<1$.

\paragraph{Mean propagation and output responses.}
Use $H_u,P_{u\to t},P_t$ from \eqref{arithmetic:eq:products}.
Using the eigenpairs of \eqref{app:eq:oracle-geometry} and the output weights of
\eqref{app:eq:output-weights}, define
\begin{equation}\label{eq:schedule-filters}
\begin{aligned}
 d_{\iout,j}=\sum_{t=1}^Tw_t^{\iout}\pi_j(t,0),\qquad
 \ell_{\iout,j,u}=\sum_{t=u+1}^Tw_t^{\iout}\pi_j(t,u+1).
\end{aligned}
\end{equation}
Empty products equal one. An error introduced at update $u$ first enters
state $u+1$. For positive semidefinite $A$, its direct output response is
\begin{equation}\label{eq:output-response}
 \mathcal R_u^{\iout}(A)=\tfrac12\sum_{j=1}^M
             g_j\ell_{\iout,j,u}^2\phi_j^\top A\phi_j.
\end{equation}
This is the prediction variance produced by covariance $A$ under the
subsequent mean updates and output averaging. Stochastic feedback is
accounted for separately below. Write $P_{\rm out}=\sum_{t=1}^T w_t^{\rm out}P_t$ and
$L_{T,u}=\sum_{t=u+1}^T w_t^{\rm out}P_{u+1\to t}$. Equivalently,
\begin{equation}\label{as:output-filter}
 \mathcal R_u^{\rm out}(A)=\tfrac12\tr(GL_{T,u}AL_{T,u}^{\T}).
\end{equation}
The response is nonnegative for $A\succeq0$. Consider $\iout\in\{\mathrm{last},\mathrm{avg}\}$.
Define the bias and SGD-variance functionals by
\begin{equation}\label{eq:scheduled-clean-functionals}
 \mathcal B^{\iout}=\tfrac12\sum_{j=1}^M g_j^{b/a}d_{\iout,j}^2,
 \qquad
 \mathcal V^{\iout}=\sum_{u<T}h_uR_u\mathcal R_u^{\iout}(G),
\end{equation}
where $h_u,R_u$ are given in \eqref{eq:clock} and
\eqref{eq:moment-sources}. On the sketch event,
Definition~\ref{as:def-witness} compares $\mathcal B^{\iout}$ with
the exact mean bias
$\frac12\sum_jg_jd_{\iout,j}^2\phi_j^\top V\phi_j$.

\paragraph{Rounding sources and covariance comparisons.}
Use the lower and upper rounding coefficients
$f_u^\epsilon,\beta_u^\epsilon,m_u^\epsilon$ of
\eqref{eq:rounding-source-coefficients} and \eqref{arithmetic:eq:coefficients}.
For brevity, write $c_f=c_b=3/4$ and $c_m=3/16$ for coefficients in \eqref{eq:rounding-source-coefficients}. The exact-arithmetic comparison uses the
same width, learning rates, batch size and output. No commutation of
$G,V,J$ is assumed.
For $\epsilon\in\{-,+\}$,
write $x_t^\epsilon=\tr(GX_t^\epsilon)$ and
$y_t^\epsilon=\tr(JX_t^\epsilon)$, where $X_t^\epsilon$ is the comparison
second-moment gap from Appendix~\ref{app:actual-scaling}.
Write $Q_u^\epsilon=Q_u^\epsilon(X_u^\epsilon)$, with the affine map defined in
\eqref{core:comparison-innovations}. Its direct source is
\begin{equation}\label{as:comparison}
 D_u^\epsilon=f_u^\epsilon U_uG+(\beta_u^\epsilon R_u+m_u^\epsilon U_u)J.
\end{equation}
For the exact-arithmetic recursion, write
\begin{equation}\label{as:innovations}
 N_u=R_uG+GC_u^0G.
\end{equation}
A missing comparison superscript means plus. These definitions also apply
at zero learning rates and at exact precision, since they involve no
division by $h_u$ or a rounding coefficient.
For $u<t$, define the propagation kernels in a $G$-eigenbasis:
\begin{gather}
 \mathcal A_u(Z)=H_uZH_u+h_uGZG,\qquad
 a_{ij,u}=(1-\eta_ug_i)(1-\eta_ug_j)+h_ug_ig_j,\notag\\
 D_{ij}(t,u)=\prod_{v=u+1}^{t-1}a_{ij,v},\quad
 d_j(t,u)=D_{jj}(t,u),\quad\ell_{t,u}=\tau_t-\tau_{u+1},\notag\\
 S_\ell^{\mathcal A}(t,u)=\sum_jg_j^\ell d_j(t,u),\quad
 S_\ell^H(t,u)=\sum_jg_j^\ell\pi_j(t,u+1)^2,\notag\\
 K_{PQ}^{\mathcal A}(t,u)=\tr(P\mathcal A_{t-1}\cdots\mathcal A_{u+1}(Q)),
 \notag\\
 K_{PQ}^H(t,u)=\tr(PP_{u+1\to t}QP_{u+1\to t}),\qquad P,Q\in\{G,J\}.\label{si:eq:Akernels}
\end{gather}
Empty products of scalar factors and of linear maps are, respectively,
one and the identity map. Put
$A_u^\epsilon=x_u^\epsilon+f_u^\epsilon(U_u+y_u^\epsilon)$ and
$B_u^\epsilon=\beta_u^\epsilon(R_u+x_u^\epsilon)+m_u^\epsilon(U_u+y_u^\epsilon)$.
Thus the comparison recursion is
$X_{u+1}^\epsilon=\mathcal A_u(X_u^\epsilon)+h_u(A_u^\epsilon G+B_u^\epsilon J)$. 
The two sources coupling different rounding locations are
\begin{equation}\label{eq:scheduled-conversion-sources}
 Z_t^{b,\epsilon}=\sum_{u<t}h_u\beta_u^\epsilon R_uK_{JJ}^H(t,u),\qquad
 Y_t^{f,\epsilon}=\sum_{u<t}h_uf_u^\epsilon U_uS_2^H(t,u).
\end{equation}
The first measures the parameter second moment contributed by earlier
backward rounding; the second measures the residual second moment
contributed by earlier forward rounding. Multiplying them by a later
forward or backward rounding coefficient produces the two temporal
interaction terms. All sums are strictly causal: a source introduced at
update $u$ first enters state $u+1$.

\paragraph{Risk contributions from rounding.}
Define $\Psi_\epsilon^{\rm out}=\sum_{u<T}h_u\mathcal R_u^{\rm out}(D_u^\epsilon),$
\begin{align}
 \Conv_{b\to f}^{\epsilon,\rm out}
   &=\sum_{t<T}h_tf_t^\epsilon Z_t^{b,\epsilon}\mathcal R_t^{\rm out}(G),\quad
 \Conv_{f\to b}^{\epsilon,\rm out}
   =\sum_{t<T}h_t\beta_t^\epsilon Y_t^{f,\epsilon}\mathcal R_t^{\rm out}(J),
 \label{as:output-sources}\\
 \N^{\rm out}&=\Psi_+^{\rm out}+\Conv_{b\to f}^{+,\rm out}
                              +\Conv_{f\to b}^{+,\rm out}.
 \label{as:native}
\end{align}
The direct term $\Psi_\epsilon$ contains forward rounding, backward
rounding and their same-update interaction through $D_u^\epsilon$.
The other two terms account for interactions across updates.
The sum $\N^{\rm out}$ is a covariance comparison functional for the
SR-induced error, rather than an exact variance identity.
A missing comparison superscript means the upper coefficient; all terms
are nonnegative and strictly causal. For positive cumulative learning
rates, define the residual comparison
\begin{equation}\label{as:eq:residual-scale}
 r_t=s+\tau_t^{-(b-1)/a},\qquad \tau_t>0, \quad t\neq 0.
\end{equation}
The event $\Omega_M$ and width threshold $M_{\rm w}$ are defined in
Proposition~\ref{as:prop-event}. On this event, the fixed tuple of
Definition~\ref{as:def-witness} implies \eqref{as:U1} and the
approximation bounds \eqref{as:approx-lower}.
Appendix~\ref{app:forest} fixes $\tau_*,\rho_{\rm sch},\kappa_{\rm sch}$
from this tuple, $c_0$ and the model, before any schedule is selected.
Write $\bar\eta=c_0/c_g$, $\rho_{\rm src}=1/(4\Theta)$ and
$\kappa=\kappa_{\rm sch}$.

\paragraph{Feedback quantities.}
For any finite schedule at a full-row-rank sketch define, even at zero clock,
\begin{gather}
 \Gamma_{\rm stab}(T)=\sum_{u<T}h_uq_{{\rm fw},u}q_{{\rm bw},u}/M,\notag\\
 \Gamma_{f\to b}(T)=\max_{t<T:q_{{\rm bw},t}>0}
 \frac{\sum_{u<t}h_uq_{{\rm fw},u}U_uS_2^H(t,u)}{MR_t+q_{{\rm fw},t}U_t},
 \label{as:allclock-gains}
\end{gather}with $\max\varnothing=0.$
For a time horizon with $\tau_T\ge\tau_*$, use $r_T\neq 0$ from \eqref{as:eq:residual-scale} and define
\begin{gather}
 \Gamma_\times(T)=\max_{t<T}
 \frac{q_{{\rm fw},t}}{Mr_T}\sum_{u<t}h_uq_{{\rm bw},u}R_u,\notag\\
 \Gamma_{\rm conv}(T)=\max(\Gamma_\times(T),\Gamma_{f\to b}(T)).
 \label{as:window-gains}
\end{gather}
All maxima and source sums are finite and all denominators are positive.

\begin{assumption}[Admissible learning-rate and precision schedules]
\label{as:assumption-schedules}\label{as:def-class}
The learning-rate and precision schedules are deterministic and chosen
before all random draws. For the horizon $T$, they satisfy the
rate and clock conditions \eqref{eq:scheduled-class} for $0\le u<T$
(with $\tau_*=\tau_*^{\rm sch}$), together with
\begin{equation}\label{as:class}
 \Gamma_{\rm stab}(T)\le\kappa,\qquad
 \Gamma_{\rm conv}(T)\le\kappa.
\end{equation}
Such a schedule and time horizon are called admissible. 
\end{assumption}

\paragraph{Consequences of admissibility.}
Choose $\rho_{\rm sch}\le\rho_{\rm src}$. On $\Omega_M$,
$g_1\ge c_g$ and $\theta_G\le\Theta$, so admissibility gives
\begin{equation}\label{as:inheritance}
 \eta_u\le\bar\eta,\qquad
 \rho_u\le\rho_{\rm sch}\le\rho_{\rm src},\qquad
 \rho_u\theta_G\le\tfrac14.
\end{equation}
For $M^a\ge\tau_*$, sufficiently many small, equal exact-arithmetic
steps give an admissible schedule with total learning rate $\tau_*$.
Unrolling the comparison recursion gives
\begin{equation}\label{si:eq:volcomp}
 \binom{x_t^\epsilon}{y_t^\epsilon}
 =\sum_{u<t}h_u
 \begin{pmatrix}K_{GG}^{\mathcal A}&K_{GJ}^{\mathcal A}\\
 K_{JG}^{\mathcal A}&K_{JJ}^{\mathcal A}\end{pmatrix}(t,u)
 \binom{A_u^\epsilon}{B_u^\epsilon}.
\end{equation}
All kernels and sources are nonnegative.
\subsection{Spectral propagation and causal compositions}
Lemma~\ref{si:lem:products} compares propagation and spectral sums with
scalar decay functions used in the subsequent risk bounds.
Use the propagation products, kernels and spectral sums in
\eqref{as:products} and \eqref{si:eq:Akernels}. For $v\ge0$, write
$\varphi_0(v)=M$ and $\varphi_i(v)=(1+v)^{\alpha-i}$ for $i=1,2$.
Set $\Lambda_0=(1-c_0)^{-1}$, and abbreviate
$d_j^H=\pi_j(t,u+1)^2$ and $d_j^{\mathcal A}=d_j(t,u)$.
For $u<\min(v,w)$ and $i=1,2$, define the paired sums
\begin{equation}\label{sc:eq:pairdef}
 \Lambda_i^{\rm pair}(u;v,w)=\sum_jg_j^i\pi_j(v,u+1)\pi_j(w,u+1),\qquad
 d=\tau_v+\tau_w-2\tau_{u+1}.
\end{equation}
Let $a_j(x)$ denote either $e^{-g_jx}$ or a product of two propagation
factors along the same eigenvector, whose intervals have total cumulative learning rate $x$.
Empty sums and maxima are zero.

\begin{lemma}[Propagation, spectral sums and causal composition]
\label{si:lem:products}\label{si:prop:modal}\label{si:th:spectral}
\label{si:lem:telescoping}\label{si:prop:composition}
\label{cl:spectral}\label{sc:lem:pairspectral}
Assume \eqref{as:U1} and $0\le\eta_ug_1\le c_0$.
The constants below depend only on the fixed model,
comparison bounds and $c_0$.
\begin{enumerate}[label=(\roman*),leftmargin=*]
\item \emph{Products and kernel geometry.}
For every cumulative-time interval,
\begin{equation}\label{sc:eq:productnew}
 e^{-\Lambda_0g_j(\tau_t-\tau_u)}\le\pi_j(t,u)
 \le e^{-g_j(\tau_t-\tau_u)},\qquad
 1-\pi_j(t,0)\le\min(1,g_j\tau_t).
\end{equation}
For $\mathsf X\in\{H,\mathcal A\}$, suppressing $(t,u)$,
\begin{gather}
 K_{GG}^{\mathsf X}=S_2^{\mathsf X},\qquad
 K_{GJ}^{\mathsf X}=K_{JG}^{\mathsf X}
 =\sum_jg_jJ_{jj}d_j^{\mathsf X},\notag\\
 c_JS_1^{\mathsf X}\le K_{GJ}^{\mathsf X}\le C_JS_1^{\mathsf X},\qquad
 c_J^2S_0^{\mathsf X}\le K_{JJ}^{\mathsf X}\le C_JS_0^{\mathsf X}.
 \label{si:eq:modalkernels}
\end{gather}

\item \emph{Spectral comparisons.}
There are $0<c_P\le1\le C_P$ such that
\begin{gather}
 c_P\varphi_i(\ell_{t,u})\le S_i^{\mathsf X}(t,u)
       \le C_P\varphi_i(\ell_{t,u}),\label{si:eq:explicitsp}\\
 K_{PQ}^H\le K_{PQ}^{\mathcal A}\le C_3K_{PQ}^H
       \quad(P,Q\in\{G,J\}),\label{si:eq:C3}
\end{gather} 
for $0\le\ell_{t,u}\le M^a$ and $i=0,1,2$.
The upper profile bounds hold for all nonnegative gaps.
For $r\in\{1,2,b/a,2-\omega_V\}$, there are
$0<\ell_r\le u_r<\infty$ with
\begin{equation}\label{cl:heat}
 \ell_r(1+x)^{\alpha-r}\le\sum_jg_j^r a_j(x)
                         \le u_r(1+x)^{\alpha-r}.
\end{equation}
Its lower bound requires $x\le2M^a$ and its upper bound is global.
Set
\begin{equation}\label{sc:eq:KS}
 K_S=\max\{2^{2-\alpha}C_P,c_P^{-1}\}.
\end{equation}
Then, for $i=1,2$,
\begin{equation}\label{sc:eq:pairsandwich}
 K_S^{-1}\varphi_i(d)\le\Lambda_i^{\rm pair}(u;v,w)
                           \le K_S\varphi_i(d)
\end{equation}
where the lower bound requires $d\le2M^a$.
In particular, for $\ell_{t,u}\le M^a$,
\begin{equation}\label{cl:JJ}
 c_J^2c_PM\le K_{JJ}^H(t,u)\le C_JM.
\end{equation}

\item \emph{Telescoping and weighted sums.}
At every finite cumulative time,
\begin{gather}
 \sum_{u<t}\eta_ug_j\pi_j(t,u+1)=1-\pi_j(t,0),\qquad
 \sum_{u<t}\eta_ug_j\pi_j(t,u+1)^2\le1-\pi_j(t,0),
 \label{si:eq:Htel}\\
 \sum_{u<t}h_uS_2^H(t,u)\le\theta_G\max_{u<t}\rho_u.
 \label{si:eq:forwardtail}
\end{gather}

\item \emph{Causal sums and compositions.}
Let $D=\ell_{t,u}\le M^a$ and
$\mathsf X,\mathsf Y,\mathsf Z\in\{H,\mathcal A\}$.
For $(i,j,m)\in\{(2,2,2),(2,1,1),(1,2,1)\}$,
\begin{gather}
 \sum_{u<v<t}\eta_vS_1^{\mathsf X}(t,v)\le C_KM,\qquad
 \sum_{u<v<t}\eta_vS_1^{\mathsf X}(v,u)\le C_KM,
 \label{si:eq:schedexp-forwardmass}\\
 \sum_{u<v<t}\eta_vS_2^{\mathsf X}(t,v)\le C_K, \quad \sum_{u<v<t}\eta_vS_i^{\mathsf X}(t,v)S_j^{\mathsf Y}(v,u)
 \le C_KS_m^{\mathsf Z}(t,u),
 \label{si:eq:CKexports}\\
 \sum_{u<v<t}\eta_vS_1^{\mathsf X}(t,v)S_1^{\mathsf Y}(v,u)
 \le C_K\min\{MS_1^{\mathsf Z},S_0^{\mathsf Z}\}(t,u).
 \label{si:eq:schedexp-min11}
\end{gather}
The fixed-terminal one-kernel bounds also hold over the full history
$v<t$ when $\tau_t\le M^a$.
Replacing $\eta_v$ by $h_v$ multiplies each bound by the maximum
$\rho_v$ over its summation indices.
\end{enumerate}
\end{lemma}
\begin{proof}
In eigenbasis of $G$, diagonal action gives
\begin{equation}\label{si:eq:kernelentry}
\begin{aligned}
 \mathcal A_{t-1}\cdots\mathcal A_{u+1}(Q)
 &=\sum_{i,j}Q_{ij}D_{ij}(t,u)\phi_i\phi_j^{\T}, \quad
 K_{PQ}^{\mathcal A}=\sum_{i,j}P_{ij}Q_{ij}D_{ij}.
\end{aligned}
\end{equation}
For mean propagation, replace $D_{ij}$ by
$\pi_i(t,u+1)\pi_j(t,u+1)$.
The vectors $(1-\eta_vg_i,\sqrt{h_v}g_i)$ give, by
Cauchy--Schwarz,
\begin{equation}\label{si:eq:productbounds}
 0\le\pi_i\pi_j\le D_{ij}\le\sqrt{d_id_j}\le(d_i+d_j)/2.
\end{equation}
For $x=\eta_vg_j\in[0,c_0]$,
$e^{-\Lambda_0x}\le1-x\le e^{-x}$ and
$(1-x)^2\le a_{jj,v}=1-2x+(1+B^{-1})x^2\le e^{-2(1-c_0)x}$.
Multiplication and $1-\prod_v(1-x_v)\le\sum_vx_v$ give
\eqref{sc:eq:productnew} and
\begin{equation}\label{si:eq:damping}
 e^{-2\Lambda_0g_j\ell_{t,u}}\le d_j(t,u)
       \le e^{-2(1-c_0)g_j\ell_{t,u}}.
\end{equation}
The first kernel identities follow by taking traces. For $JJ$, use
\[
 c_J^2\sum_jd_j\le\sum_jJ_{jj}^2d_j
 \le\sum_{i,j}J_{ij}^2D_{ij}
 \le\sum_i(J^2)_{ii}d_i\le C_J\sum_id_i.
\]
This argument also applies to mean propagation.

All required powers satisfy $0<r\le2$ and $ar>1$.
For $x\ge1$, bound
$g_j^re^{-xg_j}\le\min\{x^{-r},C_g^rj^{-ar}\}$ and compare the
decreasing sum with its integral. For $x\le1$, sum $C_g^rj^{-ar}$.
These calculations give the global upper order $(1+x)^{\alpha-r}$.
For a lower bound with $x\le2M^a$, set
$k=\max\{1,(x/2)^\alpha\}\le M$.
The interval $[k/2,k]$ contains at least $k/4$ integer indices (use the first index
if $k<2$), so
\[
 \sum_jg_j^re^{-\Lambda_0xg_j}
 \ge\frac{c_g^re^{-2^{a+1}C_g\Lambda_0}}4
                  (1+x/2)^{\alpha-r}.
\]
The last half of the eigenvalue indices gives the zero-order lower coefficient
$\frac12e^{-2^{a+1}C_g\Lambda_0}$.
The product bounds now give \eqref{cl:heat} and
\eqref{si:eq:explicitsp}, with the common constants $c_P,C_P$.
Together with the bounds in the eigenbasis of $G$, the factors
$1,C_J/c_J,C_J/c_J^2$ give
$C_3=(C_P/c_P)\max(1,C_J/c_J,C_J/c_J^2)$ and \eqref{cl:JJ}.
A paired product lies between $e^{-\Lambda_0g_jd}$ and $e^{-g_jd}$.
Apply the same sum and integral estimates to these exponential sums
at clock $d/2$, and use
$\varphi_i(d)\le\varphi_i(d/2)\le2^{i-\alpha}\varphi_i(d)$.
This proves \eqref{sc:eq:pairsandwich} with \eqref{sc:eq:KS}.

The identity
$\eta_ug_j\pi_j(t,u+1)=\pi_j(t,u+1)-\pi_j(t,u)$
gives \eqref{si:eq:Htel}, since $0<\pi_j\le1$.
Multiply its inequality by $g_j$, sum over eigenvector indices and use
$h_u\le(\max_{v<t}\rho_v)\eta_u$ to obtain
\eqref{si:eq:forwardtail}.

Put $A=1+\bar\eta$ and choose the fixed coefficients $F_1=A/\alpha, F_2=A^{2-\alpha}/(1-\alpha), 
 C_{\rm sc}=8A^2\max(F_1,F_2),$ and 
\begin{gather}
 C_K=\max\{C_PF_2,2^\alpha C_PF_1,
                    2^\alpha C_P^2C_{\rm sc}/c_P\}.
 \label{si:eq:cellconstants}
\end{gather}
On $[b,b+\eta]$,
$\varphi_i(b)\le A^{i-\alpha}\varphi_i(s)$.
Integrating over disjoint intervals indicates
\begin{equation}\label{si:eq:F12}
 \sum_v\eta_v\varphi_1(b_v)\le F_1(1+D)^\alpha,\qquad
 \sum_v\eta_v\varphi_2(b_v)\le F_2.
\end{equation}
For a composition, let $a_v=\tau_t-\tau_{v+1}$ and
$b_v=\tau_v-\tau_{u+1}$. Then
$a_v+b_v+\eta_v=D$, hence
$1+\max(a_v,b_v)\ge(1+D)/(2A)$.
When $a_v\ge b_v$, use
$\varphi_i(a_v)\le(2A)^{i-\alpha}\varphi_i(D)$ and apply
\eqref{si:eq:F12} to the other factor; interchange the roles otherwise.
The resulting bounds are
\begin{equation}\label{si:eq:scalconv}
 \sum_{u<v<t}\eta_v\varphi_i(a_v)\varphi_j(b_v)
 \le C_{\rm sc}
 \begin{cases}
 \varphi_2(D),&(i,j)=(2,2),\\
 \varphi_1(D),&(i,j)=(2,1),(1,2),\\
 (1+D)^{2\alpha-1},&(i,j)=(1,1).
 \end{cases}
\end{equation}
The mixed cases use $2\alpha-2\le\alpha-1$.
Substitute the common spectral profiles. For the last case, use
$(1+D)^\alpha\le2^\alpha M$ and $c_PM\le S_0$ to obtain both
alternatives in \eqref{si:eq:schedexp-min11}.
The same cell integrals give the one-kernel bounds and their
full-history versions. Finally $h_v=\rho_v\eta_v$ gives the weighted
bounds, including zero-rate steps.
\end{proof}

\subsection{Exact-arithmetic risk and source moments}
We use standard second-moment methods for linear SGD
\citep{wu2022last,lin2024scaling}. In the eigenbasis of $G$, the diagonal entries
of the exact-arithmetic recursion in Proposition~\ref{arithmetic:prop:exact}
obey the recurrence of \citet[Theorem~1]{bordelon2026optimal}.
Theorem~\ref{as:theorem-clean} expresses prediction risk through propagated
contributions from the initial error, minibatch sampling and rounding.
It also bounds the exact-arithmetic moments that enter the subsequent
rounding-error analysis.

\Needspace{8\baselineskip}
\begin{theorem}[Risk identities and exact-arithmetic source bounds]
\label{as:theorem-clean}\label{as:lemma-identities}\label{si:th:I1}
\label{si:prop:cleandecomp}\label{si:lem:maxprinciple}
\label{si:lem:unimodal}\label{si:prop:rawsources}\label{si:th:S1}
For \eqref{eq:output-weights}, the following hold.
\begin{enumerate}
\item \emph{Risk identities.}
At a fixed full-row-rank sketch, every finite deterministic schedule satisfies
\begin{gather}
 E_0^{\rm out}=\tfrac12\tr(GP_{\rm out}VP_{\rm out}^{\T})
                   +\sum_{u<T}h_u\mathcal R_u^{\rm out}(N_u),
 \label{as:clean-identity}\\
 \Delta^{\rm out}=\sum_{u<T}h_u\mathcal R_u^{\rm out}(Q_u)\ge0.
 \label{as:actual-identity}
\end{gather}
The first identity separates the squared conditional mean error from
the variance caused by random SGD updates, with the teacher averaged afterward. $E_0^{\rm out}$ is defined by \eqref{eq:estimation-risk}.

\item \emph{Risk comparison.}
If $c_VG^{-\omega_V}\preceq V\preceq C_VG^{-\omega_V}$, then
\begin{equation}\label{as:clean-coefficientwise}
 c_V\mathcal B^{\rm out}+\mathcal V^{\rm out}\le E_0^{\rm out}
                   \le C_V\mathcal B^{\rm out}+2\mathcal V^{\rm out}.
\end{equation}

\item \emph{Source bounds.}
Under \eqref{as:U1}, $0\le\eta_ug_1\le c_0<1$ and
$\rho_u\le\rho_{\rm src}$, exact-arithmetic moments satisfy
\begin{gather}
 s\le R_u\le K_s^{\rm ref}s,\qquad
 K_s^{\rm ref}=C_R^{\rm raw}(1+\sigma^{-2}),
 \label{eq:reference-history-factor}\\
 U_t\le C_U^{\rm raw}(1+\tau_t)^\nu,
 \label{si:eq:Uupper}\\
 U_t\ge c_U\tau_t^\nu\qquad(1\le\tau_t\le M^a),
 \label{si:eq:Ulower}\\
 U_u\le K_UU_t\qquad(0\le u\le t,\ \tau_t\le M^a).
 \label{sc:eq:alltimeU}
\end{gather}
The constants are independent of width, time horizon, batch size and schedules. 
\end{enumerate}
\end{theorem}
\begin{proof}
Let $e_u=v_u-v_*$ and $\mathscr F_u$ contain the sketch, teacher and
history before update $u$. Proposition~\ref{arithmetic:prop:exact} shows that $\zeta_{u+1}=e_{u+1}-H_ue_u$ and $\E[\zeta_{u+1}\mid\mathscr F_u]=0.$
These zero-mean update fluctuations are orthogonal across times and to $v_*$.
Then, $v^{\rm out}-v_*=-P_{\rm out}v_*+\sum_{u<T}L_{T,u}\zeta_{u+1}$.
Their second moments are $h_uN_u$ in exact arithmetic and
$h_u(N_u+Q_u)$ with SR.
Orthogonality proves \eqref{as:clean-identity}--\eqref{as:actual-identity}.
The state recursion and \eqref{arithmetic:eq:W} also indicates
\begin{gather}
 C_t^0=P_tVP_t+\sum_{u<t}h_uP_{u+1\to t}N_uP_{u+1\to t},\notag\\
 W_t^0=(I-P_t)V(I-P_t)+\sum_{u<t}h_uP_{u+1\to t}N_uP_{u+1\to t}.
 \label{si:eq:cleandecomp}
\end{gather}
For $C\succeq0$, $GCG\preceq\tr(GC)G$, so
\begin{equation}\label{as:N-sandwich}
 R_uG\preceq N_u\preceq2R_uG.
\end{equation}
The positive output response and teacher sandwich prove
\eqref{as:clean-coefficientwise}.

Since $\tr(GP_tVP_t)\le\tr(GV)\le C_{\rm bias}$,
\eqref{si:eq:cleandecomp} and \eqref{as:N-sandwich} give
$R_t\le s+C_{\rm bias}+2\sum_{u<t}h_uR_uS_2^H(t,u)$.
Taking the finite maximum and applying \eqref{si:eq:Htel} with
$\rho_{\rm src}\theta_G\le1/4$ absorbs half that maximum:
\begin{equation}\label{si:eq:Rglobal}
 s\le R_t\le2(s+C_{\rm bias})\le C_R^{\rm raw}(1+s).
\end{equation}
Now $s\ge\sigma^2>0$ leads to \eqref{eq:reference-history-factor}.

Let $U_t^{\rm det},U_t^{\rm noise}$ be the $J$-traces of the mean and
fluctuation parts in \eqref{si:eq:cleandecomp}. Then
\[
 U_t^{\rm det}\le C_VC_J\sum_jg_j^{-\omega_V}\min(1,\tau_tg_j)^2,
 \qquad
 U_t^{\rm noise}\le2C_JC_R^{\rm raw}(1+s)
       \sum_{j,u<t}h_ug_j\pi_j(t,u+1)^2.
\]
Telescoping bounds the inner time sum by
$\rho_{\rm src}\min(1,\tau_tg_j)$.
For $Z=K_0(1+\tau_t)$, split the following sums at $Z^\alpha$:
\begin{gather}\label{si:eq:powersources}
 \sum_{j\ge1}j^{a-b}\min(1,Zj^{-a})^2
 \le\frac{2a+2b-1}{a+b-1}Z^\nu,\quad
 \sum_{j\ge1}\min(1,Zj^{-a})
 \le\frac{2a-1}{a-1}Z^\alpha.\notag
\end{gather}
Use the integral plus maximum for the first sum and the integral for
the second. With $s\le\sigma^2+C_{\bar A}$ and $\alpha<\nu$, these
give \eqref{si:eq:Uupper}.
For $1\le\tau_t\le M^a$, sum the mean contribution over indices
$j\le n=\lfloor\tau_t^\alpha\rfloor$.
They satisfy $1-\pi_j(t,0)\ge(1-e^{-1})\min(1,c_g)$, and
$\sum_{j\le n}j^{a-b}\ge
2^{-(a-b+1)}\tau_t^\nu/(a-b+1)$.
The lower teacher, spectral and eigenbasis bounds on $J$ imply \eqref{si:eq:Ulower}.

For the early trajectory, set
\begin{gather}
 R_* = C_R^{\rm raw}(1+\sigma^2+C_{\bar A}),\qquad
 D_V=C_g^{2-\omega_V}\frac{a+b}{a+b-1},\qquad
 u_* =c_Vc_Jc_g^{2-\omega_V}e^{-2C_g},\notag\\
 K_{\rm pre}=\frac{C_J(C_VD_V+2R_*\Theta)}{u_*},\qquad
 K_U=\max\{1,2^\nu C_U^{\rm raw}/c_U,K_{\rm pre}\}.
 \label{sc:eq:KU}
\end{gather}
For $0<\tau_t\le1$, $1-\pi_j(u,0)\le g_j\tau_u$ and
$\sum_{v<u}h_v\le\tau_u^2$ yield
$U_u\le C_J(C_VD_V+2R_*\Theta)\tau_t^2$ for $u\le t$.
The contribution from the leading eigenvector gives $U_t\ge u_*\tau_t^2$.
Together with \eqref{si:eq:Uupper}--\eqref{si:eq:Ulower} for
$\tau_t\ge1$ and $U_u=U_t=0$ at $\tau_t=0$, this proves
\eqref{sc:eq:alltimeU}.

Finally, $s\le r_t\le(1+\sigma^{-2})s$ for $\tau_t\ge1$.
For $\tau_*\le\tau_T\le M^a$, these bounds imply $c_{\rm src}r_t\le R_t\le C_{\rm src}r_t $ for $ \tau_t\ge\tau_*,$ and 
\begin{gather}
 c_{\rm src}r_T\le R_u\le C_{\rm src}(1+s)\quad(0\le u\le T),
 \label{si:eq:S1exports}\\
 U_u\le C_{\rm src}U_t\quad(u\le t,\ \tau_t\ge\tau_*),\quad
 U_t\le C_{\rm src}(1+\tau_t)R_t\quad(0\le t\le T).\notag
\end{gather}
The constants $c_{\rm src},C_{\rm src}$ are defined in \eqref{as:F2}.
\end{proof}

\subsection{Bounds for rounding-induced second moments}
Rounding changes the residual and parameter second moments, which
in turn determine the variance of later rounding errors.
The quantities $x_t^+$ and $y_t^+$ bound these increases, respectively.
Lemma~\ref{sc:th:primary} controls these increases using the
exact-arithmetic moments and an explicit interaction term.

Use $k_2=S_2^{\mathcal A}$, $k_1=C_JS_1^{\mathcal A}$,
$k_0=C_JS_0^{\mathcal A}$ and the causal operators
$(\Lop_i z)_t=\sum_{u<t}h_uk_i(t,u)z_u$ for $z\ge0$.
Operator products act from right to left.
In this subsection suppress the upper-comparison superscript: write
$x=x^+$, $y=y^+$ and $f=f^+$, $\beta=\beta^+$, $m=m^+$.
Thus $m=f\beta/3$, including when a coefficient is zero.
With $B^0=\beta R+mU$, define the source parts  
\begin{gather}
 X_f=\Lop_2(fU),\quad X_B=\Lop_1B^0,\quad
 Y_f=\Lop_1(fU),\quad Y_b=\Lop_0(\beta R),\quad Y_m=\Lop_0(mU),
 \label{sc:eq:sourcesplit}\\
 \Psi_x=X_f+X_B,\qquad \Psi_y=Y_f+Y_b+Y_m.
 \label{sc:eq:profiles}
\end{gather}
We use the bounds of Theorem~\ref{as:theorem-clean} and the 
constant $C_K$ from Lemma~\ref{si:lem:products}. Consider
\begin{gather}
 c_R=\min(c_{\rm src},1),\quad
 C_U^{\rm fb}=\max\{C_{\rm src},C_{\rm src}(1+\tau_*)^\nu,1\},\notag\\
 C_{m\beta}=c_m\Theta/(c_bc_R\sigma^2),\quad
 C_{fr}=c_f\Theta/\sigma^2,\quad C_{yA}=C_Jc_f\Theta C_U^{\rm fb}C_K,\notag\\
 C_8=(C_Jc_fc_b\Theta^2+C_{yA}+C_Jc_m\Theta^2C_U^{\rm fb})(1+C_{fr}),\notag\\
 \bar C_0=c_f\max(\Theta,c_b/c_m),\quad
 C_\Theta=2C_Jc_f\Theta C_K/c_P+C_{yA}\max(1,C_{m\beta}),\notag\\
 C'_\Theta=C_Jc_m\Theta^2/c_P,\qquad
 C_7=c_b\Theta C_JC_K/c_P.\label{as:F4}
\end{gather}
These coefficients depend only on the fixed model, comparison bounds and
minimum cumulative learning rate $\tau_*$, before the rate and feedback thresholds are chosen.

Assume Assumption~\ref{as:assumption-schedules}. Then all cumulative-time
gaps are at most $M^a$, and Theorem~\ref{as:theorem-clean} gives
$R_u\ge c_Rr_T$ throughout the history, with
$c_R=\min(c_{\rm src},1)$ as in \eqref{as:F4}.
Here sequence products, such as $fU$ or $\beta R$, are evaluated
pointwise in time. The kernel comparison \eqref{si:eq:volcomp} gives
\begin{gather}
 \mathsf D=\Lop_2+\Lop_1\beta,\quad
 \mathsf E=\Lop_2f+\Lop_1m,\quad
 \mathsf F=\Lop_1+\Lop_0\beta,\quad
 \mathsf H=\Lop_1f+\Lop_0m,
 \label{sc:eq:fourblocks}\\
 \Xi=\mathsf E\Psi_y,
 \label{sc:eq:Xi}\\
 x^+\le\Psi_x+\mathsf Dx^++\mathsf Ey^+,\qquad
 y^+\le\Psi_y+\mathsf Fx^++\mathsf Hy^+.
 \label{sc:eq:scalarineq}
\end{gather}

\begin{lemma}[Bounds for rounding-induced second moments]\label{sc:th:primary}
On the common event, every admissible deterministic schedule satisfies,
simultaneously at every $0\le t\le T$,
\begin{equation}\label{sc:eq:primaryconclusion}
 y_t^+\le4\Psi_y(t),\qquad x_t^+\le4\Psi_x(t)+8\Xi(t).
\end{equation}
\end{lemma}
\begin{proof}
Let us consider $w_t=\max(U_t,1)$ and define
\begin{gather}
 A_\times=C_Jc_fc_b\Theta^2,\quad A_m=C_Jc_m\Theta^2C_U^{\rm fb},\quad
 \delta=\max(\rho_{\rm sch},\Gamma_{\rm stab},\Gamma_\times),\notag\\
 \lambda=C_8\delta,\qquad
 \gamma=C_3\bar C_0\Gamma_{f\to b},\label{sc:eq:localconstants}\\
 \epsilon_0=C_K\rho_{\rm sch},\quad\epsilon=C_JC_K\rho_{\rm sch},\quad
 \zeta=C_7\rho_{\rm sch},\quad e=\epsilon_0/c_R,\quad
 \epsilon_H=C_\Theta\rho_{\rm sch}+C'_\Theta\Gamma_{\rm stab}. \label{sc:eq:smallgains}
\end{gather}
The fixed admissibility thresholds are chosen so that
\begin{gather}
 \lambda\le\tfrac12,\quad\gamma\le\tfrac1{128},\quad
 \epsilon_0\le\epsilon\le\tfrac1{64},\quad\zeta\le\tfrac1{64},\quad
 e\le\tfrac1{128},\quad\epsilon_H\le\tfrac18.
 \label{sc:eq:numericalgains}
\end{gather}
The exact-arithmetic moment bounds, kernel masses in Lemma~\ref{si:lem:products}
and the definitions of the feedback gains give
\begin{gather}\label{sc:eq:allhistoryw}
 U_u\le C_U^{\rm fb}w_t\ (u\le t),\qquad
 Y_f(t)\le C_{yA}\rho_{\rm sch}w_t,\qquad
 Y_m(t)\le A_m\Gamma_{\rm stab}w_t,\\
 f_tY_b(t)\le A_\times\Gamma_\times r_T.
 \label{sc:eq:backforward}
\end{gather}
Using $f_tw_t\le f_tU_t+C_{fr}r_T$ and the definition of $C_8$ gives
the first bound below. For the second, the spectral comparison and
forward-to-backward gain give
$\beta_tX_f(t)\le C_3\Gamma_{f\to b}
[c_f\theta_G\beta_tR_t+(c_fc_b/c_m)m_tU_t]\le\gamma B_t^0$;
at $\beta_t=0$ it follows from nonnegativity. Thus
\begin{equation}\label{sc:eq:localsources}
 f_t\Psi_y(t)\le\lambda(f_tU_t+r_T),\qquad
 \beta_tX_f(t)\le\gamma B_t^0.
\end{equation} 
Reordering the causal sums and applying Lemma~\ref{si:lem:products} gives,
for every $z\ge0$,
\begin{gather}
 \Lop_2\Lop_2z\le\epsilon_0\Lop_2z,\quad
 \Lop_2\Lop_1z\le\epsilon_0\Lop_1z,\quad
 \Lop_1\Lop_2z\le\epsilon_0\Lop_1z,\quad
 \Lop_1\Lop_1z\le\epsilon\Lop_0z,
 \label{sc:eq:blockcomps}\\
 \Lop_1(\beta\Lop_1z)\le\zeta\Lop_1z,\qquad
 \Lop_0(\beta\Lop_1z)\le\zeta\Lop_0z.
 \label{sc:eq:betacomps}
\end{gather}
Use \eqref{si:eq:CKexports} and \eqref{si:eq:schedexp-min11} for the first line.
For the second, use
$\beta_v\le c_b\Theta/M$, the $MS_1$ composition bound and the
$S_1$ mass bound; $S_0\ge c_PM$ gives the $\Lop_0$ comparison.
The same bounds yield
$\Lop_1(f\Psi_y)\le(C_Jc_f\Theta C_K/c_P)\rho_{\rm sch}\Psi_y$ and
$\Lop_0(m(Y_b+Y_m))\le C'_\Theta\Gamma_{\rm stab}(Y_b+Y_m)$.
Moreover $m_t\le C_{m\beta}\beta_tR_t$ gives
$m_tY_f(t)\le C_{yA}\rho_{\rm sch}\max(1,C_{m\beta})(m_tU_t+\beta_tR_t)$.
Applying $\Lop_0$ proves
\begin{equation}\label{sc:eq:Hgain}
 \mathsf H\Psi_y\le\epsilon_H\Psi_y.
\end{equation} 
Apply these compositions and \eqref{sc:eq:localsources} to the source
profiles to obtain
\begin{gather}
 \mathsf D\Psi_x\le(\epsilon+\zeta+\gamma)\Psi_x,\qquad
 \mathsf F\Psi_x\le(\epsilon+\zeta+\gamma)\Psi_y.
 \label{sc:eq:sourceblockbounds}
\end{gather}
Let $\Xi=E_f+E_m$ with $E_f=\Lop_2(f\Psi_y)$ and
$E_m=\Lop_1(m\Psi_y)$.
The local source bound and $S_2$ give
$E_f\le\lambda X_f+\lambda\epsilon_0r_T$ and
$\beta E_f\le\lambda(\gamma+e)B^0$, since $R_t\ge c_Rr_T$.
With \eqref{sc:eq:blockcomps}--\eqref{sc:eq:Hgain}, these imply
\begin{gather}
 \mathsf D\Xi\le(\epsilon+\zeta)\Xi+\lambda(\gamma+e)X_B,\qquad
 \mathsf F\Xi\le[(\epsilon+\zeta)\epsilon_H+\lambda(\gamma+e)]\Psi_y.
 \label{sc:eq:crossblockbounds}
\end{gather}
All profiles vanish initially and depend only on earlier indices.
Substituting the induction bounds into \eqref{sc:eq:scalarineq} and using
\eqref{sc:eq:numericalgains} gives $x_t^+\le\tfrac{39}{32}\Psi_x(t)+\tfrac{17}{4}\Xi(t)
       \le4\Psi_x(t)+8\Xi(t)$ and  $y_t^+\le\tfrac74\Psi_y(t)\le4\Psi_y(t).$
\end{proof}

\subsection{From state moments to output prediction risk}\label{sc:sec:outputassembly}
Squared output filters pair two state indices. Their total cumulative-time
gap can therefore reach $2M^a$, even though $\tau_T\le M^a$.
Lemma~\ref{sc:lem:afterfilter} propagates the state estimates through these
filters, preparing the output-risk bound in Theorem~\ref{sc:th:output}.
\begin{lemma}[Finite output-filter compositions]
\label{sc:lem:paircells}\label{sc:lem:afterfilter}
Suppose the sketch satisfies \eqref{as:U1}, and the schedule satisfies
$0\le\eta_tg_1\le c_0<1$, $\tau_T\le M^a$ and
$\rho_t\le\rho_{\rm sch}$, with the notation of \eqref{eq:clock}.
There is a constant $K_R>0$, depending only on the fixed model and comparison
bounds and $c_0$, such that for either output in
\eqref{app:eq:output-weights} and every $u<T$, the kernels in
\eqref{si:eq:Akernels} and responses in \eqref{eq:output-response} satisfy
\begin{gather}
 \sum_{u<t<T}h_tS_2^H(t,u)\mathcal R_t^{\rm out}(G)
       \le K_R\rho_{\rm sch}\mathcal R_u^{\rm out}(G),\label{sc:eq:after22}\\
 \sum_{u<t<T}h_tS_1^H(t,u)\mathcal R_t^{\rm out}(G)
       \le K_R\rho_{\rm sch}\mathcal R_u^{\rm out}(J),\label{sc:eq:after12}\\
 M^{-1}\sum_{u<t<T}h_tS_1^H(t,u)\mathcal R_t^{\rm out}(J)
       \le K_R\rho_{\rm sch}\mathcal R_u^{\rm out}(J).\label{sc:eq:after11}
\end{gather}
The same constant applies to all such schedules and both outputs.
\end{lemma}
\begin{proof}
Use the paired bounds of
Lemma~\ref{sc:lem:pairspectral}, keeping its exact $K_S$ and $2M^a$ window. For $u<t<n=\min(v,w)$ put $b_t=\tau_t-\tau_{u+1}$ and
$a_t=\tau_v+\tau_w-2\tau_{t+1}$, and
$d=\tau_v+\tau_w-2\tau_{u+1}$. Then
\begin{equation}\label{sc:eq:pairshift}
 a_t+2b_t+2\eta_t=d.
\end{equation}
The cells $[b_t,b_t+\eta_t]$ tile $[0,\tau_n-\tau_{u+1}]$;
$[a_t,a_t+2\eta_t]$ tile a subset of $[0,d]$ in reverse.
The $b_t$ sums obey \eqref{si:eq:F12}. Integrating on the reversed
cells of length $2\eta_t$ gives
$\sum\eta_t\varphi_1(a_t)\le F_1^{\rm pair}(1+d)^\alpha$ and
$\sum\eta_t\varphi_2(a_t)\le F_2^{\rm pair}$, where
$A_2=1+2\bar\eta$, $H_{\rm pair}=3(1+\bar\eta)$ and
\begin{gather}
 F_1^{\rm pair}=A_2/(2\alpha),\qquad
 F_2^{\rm pair}=A_2^{2-\alpha}/[2(1-\alpha)],\notag\\
 F_{\max}=\max(F_1,F_2,F_1^{\rm pair},F_2^{\rm pair}),\qquad
 C_{\rm pair}=2H_{\rm pair}^2F_{\max}.
 \label{sc:eq:paircellconstants}
\end{gather}
By \eqref{sc:eq:pairshift},
$1+d\le1+3\max(a_t,b_t)+2\bar\eta\le H_{\rm pair}[1+\max(a_t,b_t)]$.
Split at $a_t\ge b_t$.
Summing the smaller-argument profile yields coefficients
$H_{\rm pair}^{2-\alpha}(F_2+F_2^{\rm pair})$,
$H_{\rm pair}^{2-\alpha}F_1+H_{\rm pair}^{1-\alpha}F_2^{\rm pair}$,
and $H_{\rm pair}^{1-\alpha}(F_1+F_1^{\rm pair})$, respectively,
all bounded by $C_{\rm pair}$; here $2\alpha-2\le\alpha-1$.
Since $(1+d)^\alpha/M\le3^\alpha$ on the pair window,
\begin{gather}
 \sum\eta_t\varphi_2(b_t)\varphi_2(a_t)\le C_{\rm pair}\varphi_2(d),\quad 
 \sum\eta_t\varphi_1(b_t)\varphi_2(a_t)\le C_{\rm pair}\varphi_1(d),\notag\\
 M^{-1}\sum\eta_t\varphi_1(b_t)\varphi_1(a_t)
                         \le3^\alpha C_{\rm pair}\varphi_1(d).
 \label{sc:eq:pairscalarcomps}
\end{gather}

Substitute $S_i^H\le C_P\varphi_i$, the paired bounds at $a_t,d$,
and $h_t\le\rho_{\rm sch}\eta_t$:
\[
 \sum_{u<t<\min(v,w)}h_tS_i^H(t,u)\Lambda_j^{\rm pair}(t;v,w)
 \le K_{\rm pair}\rho_{\rm sch}\Lambda_m^{\rm pair}(u;v,w),
\]
for $(i,j,m)=(2,2,2),(1,2,1),(1,1,1)$, with an extra $M^{-1}$ on
 the last left side, where
\begin{equation}\label{sc:eq:KR}
 K_{\rm pair}=3^\alpha C_PK_S^2C_{\rm pair},\qquad
 K_R=K_{\rm pair}\max(1,1/c_J,C_J/c_J).
\end{equation}
Define $\Lambda_J^{\rm pair}=\sum_jg_jJ_{jj}\pi_j(v,u+1)\pi_j(w,u+1)$.
The bounds in the eigenbasis of $G$ give $c_J\Lambda_1^{\rm pair}\le\Lambda_J^{\rm pair}\le C_J\Lambda_1^{\rm pair}$, and
\[
 \mathcal R_u^{\rm out}(G)=\tfrac12\sum_{v,w>u}w_v^{\rm out}w_w^{\rm out}\Lambda_2^{\rm pair},
 \qquad
 \mathcal R_u^{\rm out}(J)=\tfrac12\sum_{v,w>u}w_v^{\rm out}w_w^{\rm out}\Lambda_J^{\rm pair}.
\]
Multiply by $w_v^{\rm out}w_w^{\rm out}/2$ and reorder.
The factors $1,1/c_J,C_J/c_J$ give $K_R$, also for empty sums and zero cells.
\end{proof}

\paragraph{The SR-induced risk and its remainder.}
The direct and temporal-interaction contributions are defined in
\eqref{as:output-sources}; their sum is $\N^{\rm out}$ in
\eqref{as:native}. The risk increase $\Delta^{\rm out}$ is defined in
\eqref{eq:penalty}. For the comparison traces $x_t^+,y_t^+$ of
Appendix~\ref{app:model}, define
\begin{equation}\label{sc:eq:remainder}
 \rem^{\rm out}=\sum_{t<T}h_t[(2x_t^++f_ty_t^+)\mathcal R_t^{\rm out}(G)
 +(\beta_tx_t^++m_ty_t^+)\mathcal R_t^{\rm out}(J)].
\end{equation}
Here and below, $f,\beta,m$ are the upper coefficients in
\eqref{eq:rounding-source-coefficients}.

The following theorem bounds prediction risk using the state-moment
estimates. Under admissibility, repeated rounding interactions contribute
at most a fixed multiple of the leading terms.
Proposition~\ref{as:prop-native-lower} then supplies the matching lower
bound for the full rounding expression.
\begin{theorem}[Bounds for the SR-induced output risk]\label{sc:th:output}
On $\Omega_M$, under Assumption~\ref{as:assumption-schedules}, there is
a fixed constant $C_{13}>0$ such that, for either output in
\eqref{app:eq:output-weights},
\begin{align}
 \Psi_-^{\rm out}&\le\Delta^{\rm out}\le(1+C_{13})\N^{\rm out},
 \label{sc:eq:penaltytarget}\\
 \rem^{\rm out}&\le C_{13}[\Psi_+^{\rm out}+\Conv_{b\to f}^{\rm out}].
 \label{sc:eq:laststrong}
\end{align}
The bound on $\Gamma_{f\to b}$ in \eqref{as:class}, defined in
\eqref{as:allclock-gains}, controls the forward-to-backward interaction
by a fixed multiple of the direct contribution $\Psi_+^{\rm out}$.
\end{theorem}
\begin{proof}
Use the source profiles in \eqref{sc:eq:sourcesplit}--\eqref{sc:eq:profiles}
and the causal operators defined immediately before them, with $\Xi$ from
\eqref{sc:eq:Xi}. Write $Z_t^b=Z_t^{b,+}$ and
$Y_t^{f,H}=Y_t^{f,+}$ from \eqref{eq:scheduled-conversion-sources}.
For this proof set
\begin{gather}
 E_U=K_U[C_Jc_f\Theta C_K\rho_{\rm sch}+C_Jc_m\Theta^2\kappa],\quad
 K_b=C_JC_3/c_J^2,\notag\\
 \chi_b=\frac{A_\times\kappa}{3c_R},\quad
 \chi_f=\frac{A_\times\kappa C_K\rho_{\rm sch}}{c_R},\quad
 D_{\rm out}=C_3K_R\rho_{\rm sch},\label{sc:eq:outconstants}
\end{gather}
where $A_\times$ is defined in \eqref{sc:eq:localconstants}.
Abbreviate
$\langle z,Q\rangle_{\rm out}=\sum_{t<T}h_tz_t\mathcal R_t^{\rm out}(Q)$
and
\[
 P_f=\langle fU,G\rangle_{\rm out},\quad
 P_b=\langle\beta R,J\rangle_{\rm out},\quad
 P_m=\langle mU,J\rangle_{\rm out},\quad
 P=P_f+P_b+P_m=\Psi_+^{\rm out}.
\]
The  parameter comparison \eqref{sc:eq:alltimeU} bounds
$Y_f+Y_m$ by $E_UU$. The kernel comparison in
Lemma~\ref{si:lem:products} gives
$C_JS_0^{\mathcal A}\le(C_J/c_J^2)K_{JJ}^{\mathcal A}\le K_bK_{JJ}^H$.
Together with \eqref{sc:eq:backforward}, $m=f\beta/3$ and
$R_t\ge c_Rr_T$, these yield
\begin{gather}
 \Psi_y(t)\le Y_b(t)+E_UU_t,\quad
 Y_b(t)\le K_bZ_t^b,\quad
 f_tY_b(t)\le A_\times\kappa r_T,\quad
 m_tY_b(t)\le\chi_b\beta_tR_t.\label{sc:eq:outsourcepoint}
\end{gather}
Define $X_m=\Lop_1(mU)$,
$E_f^b=\Lop_2(fY_b)$ and $E_m^b=\Lop_1(mY_b)$.
The definition of $\Xi$ and the first bound above imply
\begin{gather}
 \Xi\le E_f^b+E_m^b+E_U(X_f+X_m),\qquad
 \beta_tE_f^b(t)\le\chi_f\beta_tR_t.\label{sc:eq:outXidecomp}
\end{gather}
For the second inequality, use
$E_f^b\le A_\times\kappa r_TC_K\rho_{\rm sch}$ and $R_t\ge c_Rr_T$.

Lemma~\ref{sc:lem:afterfilter} and the kernel comparison in
Lemma~\ref{si:lem:products} give, for every $z\ge0$,
\begin{gather}
 \langle\Lop_2z,G\rangle_{\rm out}\le D_{\rm out}\langle z,G\rangle_{\rm out},
 \quad
 \langle\Lop_1z,G\rangle_{\rm out}\le C_JD_{\rm out}\langle z,J\rangle_{\rm out},
 \label{sc:eq:bracketcomps}\\
 \langle\beta\Lop_1z,J\rangle_{\rm out}
 \le c_b\Theta C_JD_{\rm out}\langle z,J\rangle_{\rm out}.\notag
\end{gather}
The last inequality also uses $\beta_t\le c_b\Theta/M$.
Applying these estimates to the source profiles gives, 
\begin{align*}
 \langle\Psi_x,G\rangle_{\rm out}
   +\langle\beta\Psi_x,J\rangle_{\rm out}
   &\le C(P+\Conv_{f\to b}^{\rm out}),\\
 \langle f\Psi_y,G\rangle_{\rm out}
   +\langle m\Psi_y,J\rangle_{\rm out}
   &\le C(P+\Conv_{b\to f}^{\rm out}),\\
 \langle\Xi,G\rangle_{\rm out}
   +\langle\beta\Xi,J\rangle_{\rm out}
   &\le C(P+\Conv_{b\to f}^{\rm out}+\Conv_{f\to b}^{\rm out}),
\end{align*}for a fixed $C$.
The first uses $X_f\le C_3Y^{f,H}$; the second uses
\eqref{sc:eq:outsourcepoint}. For the third, apply
\eqref{sc:eq:bracketcomps} to the decomposition
\eqref{sc:eq:outXidecomp}, using its pointwise bound for
$\langle\beta E_f^b,J\rangle_{\rm out}\le\chi_fP_b$.
Substitute into \eqref{sc:eq:remainder} and use the bounds
$x^+\le4\Psi_x+8\Xi$, $y^+\le4\Psi_y$ of
Lemma~\ref{sc:th:primary}. This gives a fixed $C_*^{\rm out}\ge1$ with
\begin{equation}\label{sc:eq:explicitremcoeff}
 \rem^{\rm out}\le C_*^{\rm out}
 [P+\Conv_{b\to f}^{\rm out}+\Conv_{f\to b}^{\rm out}].
\end{equation}
Substituting the gain definition \eqref{as:allclock-gains} and the
rounding coefficients gives
\[
 \beta_tY_t^{f,H}\le\Gamma_{f\to b}
 [c_f\theta_G\beta_tR_t+(c_fc_b/c_m)m_tU_t]
 \le\bar C_0\Gamma_{f\to b}B_t^0,
\]
where $B^0=\beta R+mU$ as in \eqref{sc:eq:sourcesplit}, and
$\bar C_0$ is defined in \eqref{as:F4}.
This holds also at $\beta_t=0$. Multiply by
$h_t\mathcal R_t^{\rm out}(J)$ and sum, using
$B_t^0J\preceq D_t^+$ and $\Gamma_{f\to b}\le\kappa$:
\begin{equation}\label{sc:eq:lastabsorb}
 \Conv_{f\to b}^{\rm out}\le\bar C_0\kappa\Psi_+^{\rm out}.
\end{equation}
Taking $C_*^{\rm out}$ large enough for both outputs, define
\begin{equation}\label{sc:eq:C13}
 C_{13}=(1+\bar C_0\kappa)C_*^{\rm out}.
\end{equation}
Then \eqref{sc:eq:explicitremcoeff} proves \eqref{sc:eq:laststrong} and
\begin{equation}\label{sc:eq:remtarget}
 \rem^{\rm out}\le C_{13}\N^{\rm out}.
\end{equation}
All coefficients are fixed independently of the schedule and output;
$C_{13}$ is not used to choose admissibility thresholds.

Since $GX_t^+G\preceq x_t^+G$, the upper source in
\eqref{core:comparison-innovations} is bounded by $D_t^+$ plus the
four terms in \eqref{sc:eq:remainder}. The covariance sandwich
\eqref{arithmetic:eq:sandwich} and output identity
\eqref{as:actual-identity} therefore give
$\Delta^{\rm out}\le\Psi_+^{\rm out}+\rem^{\rm out}$.
The same positivity argument gives $\Delta^{\rm out}\ge\Psi_-^{\rm out}$,
since $Q_t^-\succeq D_t^-$.
Equation~\eqref{sc:eq:remtarget} now proves \eqref{sc:eq:penaltytarget}.
\end{proof}

\begin{proposition}[Matching lower bound for SR-induced error]
\label{as:prop-native-lower}
For every full-row-rank $S$, finite deterministic schedule and either
output in \eqref{eq:output-weights},
\begin{equation}\label{as:strong-native-lower}
 \Delta^{\rm out}\ge\Psi_-^{\rm out}+
              \Conv_{b\to f}^{-,\rm out}+\Conv_{f\to b}^{-,\rm out}
              \ge c_{\rm native}\N^{\rm out},
\end{equation}
where $r_f=r_b=1/144,r_m=(16/3)c_\times,$ and
\begin{gather}
 c_{\rm native}=\min(r_f,r_b,r_m,r_fr_b)=r_m>0.
 \label{as:c-native}
\end{gather}
\end{proposition}
\begin{proof}
Unroll the separate sources $\beta_u^-R_uJ$ and $f_u^-U_uG$ to get
$y_t^-\ge Z_t^{b,-}$ and $x_t^-\ge Y_t^{f,-}$. Distinct terms give
$Q_t^-\succeq D_t^-+f_t^-Z_t^{b,-}G+\beta_t^-Y_t^{f,-}J$.
Apply the positive response, sum and use $Q_t\succeq Q_t^-$.
The coefficient ratios in \eqref{arithmetic:eq:coefficients} are
$r_f,r_b,r_m$, so $\Psi_-\ge\min(r_f,r_b,r_m)\Psi_+$ and each
conversion has ratio $r_fr_b$. Lemma~\ref{arithmetic:lem:joint}
induces $r_m<r_fr_b$.
\end{proof}

\paragraph{Corollary: covariance and full-risk comparisons.}
\label{as:covariance-comparison}
 Suppose $S\in\Omega_M$, Assumption~\ref{as:assumption-schedules} holds,  
\begin{equation}\label{as:main-width}
 M\ge M_{\rm w},\qquad M^a\ge\tau_*.
\end{equation}
For either output in \eqref{app:eq:output-weights}, the SR-induced risk
satisfies
\begin{equation}\label{as:native-law}
 c_{\rm native}\N^{\rm out}\le\Delta^{\rm out}
 \le(1+C_{13})\N^{\rm out}.
\end{equation}
With $\mathcal B^{\rm out},\mathcal V^{\rm out}$ from
\eqref{eq:scheduled-clean-functionals} and $\N^{\rm out}$ from
\eqref{as:native}, the full prediction risk satisfies
\begin{align}
 \mathscr R^{\rm out}&\ge c_AM^{1-b}+c_V\mathcal B^{\rm out}+\mathcal V^{\rm out}
                                      +c_{\rm native}\N^{\rm out},\notag\\
 \mathscr R^{\rm out}&\le C_A^{\rm sch}M^{1-b}+C_V\mathcal B^{\rm out}+2\mathcal V^{\rm out}
                                      +(1+C_{13})\N^{\rm out}.
 \label{as:full-componentwise}
\end{align}
In particular, common lower and upper constants are
\begin{gather}
 c_{\rm full}=\min\{c_A,c_V,1,c_{\rm native}\}>0,\qquad
 C_{\rm full}=\max\{C_A^{\rm sch},C_V,2,1+C_{13}\}<\infty.
 \label{as:full-constants}
\end{gather}

\begin{proof}
Theorem~\ref{sc:th:output} and
Proposition~\ref{as:prop-native-lower} give \eqref{as:native-law}.
Add the exact-arithmetic comparison of Theorem~\ref{as:theorem-clean}
and the approximation bounds \eqref{as:approx-lower}, using the risk
decomposition \eqref{eq:estimation-risk}, to obtain
\eqref{as:full-componentwise}. Taking the minimum and maximum of its
coefficients gives \eqref{as:full-constants}. The exact-arithmetic part
also gives $\mathscr R_0^{\rm out}\asymp
M^{1-b}+\mathcal B^{\rm out}+\mathcal V^{\rm out}$.
All bounds hold deterministically on $\Omega_M$ for every admissible
schedule chosen before the random draws; admissibility is checked at
the realized sketch. The event's probability guarantee is for each
width separately, as in Proposition~\ref{as:prop-event}.
\end{proof}

\paragraph{Replacing the exact-arithmetic residual by its noise scale.}
\label{as:cor-optional-noise}
On $\Omega_M$ under the cap and source rate, define
$\mathcal V^{\rm out,(s)}=\tfrac12\sum_{u<T}h_us\sum_jg_j^2\ell_{{\rm out},j,u}^2$.
Multiplying \eqref{eq:reference-history-factor} by each nonnegative
squared-filter term gives
$\mathcal V^{\rm out,(s)}\le\mathcal V^{\rm out}\le
K_s^{\rm ref}\mathcal V^{\rm out,(s)}$ for any finite schedule and
prescribed output. Thus \eqref{as:full-componentwise} allows
this replacement, with unchanged lower constant and $\N$, and upper
constant $\max\{C_A^{\rm sch},C_V,2K_s^{\rm ref},1+C_{13}\}$.
The factor $K_s^{\rm ref}$ is distinct from the applications' $K_s$.

\subsection{Construction of the admissibility thresholds}\label{app:forest}
Fix the population family, its positive noise variance $\sigma^2$, a valid
comparison tuple $\Worig$ from Definition~\ref{as:def-witness}, and any
 $c_0\in(0,1)$. Recall $\Theta=C_\theta^{\rm tr}$ and the exponents
$\alpha,\nu,\omega_V$ from \eqref{eq:source-exponents}.
The choices have a definite order: first choose the cumulative-time
threshold $\tau_*$ below, then compute the auxiliary bounds, and finally
define the feedback threshold $\kappa_{\rm sch}$ and rate ceiling
$\rho_{\rm sch}$ in \eqref{as:F5}. This order ensures that the
thresholds depend only on fixed model quantities, with no circular dependence on the schedules to be compared.

\paragraph{Cumulative optimization time threshold.}
Set $\bar\eta=c_0/c_g$, $C_{\bar A}=2C_A^{\rm sch}$ and
$\Lambda_0=(1-c_0)^{-1}$. Choose and fix
$\tau_*=\tau_*^{\rm sch}=\tau_{*,\rm sch}$ satisfying
\begin{equation}\label{as:F0}
 \tau_*\ge32(1+\bar\eta+c_g^{-1}),\qquad
 c'=2(1-c_0),\quad C'=2/(1-c_0).
\end{equation}
The chosen $\tau_*$ may exceed this lower bound. It is used throughout the
scheduled theorem and its constant-learning-rate corollaries. The
coefficients $C_{\rm fb}$ and $C_U^{\rm fb}$ below depend on this choice,
so the rate and feedback thresholds are computed afterwards.

\paragraph{Bounds for propagation and spectral sums.}
The following coefficients give the spectral bounds and the comparison
factor $C_3$ in Lemma~\ref{si:lem:products}. For $\ell=1,2$ define
\begin{gather}
 C_{U,\ell}=2^\ell\left(1+C_g^\ell\frac a{a-1}\right),\quad
 c_{L,\ell}(\Lambda)=\tfrac14\Lambda^{-\alpha}c_g^\ell
                         e^{-2^{a+1}C_g\Lambda},\notag\\
 c_{L,0}(\Lambda)=\tfrac12e^{-2^{a+1}C_g\Lambda},\quad
 L_0=c_{L,0}(\Lambda_0),\quad U_0^{\rm sp}=1,\notag\\
 L_\ell=c_{L,\ell}(\Lambda_0)(1-c_0)^{\ell-\alpha},\quad
 U_\ell^{\rm sp}=C_{U,\ell}(1-c_0)^{\alpha-\ell},\notag\\
 c_P=\min(1,L_0,L_1,L_2),\quad
 C_P=\max(1,U_0^{\rm sp},U_1^{\rm sp},U_2^{\rm sp}),\quad
 C_3=\frac{C_P}{c_P}\max(1,C_J/c_J,C_J/c_J^2).
 \label{as:F1}
\end{gather}
The superscript ``sp'' distinguishes fixed spectral constants from the
exact-arithmetic source $U_t$.
\paragraph{Bounds for the exact-arithmetic case.}
The next coefficients bound the residual and parameter histories in
Theorem~\ref{as:theorem-clean}. The additional coefficient $C_{\rm fb}$
bounds the forward-to-backward gain for constant learning rates and precisions in
\eqref{ca:constant-gains}.
\begin{gather}
 K_0=\max(1,C_g/(1-c_0)),\quad
 c_L^{\rm src}=(1-e^{-1})^2\min(1,c_g)^2 C_g^{-\omega_V}
                           \frac{2^{-(a-b+1)}}{a-b+1},\notag\\
 C_{\rm bias}=C_VC_g^{b/a}\frac b{b-1},\quad
 \rho_{\rm src}=\frac1{4\Theta},\quad C_R^{\rm raw}=2\max(1,C_{\rm bias}),\notag\\
 \begin{aligned}
 C_U^{\rm raw}={}&C_VC_Jc_g^{-\omega_V}\frac{2a+2b-1}{a+b-1}K_0^\nu +2C_JC_R^{\rm raw}(1+\sigma^2+C_{\bar A})\rho_{\rm src}
                          \frac{2a-1}{a-1}K_0^\alpha,
 \end{aligned}\notag\\
 c_U=c_Vc_Jc_L^{\rm src},\quad
 c_{\rm src}=\min\{\sigma^2/(1+\sigma^2),c_U\},\notag\\
 C_{\rm src}=\max\{C_R^{\rm raw}(1+\sigma^{-2}),
 C_R^{\rm raw}(1+\sigma^2+C_{\bar A}),C_U^{\rm raw},
 2^\nu C_U^{\rm raw}/c_U,C_U^{\rm raw}/\sigma^2\},\label{as:F2}\\
 c_R^{\rm glob}=c_{\rm src},\quad C_R^{\rm glob}=C_{\rm src}(1+\sigma^{-2}),\quad
 C_{\rm fb}=\max\{C_{\rm src},C_{\rm src}(1+\tau_*)^\nu/(c_{\rm src}\sigma^2)\}.
 \label{as:F3}
\end{gather}
$K_0$ is distinct from the sketch norm bound $K_{0,\rm geo}$ in
\eqref{as:geometry-constants}. The exact-arithmetic moment bounds are
established under $\eta_u/B\le\rho_{\rm src}$; the same constants apply
under the smaller ceiling $\rho_{\rm sch}$ chosen below.

\paragraph{Bounds for rounding-induced errors.}
The coefficients for Lemma~\ref{sc:th:primary} are defined in
\eqref{as:F4}. They depend on the fixed source and
kernel bounds and $\tau_*$, independently of the thresholds below. 
With the preceding coefficients fixed, define
\begin{gather}
 \rho_{\rm fb}=\min\{\rho_{\rm src},(64C_JC_K)^{-1},(64C_7)^{-1},
                  (16C_\Theta)^{-1},(2C_8)^{-1},c_R/(128C_K)\},\notag\\
 \kappa_{\rm sch}=\min\{(128C_3\bar C_0)^{-1},(16C'_\Theta)^{-1},
                       (2C_8)^{-1},c_R/(8c_fc_b\Theta^2)\},\notag\\
 \rho_{\rm sch}=\min\{\rho_{\rm fb},\kappa_{\rm sch}/(C_{\rm fb}\Theta)\}.
 \label{as:F5}
\end{gather}
Together with the imposed feedback bounds, these minima give
\eqref{sc:eq:numericalgains}, as required by Lemma~\ref{sc:th:primary};
the two summands of $\epsilon_H$ are each at most $1/16$.
Assumption~\ref{as:assumption-schedules} requires each of the three
feedback gains to be at most $\kappa_{\rm sch}$. The construction does
not derive those conditions from the learning-rate cap alone.
The final minimum also ensures
$C_{\rm fb}\Theta\rho_{\rm sch}\le\kappa_{\rm sch}$, as used in the
constant-learning-rate gain bound \eqref{ca:constant-gains}.
Write $\kappa=\kappa_{\rm sch}$ only as an abbreviation. These are
sufficient thresholds, with no claim of sharpness. All denominators are
positive. The resulting constants are finite
and independent of width, horizon, batch, learning-rate and precision
schedules, and the realized sketch.

\endgroup

\section{Proofs of the neural and compute scaling laws}
\label{app:scaling-proofs}

This appendix proves the scaling laws stated in Section~\ref{sec:dynamics}.
Appendix~\ref{cl:finite} proves Theorem~\ref{thm:scheduled-main} by converting
the covariance bounds of Appendix~\ref{app:actual-schedules} into explicit
scalar expressions. Corollaries~\ref{ca:constant-reference}
and~\ref{ca:constant-corollary} give the constant-learning-rate laws.
Proposition~\ref{app:prop-fixed-bit-compute} proves the compute orders in
Table~\ref{tab:compute-precision};
Propositions~\ref{r4:prop-variable-rho-average}
and~\ref{r4:prop-variable-rho-last} allow learning rate and batch size to
vary across budgets and treat the two outputs. Let 
\begin{equation}\label{app:eq:full-functional}
 \mathcal L_S^{\iout}=M^{1-b}+\mathcal B^{\iout}
                         +\mathcal V^{\iout}+\mathcal N^{\iout},
 \qquad c_F=c_{\rm full},\quad C_F=C_{\rm full},
\end{equation}
where the terms are defined in \eqref{eq:scheduled-clean-functionals}, \eqref{as:native}, and \eqref{as:full-constants}.

\begingroup
\let\Conv\relax\newcommand{\Conv}{\operatorname{Conv}}
\emergencystretch=3em
\subsection{Scalar comparisons and proof of Theorem~\ref{thm:scheduled-main}}
\label{cl:finite}

We convert Appendix~\ref{app:actual-schedules}'s covariance comparisons
into the scalar law of Section~\ref{sec:general-schedules}. Recall Assumption~\ref{as:assumption-schedules} and exponents $\alpha,\nu,\omega_V$ from
\eqref{eq:source-exponents}.

\paragraph{Effective noise and exact-arithmetic residual bounds.}
The approximation bound and $M\ge1$ give
\begin{equation}\label{eq:noise-normalization}
 K_s=1+2C_A^{\rm sch}/\sigma^2,\qquad
 \sigma^2\le s(S)\le K_s\sigma^2.
\end{equation}
This normalization differs from $K_s^{\rm ref}$ in
\eqref{eq:reference-history-factor}, which bounds the exact-arithmetic
residual throughout training.
Equation~\eqref{si:eq:Rglobal} supplies the fixed choice
\begin{equation}\label{cd:Csigma}
 C_\sigma=2(1+C_{\rm bias}/\sigma^2),\qquad
 s\le R_t\le C_\sigma s.
\end{equation}
Thus $\sigma^2\le R_t\le C_\sigma K_s\sigma^2$ throughout training.

\paragraph{Decomposing the scalar rounding scale.}
Reuse $K_r^{\iout},\widehat B^{\iout},\widehat V^{\iout}$ and
$\widehat U_t$ from \eqref{s3:eq:scalar-responses_last}--\eqref{s3:eq:scalar-source}.
For a deterministic training configuration $\pi$, write $\Phi^{\iout}(\pi)$
for the scalar objective defined below. For a fixed output, suppress
its superscript on $\widehat B,\widehat V$ and define the auxiliary
quantities needed to separate the two orders of interaction:
\begin{equation}\label{cl:clock-sources}
\begin{aligned}
 k_r(t,u)&=[1+2(\tau_t-\tau_{u+1})]^{\alpha-r}
                      \quad(u<t,\ r=1,2),\\
 \widehat Z_t&=\sigma^2\sum_{u<t}h_uq_{\bw,u},\qquad
 \widehat Y_t=\sum_{u<t}h_u\frac{q_{\fw,u}}M\widehat U_u k_2(t,u).
\end{aligned}
\end{equation}
Let $\widehat\Psi$ denote the first sum in
\eqref{s3:eq:scalar-rounding}. The two terms in its second sum are
\begin{equation}\label{cl:scalar-penalty}
\begin{aligned}
 \widehat C_{bf}&=\tfrac12\sum_{t<T}h_t\frac{q_{\fw,t}}M
                                    \widehat Z_tK_2^{\iout}(t),\quad
 \widehat C_{fb}=\tfrac12\sum_{t<T}h_t\frac{q_{\bw,t}}M
                                    \widehat Y_tK_1^{\iout}(t),\\
 \widehat N^{\iout}&=\widehat\Psi+\widehat C_{bf}+\widehat C_{fb}.
\end{aligned}
\end{equation}
This is an exact decomposition of the scalar expression, with no new
approximation. For the optimization arguments below, write
\begin{equation}\label{cl:Phi}
 \Phi^{\iout}=M^{1-b}+\widehat B^{\iout}
                        +\widehat V^{\iout}+\widehat N^{\iout}.
\end{equation}
We also suppress the superscript on $\Phi$ when the output is fixed.

The following theorem replaces the matrix expressions from
Appendix~\ref{app:actual-schedules} by explicit scalar sums.
The comparison is uniform over admissible configurations, so minimizing
the scalar objective controls prediction risk within a fixed factor.
\begin{theorem}[Actual-risk comparison with a finite discrete objective]
\label{cl:objective}
Under Assumption~\ref{as:assumption-schedules}, with both width conditions
and the fixed $C_\sigma$ in \eqref{cd:Csigma}, for $\iout\in\{\mathrm{last},\mathrm{avg}\}$, we have
\begin{equation}\label{cl:actual-Phi}
 a_\Phi\Phi^{\rm out}(\pi)\le\mathscr R^{\rm out}(\pi)
                             \le A_\Phi\Phi^{\rm out}(\pi),
\end{equation}
for fixed $0<a_\Phi\le A_\Phi<\infty$ depending only on the fixed model,
the comparison bounds in Definition~\ref{as:def-witness}, 
$c_0$, $\tau_*$, $C_\sigma$ and
$K_s$,
uniformly in $S,M,T,B,F$, precision choices and learning rates.
\end{theorem}

\paragraph{Interpretation of the scalar quantities.}
The objective $\Phi$ separates approximation error, bias, minibatch
variance and SR-induced error. The scale $\widehat U_t$ estimates the
exact-arithmetic parameter moment $U_t$, the source of direct
forward-rounding error. Its two terms describe learning the target and
fluctuations from minibatch noise.
Within $\widehat N$, $\widehat\Psi$ contains direct forward and backward
rounding and their interaction. Each contribution is
weighted by $h_u$ and its output response $K_1$ or $K_2$.
The term $\widehat C_{bf}$ captures earlier backward rounding changing
the parameter moment entering later forward rounding;
$\widehat C_{fb}$ captures earlier forward rounding changing the residual
variance entering later backward rounding. These interactions can occur even
with independent, unbiased rounding because later error variances
depend on the current state. All joint terms vanish if either rounding
location uses exact arithmetic throughout training.

\begin{proof}
We compare each nonnegative term separately. The exact-arithmetic parameter
source is the only term requiring both its deterministic and stochastic
parts to be estimated before the scalar responses can be substituted.
Put $r_*=2-\omega_V$,
$D_\nu=2^{2-\nu}\max(1,3^\nu/\nu)$, and define
\begin{align}
 u_-&=\min\{c_Vc_J\ell_{r_*}2^{\nu-2},\ c_J\ell_1\},\notag\\
 u_+&=\max\{C_VC_Ju_{r_*}(1-c_0)^{-2}D_\nu,
                    2C_\sigma K_s C_Ju_1\},\quad R_+=C_\sigma K_s.
 \label{cl:Uhat-constants}
\end{align}
We first prove
$u_-\widehat U_t\le U_t\le u_+\widehat U_t$ for every state, including
zero clock. For the mean part of \eqref{si:eq:cleandecomp}, the product
bounds give
$1-e^{-g_j\tau_t}\le1-\pi_j(t,0)
\le(1-c_0)^{-1}(1-e^{-g_j\tau_t})$. Moreover,
\[
 \sum_jg_j^{-\omega_V}(1-e^{-g_j\tau})^2
 =\int_0^\tau\int_0^\tau\sum_jg_j^{2-\omega_V}e^{-g_j(v+w)}\,dv\,dw.
\]
The teacher-source sandwich can be traced against
$(I-P_t)J(I-P_t)$ without assuming that $J$ and $V$ commute with $G$.
Lemma~\ref{cl:spectral} applies to the displayed integral since
$v+w\le2M^a$. The integral of
$(1+v+w)^{\nu-2}$ lies between $2^{\nu-2}\tau^2(1+\tau)^{\nu-2}$ and
$D_\nu\tau^2(1+\tau)^{\nu-2}$: use $v+w\le2\tau$ for the lower; for the
upper use integrand $\le1$ at $\tau\le1$ and
$\int_0^{2\tau}z(1+z)^{\nu-2}dz\le(1+2\tau)^\nu/\nu$ otherwise.
These bounds give the first entries in the minimum and maximum defining
$u_-$ and $u_+$. For the stochastic part, use
$\sigma^2G\preceq N_u\preceq2R_+\sigma^2G$ in
\eqref{si:eq:cleandecomp}, then trace against $J$. The bounds
$c_J\le J_{jj}\le C_J$ and \eqref{cl:heat} at
$x=2(\tau_t-\tau_{u+1})$ give the second entries. In particular, the
noise contribution contains $\sum_{u<t}h_uk_1(t,u)$ even before clock
one. At zero clock both the mean and noise parts vanish.

Expand each output filter into its pair of state indices. The weights
are nonnegative for both outputs under consideration, so the spectral
comparison can be applied to every pair and summed.  
\[
 \tfrac12\ell_2K_2^{\rm out}(u)\le\mathcal R_u^{\rm out}(G)
       \le\tfrac12u_2K_2^{\rm out}(u),\qquad
 \tfrac12c_J\ell_1K_1^{\rm out}(u)\le\mathcal R_u^{\rm out}(J)
       \le\tfrac12C_Ju_1K_1^{\rm out}(u).
\]
All paired clocks lie in $[0,2M^a]$. The same argument applied to the
bias gives
\begin{equation}\label{cl:reference-components}
 \ell_{b/a}\widehat B\le\mathcal B^{\rm out}
       \le u_{b/a}\widehat B,\qquad
 \ell_2\widehat V\le\mathcal V^{\rm out}
       \le R_+u_2\widehat V.
\end{equation}
The second comparison also uses
$\sigma^2\le R_u\le R_+\sigma^2$.
The bounds $\theta_G\in[c_g,\Theta]$ and
\eqref{eq:rounding-source-coefficients} imply
\[
 f_u^+\asymp q_{\fw,u}/M,\qquad
 \beta_u^+\asymp q_{\bw,u}/M,\qquad
 m_u^+\asymp q_{\fw,u}q_{\bw,u}/M^2.
\]
Together with the preceding source and output comparisons, these
bounds identify the three terms of $\widehat\Psi$.
For the ordered terms, estimate the inner sources before the outer response.
Equation~\eqref{cl:JJ} gives $Z_t^{b,+}\asymp\widehat Z_t$:
the factor $M$ in $K_{JJ}^H$ cancels the $1/M$ in $\beta_u^+$.
Similarly, $U_u\asymp\widehat U_u$ and \eqref{cl:heat} give
$Y_t^{f,+}\asymp\widehat Y_t$.
Multiplying each by its outer rounding coefficient and output response
gives $\widehat C_{bf}$ and $\widehat C_{fb}$, respectively, with the
strict time ordering unchanged. Hence fixed coefficients
$l_N,u_N$ satisfy
\begin{equation}\label{cl:scalar-components}
 l_N\widehat N^{\rm out}\le\mathcal N^{\rm out}
                       \le u_N\widehat N^{\rm out}.
\end{equation}
Choose $l_\Phi>0$ below all the component lower coefficients and
$u_\Phi<\infty$ above all the upper coefficients, including the
approximation coefficient one. Then
$l_\Phi\Phi\le\mathcal L_S^{\rm out}\le u_\Phi\Phi$ under the cap,
rate and clock conditions. The comparisons in this step use no gain
condition.
 
On the full admissible class, \eqref{as:full-componentwise} bounds
actual risk above and below by fixed multiples of
$\mathcal L_S^{\rm out}$. This proves \eqref{cl:actual-Phi}, with
$a_\Phi=c_Fl_\Phi$, $A_\Phi=C_Fu_\Phi$ and $Q_\Phi=A_\Phi/a_\Phi$.
\end{proof}

\begin{proof}[Proof of Theorem~\ref{thm:scheduled-main}]
\label{s3:proof:scalar}
Work on the sketch event of Proposition~\ref{as:prop-event} under
the full admissible class. For exact-arithmetic SGD,
combine the estimation-risk bounds \eqref{as:clean-coefficientwise}
with the approximation bounds \eqref{as:approx-lower}, using
\eqref{eq:estimation-risk}.
Substitute \eqref{cl:reference-components}. The constants
\[
 c_0^{\rm ref}=\min\{c_A,c_V\ell_{b/a},\ell_2\},\qquad
 C_0^{\rm ref}=\max\{C_A^{\rm sch},C_Vu_{b/a},2R_+u_2\}
\]
then bound $\mathscr R_0^{\rm out}$ by the corresponding multiples of
$M^{1-b}+\widehat B+\widehat V$.
For the rounding error, the
comparison \eqref{as:native-law} and
\eqref{cl:scalar-components} imply
\[
 c_{\rm native}l_N\widehat N^{\rm out}
 \le\Delta^{\rm out}
 \le(1+C_{13})u_N\widehat N^{\rm out}.
\]
The decomposition \eqref{cl:scalar-penalty} is exactly
\eqref{s3:eq:scalar-rounding} for each output. These are the first two comparisons in
\eqref{eq:scheduled-law}. The decomposition of full risk follows separately from \eqref{eq:estimation-risk}.

\end{proof}

\endgroup 
\begingroup
\emergencystretch=3em
\let\Conv\relax\newcommand{\Conv}{\operatorname{Conv}}
\subsection{Constant-learning-rate scaling laws}
\label{app:actual-schedule-cases}\label{sec:rev3_schedule_cases}
We evaluate the general-schedule expressions at constant learning rates
and precisions, obtaining Corollaries~\ref{ca:constant-reference}
and~\ref{ca:constant-corollary}.
Fix $\eta>0$, $\eta g_1\le c_0$, $\rho=\eta/B\le\rho_{\rm sch}$
and $\tau=\eta T\in[\tau_*,M^a]$. Then $T,\tau\ge32$ by \eqref{as:F0}. Define
\begin{equation}\label{ca:ca:dout}
 d_{\rm last}=1,\qquad d_{\rm avg}=\tau^{1/a-1}.
\end{equation}
The scalar $d_{\rm out}$ describes the output dependence of the variance
scales below. Here $c,C>0$ have the fixed dependencies of
Appendix~\ref{app:actual-schedules} and are uniform in
$M,T,B,\eta$, precisions and output.

\begin{lemma} 
\label{ca:ca:filterslemma}\label{ca:ca:saturation}\label{ca:ca:weighted}
\label{ca:ca:relaylemma}\label{ca:ca:components}
With $r_j=1-\eta g_j$, coefficients $d_{{\rm out},j}$ and $\ell_{{\rm out},j,u}$ in
\eqref{eq:schedule-filters} are
\begin{gather}
 d_{{\rm last},j}=r_j^T,\qquad
 d_{{\rm avg},j}=\frac{r_j(1-r_j^T)}{\tau g_j},\notag\\
 \ell_{{\rm last},j,u}=r_j^{T-1-u},\qquad
 \ell_{{\rm avg},j,u}=\frac{1-r_j^{T-u}}{\tau g_j}.
 \label{ca:ca:geometric}
\end{gather}
In particular $\mathcal B^{\rm out}\asymp\tau^{-(b-1)/a}$.
For constant precisions, write $\Psi_{+,F},\Psi_{+,B},\Psi_{+,M}$ for
the three direct terms of \eqref{as:output-sources}.
Each component lies between fixed positive multiples of its listed scale.
For $\mathcal V^{\rm out}$, $\Psi_{+,B}^{\rm out}$ and
$\Conv_{b\to f}^{\rm out}$, the upper bound is $CK_s^{\rm ref}$ times
that scale, using $s\le R_u\le K_s^{\rm ref}s$ from
\eqref{eq:reference-history-factor}.
\begin{center}
\begin{tabular}{@{}llll@{}}
\toprule Component & Scale & Component & Scale\\\midrule
$\Psi_{+,F}^{\rm out}$ & $\rho q_{\fw}\tau^\nu d_{\rm out}/M$ &
$\mathcal V^{\rm out}$ & $\rho s d_{\rm out}$\\
$\Psi_{+,B}^{\rm out}$ & $\rho q_{\bw}s\tau^{1/a}/M$ &
$\Conv_{b\to f}^{\rm out}$ & $\rho^2q_{\fw}q_{\bw}s\tau d_{\rm out}/M$\\
$\Psi_{+,M}^{\rm out}$ & $\rho q_{\fw}q_{\bw}\tau^{\nu+1/a}/M^2$ &
$\Conv_{f\to b}^{\rm out}$ & $\rho^2q_{\fw}q_{\bw}\tau^{\nu+1/a}/M^2$\\
\bottomrule
\end{tabular}
\end{center}
\end{lemma}
\begin{proof}
Put $k=\tau^{1/a}\le M$. Geometric summation gives
\eqref{ca:ca:geometric}; \eqref{sc:eq:productnew} gives
$1-r_j^n\asymp\min(1,n\eta g_j)$ and two-sided exponential bounds for
$r_j^n$. Hence $d_{{\rm avg},j}\asymp\min(1,(\tau g_j)^{-1})$.
Split the bias sum at $k$, using $(\tau g_j)^2r_j^{2T}\le1$
for the last iterate. Indices $j\in[k/2,k]$ give the lower bound
(use $j=1$ if $k<2$):
\begin{equation}\label{ca:ca:biaslaw}
 \mathcal B^{\rm out}\le C\left[
 \tau^{-2}\sum_{j\le k}j^{2a-b}+\sum_{j>k}j^{-b}\right]
 \le Ck^{1-b},\qquad \mathcal B^{\rm out}\ge ck^{1-b}.
\end{equation}
The same split gives
\begin{equation}\label{ca:ca:saturationlaw}
 \sum_{j\le M}g_j^{\ell-2}\min(1,\tau g_j)^2
 \asymp\tau^{2-\ell+1/a},\qquad \ell=1,2.
\end{equation}
The head and tail are bounded by $Ck^{a(2-\ell)+1}$ and
$C\tau^2k^{1-a\ell}$; the indices near $k$ give the lower bound,
including $k=M$.
Define $K_{\ell,\rm out}=\eta\sum_{u<T}\sum_jg_j^\ell
\ell_{{\rm out},j,u}^2$. Direct summation gives
\begin{gather}
 K_{\ell,\rm last}=\sum_jg_j^{\ell-1}
                \frac{1-r_j^{2T}}{2-\eta g_j}, \quad
 K_{\ell,\rm avg}=\frac{\eta}{\tau^2}
     \sum_{n=1}^T\sum_jg_j^{\ell-2}(1-r_j^n)^2.
 \label{ca:ca:lastmass}
\end{gather}
Use the two intervals
$F=[\lceil T/2\rceil,\lceil3T/4\rceil)$ and
$L_T=[\lceil3T/4\rceil,T)$, each containing at least $T/8$ updates.
The split at $k$ gives $K_{2,\rm last}\le\Theta$ and
$K_{1,\rm last}\le Ck$. The spectral lower bound
\eqref{si:eq:explicitsp} on $F$ gives
$\eta\sum_{u\in F}S_1^H(T,u)\ge ck$.
The first eigenvector gives
$\eta\sum_{u\in L_T}S_2^H(T,u)\ge c$, since
$\eta\lfloor T/4\rfloor\ge\tau/8\ge4$.
For the average, \eqref{ca:ca:saturationlaw} divided by $\tau$ gives
the full upper bounds and matching lower bounds on $F$.
For source weights $0\le w_u\le A_w\tau^\xi$ throughout training and
$w_u\ge a_w\tau^\xi$ on $L_T$ for last/$G$, or on $F$ otherwise,
\eqref{eq:output-response} and $c_J\le J_{jj}\le C_J$ imply
\begin{equation}\label{ca:ca:weightedlaw}
 \sum_{u<T}h_uw_u\mathcal R_u^{\rm out}(G)\asymp\rho\tau^\xi d_{\rm out},
 \qquad
 \sum_{u<T}h_uw_u\mathcal R_u^{\rm out}(J)\asymp\rho\tau^\xi k.
\end{equation}
Here $\xi\ge0$ and $a_w,A_w>0$ are fixed, and the comparison constants
may depend on $a_w,A_w$. A common nonnegative prefactor may multiply
both weight bounds and conclusions.
Use the residual bound \eqref{eq:reference-history-factor} and parameter
bounds \eqref{si:eq:Uupper}--\eqref{si:eq:Ulower}:
$U_u\le C\tau^\nu$ throughout training and $U_u\ge c\tau^\nu$
when $u\ge\lceil T/4\rceil$.
Write $f=f^+$, $\beta=\beta^+$ and $m=m^+$.
The causal sources in \eqref{eq:scheduled-conversion-sources} satisfy
\begin{gather}
 c\rho\beta Ms\tau_t\le Z_t^{b,+}
                 \le CK_s^{\rm ref}\rho\beta Ms\tau_t,\notag\\
 0\le Y_t^{f,+}\le C\rho f\tau^\nu,\qquad
 Y_t^{f,+}\ge c\rho f\tau^\nu\quad(t\ge\lceil T/4\rceil).
 \label{ca:ca:Ybounds}
\end{gather}
For $Z^{b,+}$, sum the residual bounds against $K_{JJ}^H\asymp M$
from \eqref{cl:JJ}. For $Y^{f,+}$, telescoping bounds its kernel mass
by $\eta\sum_{u<t}S_2^H(t,u)\le\Theta$.
For its lower bound, sum over the last quarter of updates before $t$.
There $\tau_u\ge3\tau_t/4\ge3\tau/16$, so $U_u\ge c\tau^\nu$;
the same sum using the first eigenvector at horizon $t$ is bounded below since
$\tau_t\ge\tau/4\ge8$.
Finally apply \eqref{ca:ca:weightedlaw} to $R$ (variance and backward),
$U$ (forward and same-update), $Z^{b,+}$ (backward-to-forward), and
$Y^{f,+}$ (forward-to-backward), multiplying by the coefficients in
\eqref{as:output-sources}. Using \eqref{eq:rounding-source-coefficients}
and $c_g\le\theta_G\le\Theta$ gives the table. Under Assumption~\ref{as:assumption-schedules},
\eqref{as:native-law} gives the comparison with actual SR-induced error.
\end{proof}

\begin{corollary}[Constant-learning-rate exact-arithmetic risk]
\label{ca:constant-reference}
Assume the model and comparison bounds of Definition~\ref{as:def-witness},
$M\ge M_{\rm w}$, $M^a\ge\tau_*$, $S\in\Omega_M$ and integer $B\ge1$.
If $\eta>0$, $\eta g_1\le c_0$, $\rho=\eta/B\le\rho_{\rm sch}$ and
$\tau=\eta T\in[\tau_*,M^a]$, both prescribed outputs satisfy
\begin{equation}\label{ca:constant-reference-laws}
\begin{aligned}
 c\bigl(\tau^{-(b-1)/a}+\rho s\bigr)
 &\le E_0^{\rm last}
 \le CK_s^{\rm ref}\bigl(\tau^{-(b-1)/a}+\rho s\bigr),\\
 c\bigl(\tau^{-(b-1)/a}+\rho s\tau^{1/a-1}\bigr)
 &\le E_0^{\rm avg}
 \le CK_s^{\rm ref}\bigl(\tau^{-(b-1)/a}+\rho s\tau^{1/a-1}\bigr).
\end{aligned}
\end{equation}
Adding $\bar A$ gives full exact-arithmetic prediction risk.
The exact-arithmetic residual bound is \eqref{eq:reference-history-factor}, distinct from
the sketch-noise normalization $K_s$. No rounding-gain condition is needed.
\end{corollary}
\begin{proof}
Theorem~\ref{as:theorem-clean} bounds exact-arithmetic by its bias and
variance functionals. Lemma~\ref{ca:ca:components} evaluates both,
with $K_s^{\rm ref}\ge1$ explicit in the upper.
Substituting $d_{\rm last}=1$ and $d_{\rm avg}=\tau^{1/a-1}$
completes the proof.
\end{proof}

\begin{corollary}[Constant-learning-rate scaling law]
\label{ca:constant-corollary}
Under the hypotheses of Corollary~\ref{ca:constant-reference}, take
constant allowed $q_{\fw},q_{\bw}$, including exact arithmetic at either
location. Under Assumption~\ref{as:assumption-schedules}, the SR-induced
errors satisfy \eqref{eq:law-last} and \eqref{eq:law-avg}; adding
\eqref{eq:clean-law-last} or \eqref{eq:clean-law-avg} gives full risk.
Within the same cap, rate and clock domain, with $C_R^{\rm glob}$ from
\eqref{as:F3},
\begin{equation}\label{ca:ca:Ainclusions}
 \{\Gamma_{\rm stab}\le\kappa/\max(1,C_R^{\rm glob})\}
 \subseteq\{\text{admissible schedules}\}
 \subseteq\{\Gamma_{\rm stab}\le\kappa\}.
\end{equation}
Thus \eqref{s3:eq:constant-gain} is sufficient with
\begin{equation}\label{app:eq:constant-sufficient-gain}
 \kappa_{\rm c}=\frac{\kappa_{\rm sch}}{\max(1,C_R^{\rm glob})}.
\end{equation}
\end{corollary}
\begin{proof}
For constant rates and precisions, the gain definitions
\eqref{as:allclock-gains}--\eqref{as:window-gains} give
\begin{gather}
 \Gamma_{\rm stab}=\rho q_{\fw}q_{\bw}\tau/M,\qquad
 \Gamma_\times=\frac{q_{\fw}q_{\bw}\eta^2}
                         {BM(s+\tau^{-(b-1)/a})}
                         \sum_{u<T-1}R_u,\notag\\
 \tfrac12c_R^{\rm glob}\Gamma_{\rm stab}\le\Gamma_\times
 \le C_R^{\rm glob}\Gamma_{\rm stab},\qquad
 \Gamma_{f\to b}\le C_{\rm fb}\rho\theta_G\le\kappa.
 \label{ca:constant-gains}
\end{gather}
The residual prefix sum is largest at $T-1$; the residual bounds and
$(T-1)/T\in[1/2,1)$ give the comparison of $\Gamma_\times$ with
$\Gamma_{\rm stab}$.
For $\Gamma_{f\to b}$, the maximum is empty if $q_\bw=0$, and its
numerator is zero at $t=0$. Otherwise write its numerator and denominator
as $N_t,D_t$. Telescoping gives
$N_t\le\rho q_\fw\theta_G\max_{u<t}U_u$.
For $\tau_t\ge\tau_*$, \eqref{si:eq:S1exports} gives
$N_t/D_t\le C_{\rm src}\rho\theta_G$; for $\tau_t<\tau_*$,
use $U_u\le C_{\rm src}(1+\tau_*)^\nu$ and $D_t\ge\sigma^2$.
The constant $C_{\rm fb}$ in \eqref{as:F3} covers both cases, giving
\begin{equation}\label{si:port:fb}
 \Gamma_{f\to b}(T)\le C_{\rm fb}\rho\theta_G.
\end{equation}
Since $\rho_{\rm sch}\le\kappa/(C_{\rm fb}\Theta)$, this gain is at
most $\kappa$. The bounds for $\Gamma_\times$ then prove
\eqref{ca:ca:Ainclusions} and the stated sufficient threshold.

Apply \eqref{as:native-law} to Lemma~\ref{ca:ca:components}.
The forward-to-backward term is $\rho$ times the joint scale.
For the last iterate, the joint term is bounded by the forward
term using $q_\bw\tau^{1/a}/M\le1$.
For the uniform average, the backward-to-forward term is bounded by
the backward term using $\rho q_\fw\le\rho_{\rm sch}$.
These multiplicative comparisons include zero rounding coefficients.
Substituting \eqref{ca:ca:dout} yields the SR laws with $s(S)$.
Use $s\le s+\tau^{-(b-1)/a}\le(1+\sigma^{-2})s$ and
Corollary~\ref{risk:noise-comparison} to obtain their stated forms with
$s_M$. Adding the exact-arithmetic and approximation bounds completes
the full-risk comparison.
\end{proof}

\endgroup 
\subsection{Compute-optimal risk with fixed precisions}
\label{app:fixed-precision-frontier}

Proposition~\ref{app:prop-fixed-bit-compute} proves the compute orders in
Table~\ref{tab:compute-precision} by balancing width and training duration.
Use the cost
\eqref{eq:declared-compute} and  
Corollary~\ref{ca:constant-corollary}. Precision is fixed within each run
and across budgets. Exact precision is allowed, with $4^{-\infty}=0$.
For a fixed precision pair write $W=c(p_\fw)+c(p_\bw)>0$; then
$f=(W/\rho)M\tau$, where $\tau=\eta T$ and $\rho=\eta/B$.

\begin{proposition}[Fixed-precision compute orders]
\label{app:prop-fixed-bit-compute}
Fix $B\in\mathbb N$, $0<\eta\le\min\{c_0/C_g,B\rho_{\rm sch}\}$ and
an allowed precision pair independently of $F$.
For all sufficiently large $F$, the optimal uniformly averaged risk
in the full constant-rate class with $f\le F$ has the order in
Table~\ref{tab:compute-precision}. 
\end{proposition}

\begin{proof}
Put $\beta=(b-1)/a$. The bias lower bound is a fixed multiple of
$\tau^{-\beta}$. If $q_\bw=0$, the window $\tau\le M^a$, budget and rate ceiling give
\begin{equation}\label{app:eq:exact-compute-clock-lower}
 F\ge(W/\rho)\tau^{1+1/a}
       \ge(W/\rho_{\rm sch})\tau^{1+1/a},
\end{equation}
proving the first lower bound.
For backward-only rounding, the constant-rate law gives
\begin{equation}\label{app:eq:backward-compute-lower}
 \mathscr R^{\rm avg}\gtrsim
 \tau^{-\beta}+\frac{\rho q_\bw\sigma^2\tau^{1/a}}M
 \gtrsim\tau^{-\beta}+\frac{Wq_\bw\sigma^2}{F}\tau^{1+1/a}.
\end{equation}
Splitting at $\tau=F^{a/(a+b)}$ proves the second lower bound.
When both scales are positive, the necessary stability-gain condition is
$\rho q_\fw q_\bw\tau/M\le\kappa_{\rm sch}$.
Thus
\begin{equation}\label{app:eq:joint-compute-clock-lower}
 F\ge(Wq_\fw q_\bw/\kappa_{\rm sch})\tau^2,
\end{equation}
which yields the third lower bound. 

For the three upper bounds, respectively set $r=1,b,a$. Choose a fixed
large $A\ge1$, also satisfying
$A\ge\rho q_\fw q_\bw/\kappa_{\rm c}$ in the third case, and take
\[
 T_F=\left\lfloor dF^{a/(a+r)}\right\rfloor,\qquad
 \tau_F=\eta T_F,\qquad M_F=\left\lceil A\tau_F^{r/a}\right\rceil.
\]
A sufficiently small fixed $d>0$ makes $f\le F$ eventually. Both width
thresholds and the minimum cumulative learning rate are met; the cap and rate ceiling were fixed above,
and $M_F\ge\tau_F^{1/a}$ verifies the clock window. The mixed gain is
zero in the first two cases and at most $\kappa_{\rm c}$ in the third,
so \eqref{s3:eq:constant-gain} verifies all feedback bounds.

It remains to check risk. Approximation error, bias and exact-SGD
variance are $O(\tau_F^{-\beta})$, since $r\ge1$ and $a>b$.
Forward rounding contributes at most $O(\tau_F^{-\beta})$.
Backward rounding is absent in the first case, is
$O(\tau_F^{-\beta})$ in the second, and is
$O(\tau_F^{(1-a)/a})=O(\tau_F^{-\beta})$ in the third.
The third case's joint term is
$O(\tau_F^{(2-a-b)/a})=O(\tau_F^{-\beta})$.
Corollary~\ref{ca:constant-corollary} therefore gives all upper bounds,
with $\tau_F\asymp F^{a/(a+r)}$.
\end{proof}
 
\subsection{Compute-optimal scaling with optimized learning rate and batch size}
\label{r4:sec:compute-consequences}

Proposition~\ref{app:prop-fixed-bit-compute} fixes the learning rate and
batch size across budgets. We now optimize both as functions of $F$,
while keeping them constant within each run. Uniform averaging has
Table~\ref{tab:compute-precision}'s risk exponents; the last-iterate
result also accounts for reducing the SGD-noise contribution.

Use the same fixed model 
with widths satisfying \eqref{as:main-width}.
Fix one allowed precision pair independently of $F$, and write
$\beta_0=(b-1)/a$ and $W_p=c(p_\fw)+c(p_\bw)>0$.
A configuration $\pi$ specifies integers $M,T,B\ge1$ and a constant
rate $\eta>0$, all chosen before the random draws.
Let $\mathcal C_p(F)$ contain those with nonempty $\Omega_M$, both width
conditions, the full constant-schedule assumptions of
Corollary~\ref{ca:constant-corollary} for every $S\in\Omega_M$, and
$f=BTMW_p=(W_p/\rho)M\tau\le F$.
All three feedback bounds in Assumption~\ref{as:assumption-schedules}
remain required.

Below $c,C,F_0>0$ depend only on the fixed model, comparison bounds,
admissibility thresholds and precision pair, including $K_s,K_s^{\rm ref}$.
They are uniform in $M,T,B,\eta,S$, but need not remain uniform as a
positive rounding scale tends to zero. Statements hold for $F\ge F_0$.

\begin{proposition}[Averaged compute-optimal risk]
\label{r4:prop-variable-rho-average}
Let $\gamma_{\rm avg}$ be the risk exponent in the corresponding row of
Table~\ref{tab:compute-precision}. Every $\pi\in\mathcal C_p(F)$ satisfies
$\mathscr R_S^{\rm avg}(\pi)\ge cF^{-\gamma_{\rm avg}}$ on
$\Omega_{M(\pi)}$. Conversely, a deterministic configuration
$\pi_F\in\mathcal C_p(F)$ attains
$\mathscr R_S^{\rm avg}(\pi_F)\le CF^{-\gamma_{\rm avg}}$ uniformly on
its event. Thus optimizing the constant learning rate and batch
size leaves Table~\ref{tab:compute-precision}'s optimal risk orders unchanged.
\end{proposition}
\begin{proof}
The lower-bound argument of
Proposition~\ref{app:prop-fixed-bit-compute} is uniform in the actual
admissible $\rho$. For $q_\bw=0$,
\eqref{app:eq:exact-compute-clock-lower} uses only its fixed ceiling.
For backward-only rounding, \eqref{app:eq:backward-compute-lower}
eliminates $\rho$ using the cost identity; for joint rounding,
\eqref{app:eq:joint-compute-clock-lower} eliminates it using the
necessary stability gain. Together with the common bias lower bound,
these same inequalities prove all three lower orders even when the
constant rate and batch vary with $F$.

For the uppers, fix $B=1$ and any sufficiently small fixed
$\eta_0\le\min\{c_0/C_g,\rho_{\rm sch}\}$. The integer policies of
Proposition~\ref{app:prop-fixed-bit-compute} belong to $\mathcal C_p(F)$
and attain precisely the three orders in
Table~\ref{tab:compute-precision}. Finally, $f\ge W_pM$ leaves
only finitely many budget-feasible widths. The pointwise lower and
uniform attainer prove the infimum statement without a joint-width
event or any coupling of sketches across widths.
\end{proof}

\begin{proposition}[Last-iterate compute orders with variable constant rate]
\label{r4:prop-variable-rho-last}
For the same fixed precisions and class $\mathcal C_p(F)$, the preceding
pointwise lower, uniform integer-attainment and conditional-infimum
 conclusions hold for the last iterate with exponent
\begin{equation}\label{r4:eq:variable-rho-last-exponent}
 \gamma_{\rm last}=\begin{cases}
 (b-1)/(a+b),&q_\fw=0,\\
 (b-1)/(2a),&q_\fw>0.
 \end{cases}
\end{equation}
Both exact and backward-only arithmetic have the first exponent;
forward-only and joint rounding have the second. Constants have the
same fixed-dependence and nonuniformity convention as above.
This statement optimizes constant-within-run rates and batches, not
arbitrary learning-rate schedules or iterate weights.
\end{proposition}
\begin{proof}
The exact-arithmetic SGD risk lower bound and capacity/cost inequalities give
\begin{equation}\label{r4:eq:constant-last-noise-lower}
 \mathscr R^{\rm last}\ge
 c\{\tau^{-\beta_0}+\rho\sigma^2\},\qquad
 \rho\ge\frac{W_pM\tau}{F}
       \ge\frac{W_p}{F}\tau^{1+1/a}.
\end{equation}
Splitting at $\tau=F^{a/(a+b)}$ proves the first lower order.
When $q_\fw>0$, the direct forward term also gives
\begin{equation}\label{r4:eq:constant-last-forward-lower}
 \mathscr R^{\rm last}\ge
 c\tau^{-\beta_0}
 +c'\frac{\rho q_\fw}{M}\tau^{1-\beta_0}
 \ge c\tau^{-\beta_0}
 +c'\frac{W_pq_\fw}{F}\tau^{2-\beta_0}.
\end{equation}
Here $0<\beta_0<1$, since $1<b<a$. Splitting at $\tau=F^{1/2}$
proves the second lower order, including the forward-only case in
which the mixed gain vanishes.

We give one integer construction for both upper orders. Set $r=b$
if $q_\fw=0$ and $r=a$ otherwise. With $\kappa_{\rm c}$ from
\eqref{app:eq:constant-sufficient-gain}, choose a fixed $d>0$ satisfying
\begin{equation}\label{r4:eq:last-attainer-slack}
 d\le\tfrac12\min\left\{1,\rho_{\rm sch},c_0/C_g,
                \kappa_{\rm c}/\max(1,q_\fw q_\bw)\right\}.
\end{equation}
Before any random draws, choose
\begin{equation}\label{r4:eq:constant-last-integer-attainer}
 M_F=\left\lfloor\left(\frac{dF}{2W_p}\right)^{1/(a+r)}\right\rfloor,
 \quad B_F=1,\quad \eta_F=dM_F^{1-r},\quad
 T_F=\left\lfloor\frac{M_F^a}{\eta_F}\right\rfloor.
\end{equation}
For sufficiently large $F$, $M_F$ exceeds the width conditions and
has a nonempty event, and
$M_F^a/2\le\tau_F=\eta_FT_F\le M_F^a$ verifies the lower and upper bounds on cumulative learning rate.
The event bound $g_1\le C_g$ and \eqref{r4:eq:last-attainer-slack}
verify the step and rate caps. If $q_\fw=0$, all mixed terms vanish.
Otherwise $r=a$ and $\Gamma_{\rm stab} \le\kappa_{\rm c} $ where $\Gamma_{\rm stab}= \rho_Fq_\fw q_\bw\tau_F/M_F.$
The verifications in \eqref{ca:ca:Ainclusions} hold for every
$S\in\Omega_{M_F}$. Moreover, $f_F=W_pM_FT_F\le\frac{W_p}{d}M_F^{a+r}\le F/2.$

At matched clock, approximation and bias have order $M_F^{1-b}$;
$\sigma^2\le s(S)\le K_s\sigma^2$ and
$s(S)+\tau_F^{-\beta_0}\le C$ uniformly.
For $r=b$ the exact-arithmetic minibatch-noise and backward-rounding terms are
$O(\rho_F)=O(M_F^{1-b})$, and forward/joint errors are zero.
For $r=a$ the forward error is
$O(\rho_F q_\fw M_F^{a-b})=O(M_F^{1-b})$;
the exact-arithmetic minibatch-noise and backward-rounding terms are $O(M_F^{1-a})$.
The last-iterate joint error is
$O(\rho_F^2q_\fw q_\bw M_F^{a-1})=O(M_F^{1-a})$.
Because $a>b$, all terms are $O(M_F^{1-b})$.
Corollary~\ref{ca:constant-corollary} therefore gives
$\mathscr R_S^{\rm last}(\pi_F)\le CM_F^{1-b}
\asymp F^{-(b-1)/(a+r)}$ uniformly on the event.
The conditional-infimum statement follows as in the preceding proof.
\end{proof}

\section{Implications for precision and training design}
\label{app:rules}

This appendix derives the design consequences discussed in
Section~\ref{sec:answers}. Proposition~\ref{prop:regime} compares rounding
locations and outputs; Propositions~\ref{t3:prop-timing}--\ref{app:prop-joint-phase-allocation}
address precision allocation across training; and
Corollary~\ref{cor:bits-full} gives bit requirements for a prescribed risk
tolerance. Lemma~\ref{prop:lr-full} and Theorem~\ref{c2:arc-main} develop
learning-rate schedules, while Proposition~\ref{app:prop-matched-compute}
and Theorem~\ref{new:global-compute} establish the compute consequences of
growing finite precisions and allowing exact arithmetic.
Only the last iterate and uniform average are considered with
Assumption~\ref{as:assumption-schedules}.  
Use the constant-schedule notation of
Section~\ref{sec:constant-schedule}. Corollary~\ref{risk:noise-comparison}
justifies replacing sketch-dependent noise by its deterministic scale,
including when the bias term is added. 

\subsection{Where: rounding location and iterate averaging}
\label{t3:sub-regime}

Proposition~\ref{prop:regime} compares the two outputs and identifies which
feature computation is more sensitive to rounding.
Fix common $M,\eta,B,T$ and $q=4^{-p}>0$, $p\in\mathbb N_0$, such
that both rounding choices $(q_\fw,q_\bw)=(q,0)$ and $(0,q)$ satisfy
Corollary~\ref{ca:constant-corollary}.
\begin{proposition}[Comparisons with rounding at one location]
\label{prop:regime}\label{cor:floor}
The two outputs satisfy
\begin{equation}\label{app:eq:output-comparison}
 \frac{\Delta^{\rm last}(q,0)}{\Delta^{\rm avg}(q,0)}
       \asymp\tau^{(a-1)/a},\qquad
 \Delta^{\rm last}(0,q)\asymp\Delta^{\rm avg}(0,q).
\end{equation}
If $c_\tau M^a\le\tau\le M^a$ for fixed $c_\tau\in(0,1]$,
\eqref{eq:location-ratio} holds. Constants in comparisons at this
matched clock may also depend on $c_\tau$.
For backward-only rounding,
\begin{equation}\label{app:eq:backward-relative}
 \frac{\Delta^{\rm last}(0,q_\bw)}{\mathscr R_0^{\rm last}(T)}
 \lesssim q_\bw\frac{\tau^{1/a}}M\le q_\bw.
\end{equation}
At a fixed-factor matched clock, the averaged counterpart is
\begin{equation}\label{app:eq:backward-average-relative}
 \frac{\Delta^{\rm avg}(0,q_\bw)}{\mathscr R_0^{\rm avg}(T)}
 \asymp\frac{\rho q_\bw(s_M+M^{1-b})}
                 {M^{1-b}+\rho s_M M^{1-a}}.
\end{equation}
\end{proposition}

\begin{proof}\label{app:proof-regime}
 Set the unused rounding scale to zero in \eqref{eq:law-last} and
\eqref{eq:law-avg}. Dividing the forward-only laws across the two outputs
gives the first comparison in \eqref{app:eq:output-comparison}; the
backward-only expressions coincide up to fixed factors.
Dividing the two location laws and then substituting $c_\tau M^a\le\tau\le M^a$ and
$s_M+M^{1-b}\asymp\sigma^2+M^{1-b}$ proves
\eqref{eq:location-ratio}.

For \eqref{app:eq:backward-relative}, the exact-arithmetic SGD risk lower bound in
Corollary~\ref{ca:constant-reference} is a fixed multiple of
$\tau^{-(b-1)/a}+\rho s_M$. The quotient
\[
 \frac{\rho(s_M+\tau^{-(b-1)/a})}
      {\tau^{-(b-1)/a}+\rho s_M}
 \le\max\{1,\rho\}\le\max\{1,\rho_{\rm sch}\}
\]
proves the first inequality, and $\tau\le M^a$ proves the second.
Finally, adding the approximation term to
\eqref{ca:constant-reference-laws} and evaluating at the matched clock
gives the denominator in \eqref{app:eq:backward-average-relative}.
\end{proof}

\subsection{When: allocating precision across updates}
\label{t3:sub-timing}

We first quantify sensitivity to rounding at each update in
Proposition~\ref{t3:prop-timing}, then allocate a precision budget in
Propositions~\ref{prop:when-full} and~\ref{app:prop-joint-phase-allocation}.
Fix the width, batch, horizon, deterministic learning rates and output.
Keep one rounding location exact throughout training and allocate
finite precisions at the other. Use the coefficients and scalar objective
in \eqref{eq:single-schedule-objective}; for nonconstant learning rates,
replace $\eta^2/B$ there by $h_u$. Thus the allocation uses the
exact-arithmetic source and output responses already defined in
Section~\ref{sec:general-schedules}.

\begin{proposition}[Single-location sensitivity and timing]
\label{t3:prop-timing}\label{app:cor-single-scheduled}
For either output, \eqref{eq:single-schedule-objective} holds with the
replacement above.
The coefficients are nonnegative and independent of the allocated
precisions. At constant positive learning rate, the backward
coefficients $A_{\bw,u}$ are nondecreasing in $u$ for the last iterate
and nonincreasing in $u$ for the uniform average. Thus the real-bit
optimizer of Proposition~\ref{prop:when-full} has the corresponding
ordering of backward precisions.
\end{proposition}

\begin{proof}
Setting one precision scale to zero throughout
\eqref{s3:eq:scalar-rounding} removes every joint term and the direct
term at that location. The remaining expression is
\eqref{eq:single-schedule-objective}; apply
Theorem~\ref{thm:scheduled-main} for its risk comparison.
For the last iterate, the exponent $1/a-1<0$ in
\eqref{s3:eq:scalar-responses_last} makes $K_1^{\rm last}(u)$ increase
as the remaining cumulative rate decreases.
For the uniform average, reindex its double sum by
$i=v-u-1$ and $j=w-u-1$. At constant rate it is
\[
 K_1^{\rm avg}(u)=\frac1{T^2}
 \sum_{i,j=0}^{T-u-1}[1+\eta(i+j)]^{1/a-1}.
\]
Increasing $u$ removes nonnegative terms. The prefactor
$h_u\sigma^2/(2M)$ is constant, proving the two coefficient orderings.
The optimizer is a nondecreasing function of each coefficient, which
proves its precision ordering. Forward coefficients also contain the
changing parameter source $\widehat U_u$, so this argument does not
assert the same monotonicity for them.
\end{proof}

\subsubsection{A shared precision budget}
\label{app:sub-allocation}
With one exact computation and one finite-precision computation, the
cost \eqref{eq:declared-compute} is equivalent to
\[
 \sum_{u<T}p_{r,u}\le P,\qquad
 P=\frac{(p_{\rm ref}+p_0)F}{BM}-T(p_{\rm ref}+2p_0).
\]
The fixed term charges the exact computation and the sign/exponent
costs. Feasibility requires $P\ge0$. Zero mantissa bits mean $q=1$;
they do not mean exact arithmetic. Exact arithmetic has $q=0$ and is
charged separately by the compute definition.

Write $\widehat N(p)=\sum_iA_i4^{-p_i}$ for
\eqref{eq:single-schedule-objective}, with $A_i\ge0$ and budget $P>0$.
Let $n$ count the positive coefficients and, when $n>0$, set
$\operatorname{GM}(A)=(\prod_{i:A_i>0}A_i)^{1/n}$.

\begin{proposition}[Optimal allocation for an additive SR-error objective]
\label{prop:when-full}\label{prop:when}\label{prop:lagrange}
\label{app:prop-integer-allocation}\label{app:lem-objective-transfer}
If $n>0$, the unique real minimizer under $p_i\ge0$ and
$\sum_i p_i\le P$ is \eqref{eq:lagrange}, with $i$ replacing $(r,u)$
and $\lambda>0$ determined by $\sum_i p_i=P$.
If every coefficient is positive and every optimal $p_i>0$, this gives
\begin{equation}\label{eq:lagrange-full}
 p_i=\frac Pn+\log_4\frac{A_i}{\operatorname{GM}(A)},\qquad
 \frac{\sum_iA_i4^{-p_i}}{\sum_iA_i4^{-P/n}}
 =\frac{\operatorname{GM}(A)}{n^{-1}\sum_iA_i}.
\end{equation}
For integer mantissas, use budget $\lfloor P\rfloor$ and assign each
successive bit where it decreases $\widehat N$ most; this is optimal
when all nonnegative integers are allowed. Flooring a feasible real
allocation increases $\widehat N$ by a factor of at most four.

On a nonempty admissible family with uniform bounds
$c_\Delta\widehat N\le\Delta\le C_\Delta\widehat N$,
$0<c_\Delta\le C_\Delta$, an exact scalar minimizer achieves SR-induced
error within a factor $C_\Delta/c_\Delta$ of the minimum in that family.
The same factor bounds full risk when exact-arithmetic risk is common.
\end{proposition}

\begin{proof}\label{app:proof-when}
On positive-weight coordinates the objective is strictly convex and
strictly decreasing. Consequently its minimizer exhausts the budget.
The optimality conditions give
$A_i4^{-p_i}=\lambda$ when $p_i>0$ and $A_i\le\lambda$ when $p_i=0$,
which yields \eqref{eq:lagrange}.
The sum $\sum_{i:A_i>0}[\log_4(A_i/\lambda)]_+$ decreases continuously from
infinity to zero and is strictly decreasing when positive, proving
uniqueness of the multiplier and the allocation. Zero-weight
coordinates receive no bits. If the lower constraints are inactive,
summing the logarithms yields \eqref{eq:lagrange-full}, with minimum
$n\operatorname{GM}(A)4^{-P/n}$.

The gain from increasing an integer mantissa from $k$ to $k+1$ is
$3A_i4^{-(k+1)}$. These gains decrease along each coordinate. Selecting
the largest remaining gain at each step therefore selects the largest
$P$ gains overall, with all predecessor gains already selected; this
proves integer optimality. Since $4^{-\lfloor p_i\rfloor}\le4\,4^{-p_i}$,
flooring increases the objective by at most four.

Finally, for any feasible comparator $p$ and an objective minimizer
$p_*$,
$\Delta(p_*)\le C_\Delta\widehat N(p_*)\le C_\Delta\widehat N(p)
\le(C_\Delta/c_\Delta)\Delta(p)$.
Taking the infimum proves the claim. Adding the common nonnegative
exact-arithmetic risk preserves the same factor.
The boundary case $P=0$ forces zero allocation; if every coefficient
vanishes, zero allocation is optimal.
\end{proof}

This proposition optimizes the scalar comparison expression, rather
than identifying an exact minimizer of actual binary-SR risk. With one
location exact, all joint feedback gains vanish. Remaining rate, width
and clock assumptions still apply. Restricted precision menus require
a separate feasible search rather than unrestricted integer greedy
allocation.

\subsubsection{Joint precision allocation across training phases}
\label{app:joint-phase-allocation}

We minimize the joint SR-error objective over predetermined training
phases and a finite precision menu, subject to the compute budget and a
sufficient feedback test. The guarantee compares schedules within this
same feasible family.

Fix eligible width $M$, batch $B$, horizon $T$ and deterministic learning
rates, and use either the last iterate or the uniformly averaged iterate.
Require the pointwise cap,
rate ceiling and clock conditions of
Assumption~\ref{as:assumption-schedules} throughout $\Omega_M$, with the
same fixed comparison bounds and both width conditions. All quantities below are chosen
before the random draws. Use $\widehat U$, $K_r^{\iout}$ and $h_u$
from Section~\ref{sec:dynamics}, and the direct coefficients
$A_{r,u}$ in \eqref{eq:single-schedule-objective} (with $\eta^2/B$ replaced by $h_u$). Define
\begin{equation}
 C_{uv}=\begin{cases}
 \dfrac{h_u\widehat U_u}{2M^2}K_1^{\iout}(u),&u=v,\\[2pt]
 \dfrac{h_uh_v\widehat U_u}{2M^2}
 [1+2(\tau_v-\tau_{u+1})]^{1/a-2}K_1^{\iout}(v),&u<v,\\[2pt]
 \dfrac{\sigma^2h_uh_v}{2M}K_2^{\iout}(u),&v<u.
 \end{cases}
 \label{app:eq:joint-coefficients}
\end{equation}
The row index of $C$ is the forward-rounding update and its column index
is the backward-rounding update. In particular, $u<v$ represents earlier
forward and later backward rounding; $v<u$ represents the reverse order.
All coefficients are nonnegative and independent of the precision choices.
In this appendix, $p$ denotes the pair of precision schedules
$\{p_{\fw,u},p_{\bw,u}\}_{u<T}$. Collecting
\eqref{s3:eq:scalar-rounding} gives the exact scalar identity with $ 4^{-\infty}=0$:
\begin{equation}\label{eq:joint-schedule-objective}
\begin{aligned}
 \widehat N^{\iout}(p)
 ={}&\sum_{u<T}\left[A_{\fw,u}4^{-p_{\fw,u}}
                 +A_{\bw,u}4^{-p_{\bw,u}}\right] +\sum_{u,v<T}C_{uv}4^{-(p_{\fw,u}+p_{\bw,v})}.
\end{aligned}
\end{equation}

\paragraph{A deterministic sufficient feasibility certificate.}
Use $R_+$ and $u_+$ from \eqref{cl:Uhat-constants}, and let $u_2$ be the upper
spectral coefficient in \eqref{cl:heat}. They are fixed
uniform constants on the sketch event. Use the deterministic
$\Gamma_{\rm stab}(T)$ of \eqref{as:allclock-gains}. For a precision
schedule $p$, define the two additional upper bounds
\begin{align}
 \overline\Gamma_{f\to b}(p)
   &=\frac{u_+u_2}{M\sigma^2}
       \max_{t<T:q_{\bw,t}>0}
       \sum_{u<t}h_uq_{\fw,u}\widehat U_u
             [1+2(\tau_t-\tau_{u+1})]^{1/a-2},\notag\\
 \overline\Gamma_\times(p)
   &=\frac{R_+\sigma^2}{M[\sigma^2+\tau_T^{-(b-1)/a}]}
       \max_{t<T}q_{\fw,t}\sum_{u<t}h_uq_{\bw,u}.
 \label{app:eq:joint-safe-gains}
\end{align}
The maximum over an empty set is zero. The certificate requires
$\Gamma_{\rm stab}(T)$ and both displayed upper bounds to be at most
$\kappa_{\rm sch}$. It is sufficient,
not necessary, for the original feedback conditions. Keeping features exact
at either rounding location throughout training makes all three quantities zero.

Fix a partition $\{I_j\}_{j=1}^J$ of the update indices into nonempty
predetermined phases and a finite menu
$\mathcal M\subset\mathbb N_0\cup\{\infty\}$, allowing exact arithmetic
only when the menu contains it. Let $\mathcal P_J^{\rm safe}(F)$ be
the assignments $p_{r,u}=p_{r,j}$ for $u\in I_j$, $p_{r,j}\in\mathcal M$,
which obey the literal charge \eqref{eq:declared-compute} and the
certificate above, including the stability bound and both conversion
upper bounds in \eqref{app:eq:joint-safe-gains}. For finite precision at both rounding locations,
the charge is equivalent to
\begin{equation}\label{app:eq:joint-bit-budget}
 \sum_{j=1}^J|I_j|(p_{\fw,j}+p_{\bw,j})\le
 \frac{(p_{\rm ref}+p_0)F}{BM}-2Tp_0.
\end{equation}
For configurations combining exact arithmetic and finite precision, use
\eqref{eq:declared-compute}.
Assume $\mathcal P_J^{\rm safe}(F)\ne\varnothing$.

\begin{proposition}[Certified finite-family precision design]
\label{app:prop-joint-phase-allocation}
For the fixed data above, \eqref{eq:joint-schedule-objective} holds
identically for the scalar objective. It can be evaluated on phases by
aggregating $A_{r,j}=\sum_{u\in I_j}A_{r,u}$ and
$C_{jk}^{\rm phase}=\sum_{u\in I_j,v\in I_k}C_{uv}$.
Exact enumeration returns a minimizer $p_*$ of $\widehat N^{\iout}$ on
$\mathcal P_J^{\rm safe}(F)$ after examining at most
$|\mathcal M|^{2J}$ assignments. Every returned schedule belongs to
the full admissible class throughout $\Omega_M$.
Let $0<c_\Delta\le C_\Delta$ be the uniform constants for
$c_\Delta\widehat N^{\iout}\le\Delta^{\iout}\le C_\Delta\widehat N^{\iout}$ in
Theorem~\ref{thm:scheduled-main}. Then, for every $S\in\Omega_M$,
\begin{align}
 \Delta_S^{\iout}(p_*)&\le\frac{C_\Delta}{c_\Delta}
                  \min_{p\in\mathcal P_J^{\rm safe}(F)}
                            \Delta_S^{\iout}(p),\notag\\
 \mathscr R_S^{\iout}(p_*)&\le\frac{C_\Delta}{c_\Delta}
                  \min_{p\in\mathcal P_J^{\rm safe}(F)}
                            \mathscr R_S^{\iout}(p).
 \label{app:eq:joint-phase-risk-transfer}
\end{align}
\end{proposition}

\begin{proof}
Collecting the direct and ordered interaction terms in
\eqref{s3:eq:scalar-rounding} gives
\eqref{eq:joint-schedule-objective}. Summing their coefficients over
the phase indices gives the stated aggregation.

The exact-arithmetic bounds \eqref{cd:Csigma},
\eqref{eq:noise-normalization}, the parameter comparison established
using \eqref{cl:Uhat-constants} in the proof of
Theorem~\ref{cl:objective}, and \eqref{cl:heat} use only the rate,
clock and spectral hypotheses. Substituting these bounds into
\eqref{as:allclock-gains}--\eqref{as:window-gains}, with
$MR_t+q_{\fw,t}U_t\ge M\sigma^2$ and
$r_T\ge\sigma^2+\tau_T^{-(b-1)/a}$, gives
$\Gamma_{f\to b}\le\overline\Gamma_{f\to b}$ and
$\Gamma_\times\le\overline\Gamma_\times$.
Thus the certificate implies full admissibility throughout $\Omega_M$
without assuming the feedback conditions being verified.
The feasible family is nonempty and has at most $|\mathcal M|^{2J}$
members, so enumeration finds a scalar-objective minimizer.
Apply Proposition~\ref{prop:when-full} with
the uniform constants of Theorem~\ref{thm:scheduled-main}.
The exact-arithmetic risk is common because width, batch, horizon,
rates and output are fixed, proving both bounds in
\eqref{app:eq:joint-phase-risk-transfer}.
\end{proof}

\subsection{How many bits: controlling the risk increase}
\label{t3:sub-bits}

Corollary~\ref{cor:bits-full} proves \eqref{eq:accuracy-bit-thresholds}
for the relative-risk requirement \eqref{eq:tolerance}, using
\eqref{eq:clean-law-last}--\eqref{eq:law-avg}.
Fix either output under the hypotheses of
Corollary~\ref{ca:constant-corollary}, and suppress its superscript.
Let $D$ be the sum on the right of \eqref{eq:clean-law-last} or
\eqref{eq:clean-law-avg}, including approximation error.
Let $A_\fw,A_\bw,A_\times$ be the coefficients of
$q_\fw,q_\bw,q_\fw q_\bw$ in \eqref{eq:law-last} or
\eqref{eq:law-avg}, and write
$N=A_\fw q_\fw+A_\bw q_\bw+A_\times q_\fw q_\bw$.
Then
\begin{equation}\label{app:eq:reference-rounding-bounds}
 c_{\rm R}D\le\mathscr R_0(T)\le C_{\rm R}D,
 \qquad c_\Delta N\le\Delta(T)\le C_\Delta N.
\end{equation}
The constants have the fixed dependencies and uniformity of
Corollaries~\ref{ca:constant-reference}--\ref{ca:constant-corollary}.
At $c_\tau M^a\le\tau\le M^a$ for fixed $c_\tau\in(0,1]$,
Table~\ref{tab:where} gives equivalent scales.

\begin{corollary}[Constant-learning-rate precision thresholds]
\label{cor:bits-full}\label{cor:bits}\label{app:prop-interaction}
\label{t3:cor-bits}
For $\vartheta>0$ and at least one rounded location,
\eqref{eq:accuracy-bit-thresholds} is sufficient for
\eqref{eq:tolerance} with $C=nC_\Delta/c_{\rm R}$ and necessary with
$C=c_\Delta/C_{\rm R}$.
Here $n=1$ when one location is rounded and $n=3$ when both are rounded. 
\end{corollary}

\begin{proof}\label{app:proof-bits}
Corollaries~\ref{ca:constant-reference}--\ref{ca:constant-corollary},
the approximation bounds on $\Omega_M$, and the uniform noise
replacement prove \eqref{app:eq:reference-rounding-bounds}.
At $c_\tau M^a\le\tau\le M^a$, each clock power differs from its
$\tau=M^a$ value by fixed factors depending on $c_\tau$. This proves
the same comparisons for Table~\ref{tab:where},
including its exact-arithmetic SGD risk row.
Write $C_-=c_\Delta/C_{\rm R}$ and $C_+=C_\Delta/c_{\rm R}$.
The central joint test is
\begin{equation}\label{app:eq:component-budget}
 C_+N\le\vartheta D\quad\Longrightarrow\quad
 \Delta\le\vartheta\mathscr R_0,
 \qquad
 \Delta\le\vartheta\mathscr R_0\quad\Longrightarrow\quad
 C_-N\le\vartheta D.
\end{equation}
The sufficient bound allows arbitrary tradeoffs among the nonnegative
terms. Bounding each active term by
$\vartheta D/(nC_+)$ is one conservative way to enforce it; it is not
required that the contributions be equal. Conversely, the necessary
bound on the sum requires every active term to be at most
$\vartheta D/C_-$. Substitute $q_\fw=4^{-p_\fw}$ and $q_\bw=4^{-p_\bw}$ in the
active bounds and take base-four logarithms. Nonnegative mantissas give
the positive parts in \eqref{eq:accuracy-bit-thresholds}, with
$C=nC_+=nC_\Delta/c_{\rm R}$ for sufficiency and
$C=C_-=c_\Delta/C_{\rm R}$ for necessity.
For sufficiency, round real bounds upward to allowed integer mantissas
and verify admissibility. If both computations are exact, $N=\Delta=0$.
\end{proof}

\paragraph{A necessary backward precision that grows with width.}
For the uniform average at $\tau\asymp M^a$, fix positive
$\rho,\sigma^2,\vartheta$ and keep forward features exact.
Then $D^{\rm avg}\asymp M^{1-b}$ and
$A_\bw\asymp\rho\sigma^2$. The lower bound for $\Delta$ and upper
bound for $\mathscr R_0$, rather than the sufficient tests alone, imply
\[
 q_\bw\le C\frac{\vartheta M^{1-b}}{\rho\sigma^2},\qquad
 p_\bw\ge\left[(b-1)\log_4M+
            \log_4\frac{c\rho\sigma^2}{\vartheta}\right]_+.
\]
This is the necessary condition stated in
Section~\ref{sec:howmuch}. With both locations rounded, the joint term
also gives a condition on $p_\fw+p_\bw$; separate one-location accuracy
tests do not control that additional term.

\paragraph{Consistency versus preserving the exact-arithmetic rate.}
Fix positive $\rho,\sigma^2$ and let $M,\tau\to\infty$ in the full
constant class. The window and $a>b$ give
$\mathscr R_0^{\rm avg}\asymp\tau^{-(b-1)/a}$, so the general-clock
relative comparison is
\begin{equation}\label{s3:eq:average-relative}
\begin{aligned}
 \frac{\Delta^{\rm avg}}{\mathscr R_0^{\rm avg}}
 \asymp{}&\frac{\rho q_\fw\tau^{1/a}}M
 +\frac{\rho q_\bw(s_M+\tau^{-(b-1)/a})\tau^{b/a}}M
 +\frac{\rho q_\fw q_\bw\tau^{1+1/a}}{M^2}.
\end{aligned}
\end{equation}
The first term is bounded by the rate ceiling. The third is bounded by
a fixed multiple of the stability gain, because $\tau^{1/a}\le M$.
Their absolute contributions vanish with the exact-arithmetic risk.
Since $s_M+\tau^{-(b-1)/a}\to\sigma^2>0$, the remaining term yields
\begin{equation}\label{app:eq:average-consistency-precision}
\begin{aligned}
 \mathscr R^{\rm avg}\to0
 &\quad\Longleftrightarrow\quad q_\bw\tau^{1/a}/M\to0,\\
 \Delta^{\rm avg}=O(\mathscr R_0^{\rm avg})
 &\quad\Longleftrightarrow\quad q_\bw\tau^{b/a}/M=O(1).
\end{aligned}
\end{equation}
Nonnegativity gives necessity and the preceding bounds give sufficiency.
These distinctions explain Section~\ref{sec:constant-schedule}'s
learning-curve discussion.

\paragraph{Sufficient width tests for a desired training duration.}
For large $\tau$ and fixed positive noise, one sufficient construction
bounds each of the three terms in $C_\Delta N$ by
$\vartheta c_{\rm R}\tau^{-(b-1)/a}/3$. This is a conservative way to
enforce the joint budget \eqref{app:eq:component-budget}. It gives
\begin{equation}\label{app:eq:accuracy-first-width-tests}
\begin{aligned}
 M&\ge\tau^{1/a},&
 M&\ge\frac{3C_\Delta\rho q_\fw}{\vartheta c_{\rm R}}\tau^{1/a},\\
 M&\ge\frac{3C_\Delta\rho q_\bw}{\vartheta c_{\rm R}}
       (s_M+\tau^{-\frac{b-1}{a}})\tau^{b/a},&
 M^2&\ge\frac{3C_\Delta\rho q_\fw q_\bw}{\vartheta c_{\rm R}}
           \tau^{1+1/a},\\
 M&\ge\frac{\rho q_\fw q_\bw}{\kappa_{\rm c}}\tau.
\end{aligned}
\end{equation}
The first is the clock window, the next three follow by substituting
the coefficients of \eqref{eq:law-avg} into the component budgets,
and the last verifies the sufficient feedback test. Together with the
remaining width, minimum cumulative learning rate and rate conditions they ensure the desired
relative accuracy. They are sufficient construction rules, not all
necessary width restrictions. Deleting an exact location deletes its
associated terms. These are the complete checks behind the leading
width scales in \eqref{r4:eq:compute-width-requirements}.
The last condition is simply
\eqref{s3:eq:constant-gain} rearranged as a lower bound on $M$.

\subsection{Learning-rate design with stochastic rounding}
\label{t3:sub-lr}

Lemma~\ref{prop:lr-full} solves a continuous relaxation at fixed total
learning rate. Theorem~\ref{c2:arc-main} then chooses that total and
constructs discrete SGD rates with a fixed-factor risk guarantee. The
inverse-square-root calculation follows
\citet[Appendix~B.1]{bordelon2026optimal}. No continuous-time limit of SGD is taken.

Fix $\tau,h,T>0$ with $\tau\le hT$, and let
$K:[0,\tau]\to(0,\infty)$ be continuous. For measurable rates
$0\le\eta\le h$ with $\int_0^T\eta=\tau$, set
$\chi(t)=\int_t^T\eta(s)\,ds$. Write $H(x)=\int_0^x\sqrt{K(v)}\,dv$ and consider
\begin{equation}\label{app:eq:lr-objective}
 \mathcal V[\eta]=\int_0^T\eta(t)^2K(\chi(t))\,dt.
\end{equation}

\begin{lemma}[Optimal learning rate at fixed total rate]
\label{prop:lr-full}\label{prop:lr}\label{t3:lem-lr-scalar}
For $T>\tau/h$, the unique minimizer is $\eta(t)=a(\chi(t))$, where
\begin{equation}\label{eq:kernel}
 a(\chi)=\min\{h,\sqrt{\lambda/K(\chi)}\},\qquad
 \int_0^\tau\frac{d\chi}{a(\chi)}=T.
\end{equation}
The time constraint determines $\lambda>0$ uniquely. At $T=\tau/h$,
the only feasible rate is $h$ almost everywhere.
The uncapped optimum is feasible exactly when
$H(\tau)\le hT\sqrt{\min_{[0,\tau]}K}$; then
\begin{equation}\label{t3:eq-lr-arc}
 a(\chi)=\frac{H(\tau)}{T\sqrt{K(\chi)}},\qquad
 \frac{H(\chi(t))}{H(\tau)}=1-\frac tT,\qquad
 \min\mathcal V=\frac{H(\tau)^2}{T}.
\end{equation}
Every feasible profile satisfies
\begin{equation}\label{t3:eq-lr-m}
 \mathcal V[\eta]\ge H(\tau)^2/T.
\end{equation}
If $K$ is strictly decreasing and $T>\tau/h$, the optimal rate starts
at the cap if needed and then decreases.
\end{lemma}

\begin{proof}
For $\lambda>0$, minimizing $zK(x)+\lambda/z$ over $0<z\le h$ gives
\eqref{eq:kernel}. The time integral
$\int_0^\tau\max\{h^{-1},\sqrt{K(x)/\lambda}\}\,dx$ is continuous,
tends to infinity as $\lambda\downarrow0$, and decreases strictly
above its minimum $\tau/h$. Thus $T>\tau/h$ determines one multiplier.

Put $m_\lambda(x)=a(x)K(x)+\lambda/a(x)$. Pointwise minimality gives
$z^2K(x)+\lambda\ge z m_\lambda(x)$ for $0\le z\le h$, including
zero. Since $\dot\chi=-\eta$,
\[
 \mathcal V[\eta]\ge\int_0^\tau m_\lambda(x)\,dx-\lambda T.
\]
Equality holds for the profile defined by
$t=\int_{\chi(t)}^\tau dx/a(x)$ and forces $\eta=a(\chi)$ almost
everywhere, proving uniqueness even against zero-rate intervals.
At $T=\tau/h$, feasibility forces $\eta=h$ almost everywhere.

Without clipping, $\sqrt\lambda=H(\tau)/T$ and the maximum rate is
$H(\tau)/(T\sqrt{\min K})$, giving the cap criterion and
\eqref{t3:eq-lr-arc}. For every feasible rate,
$H(\tau)=\int_0^T\eta(t)\sqrt{K(\chi(t))}\,dt$;
Cauchy--Schwarz gives \eqref{t3:eq-lr-m}.
Finally, decreasing $K$ makes $a(\chi)$ increase in $\chi$.
As $\chi$ decreases during training, clipping can occur only initially.
\end{proof}

\subsubsection{From the continuous optimum to discrete SGD}
\label{c2:policy}

\paragraph{Setup and comparison class.}\label{new:arc-setup}
Fix an eligible width $M$, integers $T,B\ge1$, and a constant backward
mantissa $p_\bw\in\mathbb N_0\cup\{\infty\}$. Forward features remain
exact and the output is the last iterate. Let
$q=4^{-p_\bw}$, with $q=0$ at exact precision, and use the common cap
\begin{equation}\label{app:lr-conservative-cap}
 \eta_{\max}=\min\{B\rho_{\rm sch},c_0/C_g\},
\end{equation}
where $g_1(S)\le C_g$ uniformly on $\Omega_M$.
This cap is chosen before the sketch draw. Assume
\begin{equation}\label{c2:arc-data}
 I=[\tau_*^{\rm sch},\min\{M^a,\eta_{\max}T\}]\ne\varnothing,
 \qquad BTM[1+c(p_\bw)]\le F.
\end{equation}
The cost is fixed for these data and does not depend on the rate shape.
All joint feedback gains vanish with exact forward features; the rate
ceiling, clock window and both width conditions still apply.

Write $\beta=(b-1)/a$ and define the fixed-time bias/approximation
part and the combined noise response by
\begin{equation}\label{eq:lr-precision-kernel}
 B_M(\tau)=M^{1-b}+\tfrac12(1+2\tau)^{-\beta},\qquad
 K_q(x)=(1+2x)^{1/a-2}\left[1+\frac qM(1+2x)\right].
\end{equation}
Let $\mathcal C_q$ be the measurable rates $0\le\eta\le\eta_{\max}$
on $[0,T]$ with total rate $\tau\in I$. Use the remaining clock
$\chi$ from \eqref{app:eq:lr-objective} and define
\begin{equation}\label{c2:arcJ}
 J_q[\eta]=B_M(\tau)+\frac{\sigma^2}{2B}
                  \int_0^T\eta(t)^2K_q(\chi(t))\,dt.
\end{equation}
This is a finite-horizon relaxation of the discrete scalar risk.
One time unit represents one update; we take no diffusion or
vanishing-step limit. The discrete rates associated with a continuous
profile are its cell averages
$(D_0\eta)_i=\int_i^{i+1}\eta(t)\,dt$, $0\le i<T$.

For each event sketch $S$, let $\mathcal P_q(S)$ contain all deterministic
$T$-step learning-rate schedules with these fixed $M,B,p_\bw$, exact
forward features, and the full scheduled assumptions at $S$. The
pointwise spectral cap in this comparator class may be less restrictive
than the common cap \eqref{app:lr-conservative-cap}.

\paragraph{Learning-rate construction.}
For each $\tau\in I$ with $\tau<\eta_{\max}T$, let $a_\tau$ and
$\lambda(\tau)$ be the profile and unique multiplier in
\eqref{eq:kernel}, with $h=\eta_{\max}$ and the kernel
\eqref{eq:lr-precision-kernel}. At $\tau=\eta_{\max}T$, use the
constant-cap profile.
Choose the smallest minimizer $\tau_{\rm opt}$ on $I$ of
\[
 j_q(\tau)=B_M(\tau)+\frac{\sigma^2}{2B}
                         \int_0^\tau a_\tau(x)K_q(x)\,dx.
\]
Define $\chi_*(t)$ by
$t=\int_{\chi_*(t)}^{\tau_{\rm opt}}dx/a_{\tau_{\rm opt}}(x)$ and set
$\eta_*(t)=a_{\tau_{\rm opt}}(\chi_*(t))$.

The following theorem converts the continuous solution into an admissible
discrete SGD schedule by averaging over update intervals. Its prediction
risk is within a fixed factor of the best schedule in the stated
comparison class.
\begin{theorem}[Learning-rate construction and risk bound]
\label{c2:arc-main}
The profile $\eta_*$ minimizes $J_q$ over $\mathcal C_q$.
Its cell averages are admissible discrete SGD rates and satisfy
\begin{equation}\label{c2:arc-risk}
 \mathscr R_S^{\rm last}(D_0\eta_*)
 \le C_{\rm lr}\inf_{\pi\in\mathcal P_q(S)}
                         \mathscr R_S^{\rm last}(\pi),
 \qquad S\in\Omega_M.
\end{equation}
The factor $C_{\rm lr}$ depends only on the fixed model and 
constants, admissibility thresholds and positive noise variance,
uniformly in $M,T,B,p_\bw,F$ and $S$.
\end{theorem}

\begin{proof}
\emph{Continuous optimization.}
Lemma~\ref{prop:lr-full} minimizes $J_q$ at each fixed $\tau$.
The unique multiplier depends continuously on $\tau<\eta_{\max}T$.
At the cap endpoint,
$\int_0^\tau(\eta_{\max}-a_\tau(x))\,dx
\le\eta_{\max}^2(T-\tau/\eta_{\max})\to0$.
Since $K_q$ is bounded, $j_q$ converges to the constant-cap value. Thus it is
continuous on compact $I$, so $\tau_{\rm opt}$ exists and $\eta_*$
minimizes $J_q$.

\emph{Discrete comparison.}
\phantomsection\label{c2:cell-ports}
For a discrete rate array $\pi=(\eta_i)_{i<T}$, the scalar risk
\eqref{cl:Phi} is
\[
 \Phi(\pi)=B_M(\tau)+\frac{\sigma^2}{2B}
       \sum_{i<T}\eta_i^2K_q(\tau_T-\tau_{i+1}).
\]
Let $E_0\pi$ be the piecewise-constant extension of a discrete schedule
capped by $\eta_{\max}$. On a unit interval, a capped profile's remaining
clock varies by at most $\eta_{\max}$. The two negative powers in $K_q$
therefore give
$D^{-1}K_q(\chi(i+1))\le K_q(\chi(t))\le K_q(\chi(i+1))$,
where $D=(1+2\eta_{\max})^{2-1/a}$.
Integration and Cauchy--Schwarz on each interval yield
\begin{equation}\label{c2:cell-values}
 D^{-1}\Phi(\pi)\le J_q(E_0\pi)\le\Phi(\pi),\qquad
 \Phi(D_0\eta)\le D J_q[\eta].
\end{equation}
Cell averaging preserves cumulative rate at integer times and the rate
cap. With zero joint gains and \eqref{c2:arc-data}, this proves
admissibility and unchanged cost for $D_0\eta_*$.
Also $D\le D_{\max}=(1+2c_0/C_g)^{2-1/a}$.

\emph{Comparison with all admissible schedules.}
\phantomsection\label{c2:clock-scaling}
Put $d=\min\{1,c_g/C_g\}>0$. Since $g_1(S)\ge c_g$, multiplying any
pointwise-admissible schedule by $d$ enforces the common cap. From
$d(1+2x)\le1+2dx\le1+2x$,
\[
 B_M(d\tau)\le d^{-\beta}B_M(\tau),\qquad
 d^2K_q(dx)\le K_q(x).
\]
Hence $\Phi(d\pi)\le d^{-\beta}\Phi(\pi)$.
If $\tau\ge\tau_*^{\rm sch}/d$, the scaled schedule stays in the clock
window; its piecewise-constant extension belongs to $\mathcal C_q$ and
has $J_q\le d^{-\beta}\Phi(\pi)$ by \eqref{c2:cell-values}.

For smaller $\tau$, use instead $\eta_i=\tau_*^{\rm sch}/T$, feasible
by \eqref{c2:arc-data}. Its scalar risk is at most
\begin{equation}\label{c2:arc-onset-constants}
 E_*=\tfrac32+\sigma^2\rho_{\rm sch}\tau_*^{\rm sch}.
\end{equation}
Indeed $B_M\le3/2$, $K_q\le2$ and
$\sum_i\eta_i^2/B\le\rho_{\rm sch}\tau_*^{\rm sch}$.
The comparator's bias is at least
$\tfrac12(1+2\tau_*^{\rm sch}/d)^{-\beta}$.
Thus both cases give a profile $E\pi\in\mathcal C_q$ with
\begin{equation}\label{c2:arc-caploss}
 J_q[E\pi]\le C_{\rm cap}\Phi(\pi),\qquad
 C_{\rm cap}=\max\{d^{-\beta},
        2E_*(1+2\tau_*^{\rm sch}/d)^\beta\}.
\end{equation}
Continuous optimality and \eqref{c2:cell-values} now imply
\begin{equation}\label{c2:arc-ports}
 \Phi(D_0\eta_*)\le D_{\max}\inf_{\eta\in\mathcal C_q}J_q[\eta]
 \le D_{\max}C_{\rm cap}\Phi(\pi)
 \qquad(\pi\in\mathcal P_q(S)).
\end{equation}
The two-sided risk comparison in Theorem~\ref{cl:objective} proves
\eqref{c2:arc-risk} with
$C_{\rm lr}=Q_\Phi D_{\max}C_{\rm cap}$.
The construction uses only the fixed bounds and training specifications
before the random draws, so it is independent of the realized sketch.
\end{proof} 

\paragraph{Comparing exact and rounded objectives.}
For the kernel \eqref{eq:lr-precision-kernel}, write $H_q$ for the
function $H$ in Lemma~\ref{prop:lr-full}.
Its fixed-total-rate uncapped noise value is given by
\eqref{t3:eq-lr-arc}. For exact arithmetic,
\begin{equation}\label{app:eq:lr-exact-integral}
 H_0(\tau)=a\left[(1+2\tau)^{1/(2a)}-1\right].
\end{equation}
On the common continuous class $\mathcal C_q$, put
$\tau_{\max}=\max I$ and
$\varepsilon_q=q(1+2\tau_{\max})/M$.
The pointwise bounds $K_0\le K_q\le(1+\varepsilon_q)K_0$ and common
nonnegative bias/approximation terms imply
\begin{equation}\label{app:eq:lr-objective-comparison}
 \inf_{\mathcal C_q}J_0\le\inf_{\mathcal C_q}J_q
 \le(1+\varepsilon_q)\inf_{\mathcal C_q}J_0.
\end{equation}
Thus a uniformly small correction changes the best scalar-objective
value only slightly. The discrete actual-risk comparison still carries
Theorem~\ref{c2:arc-main}'s fixed factor. Separately, nonnegativity of
SR-induced error gives
$\inf\mathscr R_0^{\rm last}\le\inf\mathscr R_q^{\rm last}$ on the
same admissible discrete class. A common-budget comparison requires
both computations to fit, namely
$F\ge BTM\max\{2,1+c(p_\bw)\}$.

\subsubsection{Interpreting the explicit learning-rate schedules}
\label{app:section4-details}

Substituting \eqref{eq:lr-precision-kernel} into
\eqref{eq:kernel} gives the cumulative-learning-rate expression
\eqref{eq:lr-rounded-schedule}, with $\chi(t)=\tau-\tau_t$ and
$\lambda_\bw=\lambda(\tau)$ for $\tau<\eta_{\max}T$.
Setting $q=0$ gives \eqref{eq:lr-exact-schedule}, with its own
normalization $\lambda_0$. At $\tau=\eta_{\max}T$, both profiles
are constant at the cap and a multiplier need not be unique. For the
same prescribed $\tau$, the two
profiles have area $\int_0^T\eta(t)\,dt=\tau$; their multipliers are
chosen separately. The theorem then minimizes the risk over $\tau$,
so optimized exact and rounded runs need not have the same area.

Without clipping, the fixed-total-rate objective value follows by
multiplying the minimum in \eqref{t3:eq-lr-arc} by $\sigma^2/(2B)$
and adding $B_M(\tau)$. This is an identity for the continuous scalar
objective; the discrete actual-risk guarantee includes the factor in
Theorem~\ref{c2:arc-main}.

\paragraph{Exact arithmetic and the finite endpoint.}
For exact arithmetic, \eqref{app:eq:lr-exact-integral} evaluates the
integral. Writing $c=\sqrt{\lambda_0}$ on an uncapped segment and
using $d\chi/dt=-\eta$ gives
\[
 \frac{d}{dt}(1+2\chi)^{1/(2a)}=-\frac ca.
\]
The segment ending at $\chi(T)=0$ therefore has the explicit time form
\begin{equation}\label{app:eq:lr-exact-time}
 \eta_{\rm exact}(t)=c\left[1+\frac ca(T-t)\right]^{2a-1}.
\end{equation}
Where $c(T-t)/a\gg1$, it follows the power $(T-t)^{2a-1}$.
This is the decay exponent of \citet[Result~2]{bordelon2026optimal}
in the regime $1<b<a$, after translating their notation. The added
one in our scalar response gives a finite endpoint: the continuous
rate tends to $c$, rather than extrapolating the bulk power to zero.
Thus the cited power-law asymptotics and our regularized scalar
expression agree in their common bulk regime, rather than being
identical finite-horizon formulas.

\paragraph{The backward-rounding correction.}
The ratio $q(1+2\chi)/M$ is the backward-rounding response relative to
the exact-SGD noise response. It is largest at the beginning, when
$\chi=\tau$, and decreases as training progresses. This is a relative
response comparison, not a claim that the earliest updates contribute
the largest absolute error. If $q(1+2\tau)/M\ll1$, the correction is
uniformly small. If it is much larger than one on an uncapped interval,
then
\[
 \eta(t)\approx\sqrt{\lambda_\bw M/q}\,
                     [1+2\chi(t)]^{(a-1)/(2a)}.
\]
For large remaining cumulative learning rate, integrating
$d\chi/dt=-\eta$ makes $\chi^{(a+1)/(2a)}$ approximately affine in
training time. Thus the corresponding rate has decay power
$(a-1)/(a+1)$. The zero point extrapolated by this bulk approximation
need not coincide with the actual training endpoint. Both dominance
and absence of clipping are required.
The approximation stops applying as the response ratio approaches one;
the full expression \eqref{eq:lr-rounded-schedule} describes that
transition and the endpoint. The normalization and selected total rate
also change with precision, so this argument does not imply a
universally lower peak or an earlier decay start. 

\subsection{Compute allocation: fixed and growing precisions}
\label{t3:sub-compute}

The compute-optimal fixed-precision laws are proved in
Proposition~\ref{app:prop-fixed-bit-compute}, with variable-constant-rate
extensions in Propositions~\ref{r4:prop-variable-rho-average}
and~\ref{r4:prop-variable-rho-last}.
Here Proposition~\ref{app:prop-matched-compute} gives the growing-precision
construction, and Theorem~\ref{new:global-compute} treats exact precision.
We also explain the allocation comparisons used in
Section~\ref{sec:compute-allocation}.
The cost is \eqref{eq:declared-compute}; for constant schedules, use
its width--clock form from Section~\ref{sec:compute-allocation} and
$BT=\tau/\rho$.
Producing features $Sx$ by dense multiplication would separately cost
order $Md$ per example. The declared linear-width charge assumes this
work is supplied, precomputed or accounted for separately; including it
can change the compute exponents.

Fix an integer $B\ge1$, $0<\eta\le\min\{c_0/C_g,B\rho_{\rm sch}\}$
and a relative-risk tolerance $\vartheta>0$. Use uniform averaging and
$T_M=\lfloor M^a/\eta\rfloor$. Both mantissas are finite integers,
constant within a run but allowed to depend on $M$.
Write $f_M$ for the charge \eqref{eq:declared-compute}.

\begin{proposition}[Growing precisions for uniform averaging]
\label{app:prop-matched-compute}
For sufficiently large eligible $M$, an admissible run satisfying
$\Delta^{\rm avg}\le\vartheta\mathscr R_0^{\rm avg}$ necessarily obeys
\begin{equation}\label{app:eq:matched-bit-growth}
 p_\bw\ge\frac{b-1}{2}\log_2M-O(1),\qquad
 p_\fw+p_\bw\ge\frac{a-1}{2}\log_2M-O(1).
\end{equation}
For sufficiently large fixed $C_\fw,C_\bw$, the choice
\begin{equation}\label{app:eq:matched-attaining-bits}
 p_\bw=\left\lceil\frac{b-1}{2}\log_2M+C_\bw\right\rceil,
 \qquad p_\fw=\left\lceil\frac{a-b}{2}\log_2M+C_\fw\right\rceil
\end{equation}
satisfies this tolerance and the full constant-rate conditions.
Among admissible finite-bit runs satisfying the tolerance, the least charge
at width $M$ is $\asymp M^{a+1}\log M$.
This split has full risk $\asymp M^{1-b}$ and attains
\eqref{r4:eq:growing-finite-risk} under budget $F$.
\end{proposition}

\begin{proof}
For large $M$, $M^a/2\le\eta T_M\le M^a$, so
\eqref{eq:clean-law-avg} gives $\mathscr R_0^{\rm avg}\asymp M^{1-b}$.
The backward and joint lower comparisons in
Corollary~\ref{cor:bits-full} require
$q_\bw\lesssim M^{1-b}$ and $q_\fw q_\bw\lesssim M^{1-a}$.
Taking logarithms proves \eqref{app:eq:matched-bit-growth}.
For \eqref{app:eq:matched-attaining-bits},
$q_\bw\le4^{-C_\bw}M^{1-b}$,
$q_\fw\le4^{-C_\fw}M^{b-a}$ and
$q_\fw q_\bw\le4^{-(C_\fw+C_\bw)}M^{1-a}$.
Substitution in \eqref{eq:law-avg} gives the desired tolerance once
the offsets are large enough. Enlarging their sum also makes
$\rho q_\fw q_\bw\tau/M\le\rho4^{-(C_\fw+C_\bw)}\le\kappa_{\rm c}$.
Thus \eqref{s3:eq:constant-gain}, the stipulated rate and sufficiently
large widths verify full admissibility.
The necessary summed-bit bound gives the cost lower bound.
The displayed split matches it: its total mantissa length is
$O(\log M)$, the finite-precision cost is affine, and
$BT_M\asymp M^a/\rho$. Hence the least charge is
$\asymp M^{a+1}\log M$, with full risk $\asymp M^{1-b}$ for this split.
Choose $M$ at a sufficiently small fixed multiple of
$(F/\log F)^{1/(a+1)}$, rounded down to an integer, and use $T_M$
above. The charge is at most $F$ and the risk has order
\eqref{r4:eq:growing-finite-risk}.
\end{proof}

\subsubsection{The optimum when exact precision is allowed}
\label{new:compute-section}

\paragraph{Comparison classes.}\label{new:global-class}
A policy chooses integer $M,T,B\ge1$, deterministic nonnegative rates,
precisions in $\mathbb N_0\cup\{\infty\}$, and either the last iterate
or uniform average, before all random draws. Require both width
thresholds, a nonempty event $\Omega_M$ and cost at most $F$.
Let $\mathfrak P_F$ consist of these policies satisfying the full
scheduled class throughout $\Omega_M$.
For the lower bound only, enlarge this class by keeping the pointwise
cap $\eta_ug_1(S)\le c_0$ and dropping the rate ceiling, the lower and upper bounds on cumulative
learning rate, and feedback restrictions. We call this the capped
comparison class; no upper scaling law is needed on the enlarged class.

\paragraph{Integer construction.}
Let $\pi_F$ use uniform averaging, exact precision at both locations,
$B=1$ and constant learning rate $\eta_0$, with
\begin{equation}\label{new:global-attainer}
 \eta_0=\tfrac12\min\{c_0/C_g,\rho_{\rm sch}\},\qquad
 M_F=\left\lfloor(\eta_0F/4)^{1/(a+1)}\right\rfloor,\qquad
 T_F=\left\lfloor M_F^a/\eta_0\right\rfloor.
\end{equation}

The following theorem identifies the best possible risk exponent under
the declared compute model when exact precision is allowed.
Its lower bound covers both outputs, and a uniform-average construction
attains the order. Exact backward features with finite forward precision
can also attain the same order.
\begin{theorem}[Compute-optimal risk when exact precision is allowed]
\label{new:global-compute}\label{new:mixed-attainer}
Let $\gamma=(b-1)/(a+1)$. There are fixed $c,C>0$ such that, for
sufficiently large $F$, every policy in the capped comparison class has
conditional risk at least $cF^{-\gamma}$
on its sketch event, for either output. The construction
\eqref{new:global-attainer} belongs to $\mathfrak P_F$ and attains
the matching upper bound. Thus, for any fixed choice
$S_M\in\Omega_M$ at each feasible width,
\begin{equation}\label{new:frontier}
\csname ltx@label\endcsname{eq:application-global-compute}
 cF^{-\gamma}\le\inf_{\pi\in\mathfrak P_F}
 \mathscr R_{S_{M(\pi)}}(\pi)
 \le\mathscr R_{S_{M_F}}(\pi_F)\le CF^{-\gamma}.
\end{equation}
Fixed finite forward precision with exact backward features also
attains this order: with $W=1+c(p_\fw)$, replace $M_F$ by
$\lfloor(\eta_0F/(2W))^{1/(a+1)}\rfloor$ and use
$T_F=\lfloor M_F^a/\eta_0\rfloor$ at that width.
A sufficiently large fixed $p_\fw$ meets any fixed positive relative
SR-error tolerance.
\end{theorem}

The same order holds for every class containing this attainer and
contained in the capped comparison class. All comparisons are
conditional on the event at each width, without a joint event over widths.

\begin{proof}
Set $m_0=\min\{M,(1+\tau_T)^{1/a}\}$.
Unbiasedness, conditional Jensen inequality and \eqref{as:U1} imply
\begin{equation}\label{new:mean-risk}
 \mathscr R^{\iout}\ge\bar A+
 \tfrac12\operatorname{tr}(GP_{\iout}VP_{\iout})
 \ge\bar A+\tfrac{c_V}{2}\sum_jg_j^{b/a}d_{\iout,j}^2.
\end{equation}
For either output, the propagation coefficients in
\eqref{eq:schedule-filters} satisfy $d_{\iout,j}\ge\pi_j(T,0)$.
We claim
\begin{equation}\label{new:mean-lower-bound}
 \mathscr R^{\iout}\ge c_*m_0^{1-b}
\end{equation}
for fixed $c_*>0$. If $m_0\ge M/4$, use
$\bar A\ge c_AM^{1-b}$ from \eqref{as:approx-lower}.
Otherwise $m_0=(1+\tau_T)^{1/a}$. The block
$\lceil m_0\rceil\le j<2\lceil m_0\rceil$ lies below $M$ and contains
at least $m_0$ indices.
There $g_j\tau_T\le C_g$ and $g_j^{b/a}\gtrsim m_0^{-b}$.
The global product bound \eqref{sc:eq:productnew} makes their sum in
\eqref{new:mean-risk} at least a fixed multiple of $m_0^{1-b}$,
including when $\tau_T=0$.

Set $w_*=2\min\{1,p_0/(p_{\rm ref}+p_0)\}>0$.
The cost and rate cap give $f\ge w_*MT$ and
$\tau_T\le\bar\eta T$, where $\bar\eta=c_0/c_g$.
Consequently,
\[
 m_0^{a+1}\le M(1+\tau_T)
 \le(1+\bar\eta)MT\le(1+\bar\eta)F/w_*.
\]
Since $1-b<0$, \eqref{new:mean-lower-bound} proves the lower order.

For sufficiently large $F$, $M_F^a/2\le\eta_0T_F\le M_F^a$.
Both width thresholds, the minimum cumulative learning rate, the rate
cap and rate ceiling hold; the feedback gains vanish.
The sketch events are nonempty by \eqref{as:event-probability}, and
$2M_FT_F\le2M_F^{a+1}/\eta_0\le F/2$.
Equation~\eqref{eq:clean-law-avg} gives full averaged risk
$O(M_F^{1-b})=O(F^{-\gamma})$.

With finite forward and exact backward precision, the gains still
vanish and the modified width gives cost at most $F/2$.
Equation~\eqref{eq:law-avg} gives forward SR-induced error
$O(q_\fw M_F^{1-b})$, so a sufficiently small fixed $q_\fw$ meets
the prescribed tolerance. The pointwise lower bound and uniform
attainment prove \eqref{new:frontier} and the intermediate-class claim.
\end{proof} 

\paragraph{Unequal precisions and the width requirement.}
The sufficient checks \eqref{app:eq:accuracy-first-width-tests} give
\eqref{r4:eq:compute-width-requirements} by using its three
leading width demands.
\begin{equation}\label{r4:eq:compute-width-requirements}
 M\gtrsim\max\left\{\tau^{1/a},\
       \frac{\rho\sigma^2q_\bw}{\vartheta}\tau^{b/a},\
       \frac{\rho q_\fw q_\bw}{\kappa_{\rm c}}\tau\right\}.
\end{equation}
At fixed total mantissa length, the affine
finite-precision cost and $q_\fw q_\bw$ stay unchanged. Reallocating one
bit from forward to backward precision divides the direct backward
comparison term by four and multiplies the direct forward term by four.

The two inequalities in \eqref{app:eq:matched-bit-growth} do not require
both mantissas to grow separately. A sufficiently large fixed forward
mantissa controls its direct term; growing the backward mantissa as
$(a-1)\log_4M+O(1)$ then meets both the backward and summed-bit tests.
The split \eqref{app:eq:matched-attaining-bits} is another sufficient
choice. These matched, constant-rate constructions do not establish
the global optimality of the logarithmic factor for all-finite policies.
The exact-allowed benchmark and growing-finite construction use
different precision families, and normalized exact cost is a modeling
choice rather than a hardware claim.
Finite precision costs less than the unit-priced exact option only when
$c(p)<1$, equivalently $p<p_{\rm ref}$. Growing mantissas eventually
cost more than that option, so the construction establishes an
achievable finite-only rate rather than an asymptotic cost advantage
over exact arithmetic.

\paragraph{Reducing the learning-rate-to-batch-size ratio.}
The integer construction \eqref{r4:eq:constant-last-integer-attainer}
also gives averaged attainers when evaluated using
\eqref{eq:clean-law-avg} and \eqref{eq:law-avg}. With backward-only rounding take $r=b$:
$\rho\asymp M^{1-b}$, risk $\asymp M^{1-b}$ and
$f\asymp M^{a+b}$. With both scales positive take $r=a$:
$\rho\asymp M^{1-a}$, risk $\asymp M^{1-b}$ and
$f\asymp M^{2a}$. To check the latter, the averaged backward, forward
and joint terms are respectively $O(M^{1-a})$, $O(M^{2-a-b})$ and
$O(M^{1-b})$. The cited construction verifies all width, clock, rate and feedback
conditions.
These policies reduce the rate ratio at matched clock;
the fixed-rate attainers instead increase width for a given clock.
Since $BT=\tau/\rho$, neither is a free reduction of stochastic noise.
At fixed $M,\tau$ and precisions, halving $\rho$ doubles both the number
of sampled examples and the modeled compute cost.

\subsubsection{Prediction risk at a fixed compute budget}
\label{app:avg-compute-example}

Figure~\ref{fig:frontier}(c) fixes $d=1024$,
$\lambda_j=j^{-2}$, $\lambda_j\vartheta_j=j^{-3/2}$,
$\sigma^2=0.1$, $B=32$ and $\eta=0.2$, with zero initialization.
The input distribution and teacher distribution are common across widths.
We evaluate the prediction risk of the uniformly averaged iterate, including
approximation error, conditional on each sketch and averaged over the
teacher, training data and SR, using Proposition~\ref{arithmetic:prop:exact}
and \eqref{arithmetic:eq:two-time}.

Use \eqref{eq:declared-compute} with $p_0=9$, $p_{\rm ref}=7$
and budget $F=218{,}103{,}808$. For each precision pair, we test various widths.
At every width, minimizing risk over
positive integer durations satisfying $f\le F$ and $\eta T\le M^2$
selects the longest feasible duration for all three sketches.
Minimizing the mean risk across these widths, we can get the following table.
\begin{center}
\begin{tabular}{lrrr}
\toprule
$(p_\fw,p_\bw)$ & $M$ & $T$ & Mean $\mathscr R^{\rm avg}(T)$\\
\midrule
$(\infty,\infty)$ & $176$ & $19{,}362$ & $0.10999058$\\
$(1,\infty)$ & $176$ & $23{,}831$ & $0.10513179$\\
$(\infty,1)$ & $176$ & $23{,}831$ & $0.10513438$\\
$(1,1)$ & $192$ & $28{,}398$ & $0.09932464$\\
\bottomrule
\end{tabular}
\end{center}
Each entry in last column averages the three trials. All selected training configurations
meet the same budget.
Joint one-bit SR $(1,1)$ permits $9.1\%$ greater width and $46.7\%$ more updates,
reducing mean risk by $9.7\%$ relative to the selected exact-arithmetic
training configuration. This illustrates the principle of Section~\ref{sec:compute-allocation}:
select precision jointly with width and training duration.
The comparison underlying \eqref{eq:precision-time-balance} extends to
different widths by evaluating each exact-arithmetic risk and SR-induced
error at its own width and duration.  

\section{Experimental details}
\label{app:experiments}
This appendix gives experimental settings, risk estimators and uncertainty
calculations. Appendices~\ref{app:exp-panels}--\ref{app:exp-figure2} describe
Figures~\ref{fig:compute_law}--\ref{fig:rules};
Appendix~\ref{app:exp-compute} describes the compute and data-reuse comparisons
in Figures~\ref{fig:frontier}--\ref{fig:reuse-compute};
Appendix~\ref{app:exp-rounding-controls} gives the supplementary controls in
Figure~\ref{fig:diagnostics}; and Appendix~\ref{app:exp-nn} describes the
neural-network experiments in Figure~\ref{fig:nn}.

\subsection{Rounding location, timing, and precision requirements}
\label{app:exp-panels}\label{app:exp-where}\label{app:exp-common}
We use \eqref{eq:sgd} with $a=2$, $b=3/2$,
$\sigma^2=0.1$, $B=32$, $d=4M$, fresh data and independent feature
rounding unless stated otherwise. The exact-arithmetic evaluation is \eqref{eq:penalty}, including approximation error. Below we first show the experimental detials of Figure~\ref{fig:compute_law}.

\paragraph{Rounding location.}
Figure~\ref{fig:compute_law}(a) uses one-bit precision at one location
and exact computation at the other, evaluated at the recorded step
nearest $\tau_T=M^a$. We present SR-induced error of each output divided by its
own exact-arithmetic risk. Risks average over teachers and training
randomness for each sketch; curves show medians across three sketches
per width.

\paragraph{Precision timing.}
\label{app:exp-timing}
Figure~\ref{fig:compute_law}(b) fixes $M=128$ and $T=31{,}715$,
switching between one and three bits after update $15{,}857$.
One bit followed by three uses $63{,}431$ mantissa bit-updates; the
reverse order uses $63{,}429$, so average precisions are nearly equal. Forward-rounding
tests keep backward precision at four bits; backward-rounding tests
keep forward features exact. We present the SR-induced error
with three bits early divided by that with three bits late. 

\paragraph{Precision requirements.}
Figure~\ref{fig:compute_law}(c) uses $M=16,32,64,128,256,512$ and
$p_\fw=p_\bw\in\{1,2,3,4\}$. At each width, precisions share the same
learning rate and training time. At the step nearest
$\tau_T=M^2$, we select the smallest tested precision satisfying accuracy tolerance in \eqref{eq:tolerance} for the last iterate with $\vartheta=1\%$. Training duration grows with
width. The table also gives requirements for two other tolerances;
``None'' means no tested precision succeeds.

\begin{center}
\begin{tabular}{lrrrrrr}
\toprule
Prediction risks & $M=16$ & $32$ & $64$ & $128$ & $256$ & $512$\\
\midrule
$1.10\times$ exact risk & $1$ & $1$ & $1$ & $1$ & $1$ & $2$\\
$1.01\times$ exact risk & $1$ & $2$ & $2$ & $2$ & $3$ & $3$\\
$1.001\times$ exact risk & $3$ & $3$ & $4$ & $4$ & None & None\\
\bottomrule
\end{tabular}
\end{center}

\subsection{Joint precision allocation and learning-rate design}
\label{app:exp-figure2}

Figure~\ref{fig:rules}(a) fixes $M=128$, $d=512$, $B=32$ and
$\eta=0.51660418$. Choose step $T=31{,}424$, nearest
$\tau_T=M^2$, we compare
$(p_\fw,p_\bw)=(1,4),(2,3),(3,2),(4,1)$ at equal compute under
\eqref{eq:declared-compute}. We show SR-induced error of tested precisions divided
by its exact-arithmetic risk for each case. 

Figure~\ref{fig:rules}(b,c) uses $M=32$, $d=128$, $B=4$, $T=4096$,
last-iterate output, exact forward features and
$p_\bw\in\{1,2,4,\infty\}$. The learning-rate
upper bound $\eta_{\max}=0.2$ is chosen before drawing sketches.
We average \eqref{eq:lr-exact-schedule} and
\eqref{eq:lr-rounded-schedule} over unit update intervals.
Figure~\ref{fig:rules}(b) shows the one-bit precision-dependent schedule,
the exact-SGD schedule design, and the constant learning-rate baseline
$\eta_t=\tau_T/T=512/4096=1/8$, all at $\tau_T=512$.
The exact-SGD design is constructed using \eqref{eq:lr-exact-schedule};
its name describes the schedule, rather than the training precision.

Figure~\ref{fig:rules}(c) compares the two learning-rate schedules
illustrated in Figure~\ref{fig:rules}(b), with both schedules evaluated at
the same backward precision. A regression problem consists of one randomly
drawn sketch and teacher, which remain fixed during repeated training.
For each schedule and precision, we test
$\tau_T\in\{16,32,64,128,256,512\}$ on four independent problems,
using 16 training runs per problem for each candidate. We choose the total
learning rate with the smallest mean last-iterate prediction risk across
these problems and runs.

For the plotted comparison, we generate eight new independent problems
that were not used for tuning. For each problem, we train with each selected
schedule 64 times. Each comparison matches one run of each schedule, using
the same batches, label noises and random numbers for rounding. Each
repetition uses new training data and random draws. The randomness for data,
label noise and rounding is mutually independent.
For each comparison, we subtract the final prediction risk of the exact-SGD
schedule from that of the precision-dependent schedule, both measured by
\eqref{eq:risk}. We average the 64 differences to obtain one result for each
problem. The plotted point is the average of these eight results. We compute
the standard error as their sample standard deviation divided by $\sqrt{8}$;
bars extend twice this amount above and below the point. Positive values
favor the exact-SGD schedule.
These experiments lie outside the conservative learning-rate class of
Assumption~\ref{as:assumption-schedules}, and do not directly test
Theorem~\ref{c2:arc-main} or establish a global optimum.

\subsection{Compute-budget comparisons}
We describe the width and duration comparisons in Figure~\ref{fig:frontier}
and the fixed-width data-reuse comparisons in Figure~\ref{fig:reuse-compute}.
All compute comparisons use \eqref{eq:declared-compute} with $p_0=9$
and $p_{\rm ref}=7$.

\paragraph{Width and training duration.}
\label{app:exp-compute}
Figure~\ref{fig:frontier}(a,b) uses $d=4M$, $\eta=0.5/g_1(S)$ and training to approximately $\tau_T=8M^2$.
Colored curves show uniformly averaged prediction risk as training
continues. At each budget, the black curve selects the smallest risk
among tested widths and recorded durations within that budget.
The red dashed curve in Figure~\ref{fig:frontier}(a) fits the
exact-arithmetic compute scaling, $F^{-1/6}$, in Table~\ref{tab:compute-precision}, adjusting only its multiplicative
coefficient. Figure~\ref{fig:frontier}(b) reuses that coefficient and
accounts for the update cost falling from $2BM$ to $5BM/4$, while
ignoring SR-induced error. Both curves are presented $6\%$ lower for
readability.  

Figure~\ref{fig:frontier}(c) uses the fixed task and three sketches of
Appendix~\ref{app:avg-compute-example}. Widths and durations are selected
using their mean prediction risk. At some small widths, $\eta T\le M^2$
ends training before the budget is used. Forward-only and backward-only
one-bit SR give nearly coincident curves in this comparison.

\paragraph{Fresh and reused data at fixed width.}
\label{app:exp-fixed-budget}
Figure~\ref{fig:reuse-compute} fixes $M=256$, $d=1024$, $B=32$ and
$\eta=0.2$. We compare exact arithmetic with independent SR for forward and
backward features at $p_\fw=p_\bw=2$; all other computations are exact.
Each of eight independent sketch--teacher--data tasks has 16 SR replicas
sharing its data and access order. Rounding streams are matched across
policies and data conditions within a task. Training uses fresh online batches or a fixed dataset of 2,048 or 16,384 examples. For the reused-data setting, after the first pass, existing examples are reshuffled without regenerating inputs or labels. All plotted reused-data runs complete the first pass. Thus fresh and reused training differ in data availability as well as
repeated access.

All risks use the uniform average outputs. Under \eqref{eq:declared-compute}, each exact-arithmetic update costs
$16384$ compute, and each update with two-bit SR at both feature locations
costs $11264$ compute. We use $F$ as a total compute ceiling. Prediction risks
were saved only at selected update counts, so for each method we take
the latest saved evaluation whose total cost is at most $F$. The table gives the selected update counts and
the percentage of $F$ actually used:
\begin{center}
\begin{tabular}{rrrrr}
\toprule
Budget $F$ & Exact updates & Exact use & SR updates & SR use\\
\midrule
$2^{24}$ & $1024$ & $100\%$ & $1361$ & $91.38\%$\\
$2^{25}$ & $2048$ & $100\%$ & $2871$ & $96.38\%$\\
$2^{26}$ & $4096$ & $100\%$ & $4722$ & $79.26\%$\\
$2^{27}$ & $8192$ & $100\%$ & $9961$ & $83.60\%$\\
\bottomrule
\end{tabular}
\end{center}
For example, at $F=2^{27}$, exact arithmetic training uses $8192$ updates and
exhausts the budget. SR can afford $11{,}915$ updates, but its latest
saved evaluation within the budget is at $9961$ updates, using only
$83.60\%$. The unused budget reflects the spacing of saved evaluations.
Figure~\ref{fig:reuse-compute} therefore compares available checkpoints
under a the same compute budget, with unequal actual compute use.

\subsection{Additional rounding locations and precision allocation}
\label{app:exp-rounding-controls}

Figure~\ref{fig:diagnostics} provides additional rounding-location controls
and an integer-precision allocation comparison, with protocols given below.
\begin{figure}[t]\centering
\includegraphics[width=.325\linewidth]{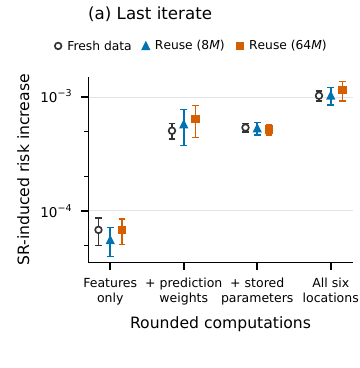}\hfill
\includegraphics[width=.325\linewidth]{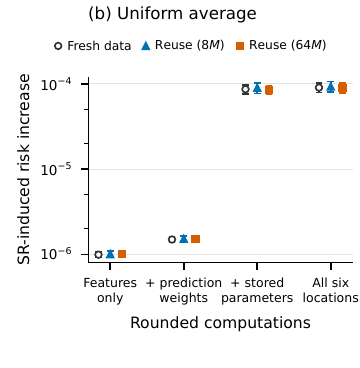}\hfill
\includegraphics[width=.325\linewidth]{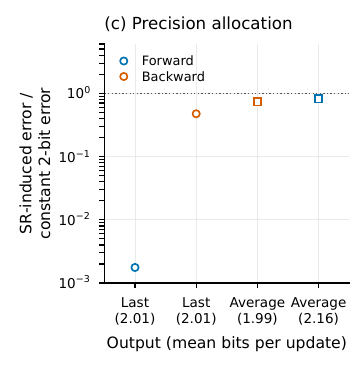}
\caption{\small Supplementary precision comparisons.
Figure~\ref{fig:diagnostics}(a,b) shows SR-induced risk increases for
last-iterate and uniformly averaged outputs at fixed duration, with
fresh or reused data. The four categories use feature rounding alone,
add temporary prediction-weight rounding, add stored-parameter rounding,
or round at all six locations, as defined below. The vertical axes use
different logarithmic ranges; bars show twice the standard error across
eight task means. Figure~\ref{fig:diagnostics}(c) compares four integer
precision schedules with constant two-bit precision at the same feature
location. Mean bit counts are shown; the realized budgets differ.}
\label{fig:diagnostics}
\end{figure}

\paragraph{Rounding locations and update order.}
Figure~\ref{fig:diagnostics}(a,b) uses the tasks and three data conditions
of Figure~\ref{fig:reuse-compute}, with $M=256$, $d=1024$, $B=32$,
$\eta_u=0.2$ and $T=16384$. Training uses fresh batches or fixed
datasets of $8M=2048$ or $64M=16384$ samples, reshuffled for each
subsequent pass with inputs and labels retained. Each of eight independent
sketch--teacher--data tasks has 16 SR replicas.

Recall $v_u$, $a_{u\ell}$, $b_{u\ell}$ and $Q_p$ from
\eqref{eq:copies}--\eqref{eq:sgd}, with $v_0=0$ and
$Q_\infty(x)=x$. Four additional precisions
$p_{\rm pred},p_{\rm res},p_{\rm grad},p_{\rm store}$ specify rounding
of prediction weights, residuals, minibatch gradients and stored
parameters, respectively. All five training consider the following quantizations:
\begin{equation}
\label{eq:exp-rounding-update}
\begin{aligned}
 a_{u\ell}&=Q_{p_{\fw,u}}(Sx_{u\ell}),&
 b_{u\ell}&=Q_{p_{\bw,u}}(Sx_{u\ell}),\\
 \widetilde v_u&=Q_{p_{\rm pred}}(v_u),&
 r_{u\ell}&=Q_{p_{\rm res}}(a_{u\ell}^{\top}\widetilde v_u-y_{u\ell}),\\
 \widehat g_u&=Q_{p_{\rm grad}}\!\left(\frac1B\sum_{\ell=1}^B b_{u\ell}r_{u\ell}\right),&
 v_{u+1}&=Q_{p_{\rm store}}(v_u-\eta_u\widehat g_u).
\end{aligned}
\end{equation}
The table lists the precision tuple
$(p_{\fw,u},p_{\bw,u},p_{\rm pred},p_{\rm res},p_{\rm grad},p_{\rm store})$ in each case:
\begin{center}
\begin{tabular}{ll}
\toprule
Training case & Precision tuple\\
\midrule
Exact-arithmetic reference & $(\infty,\infty,\infty,\infty,\infty,\infty)$\\
Features only & $(2,2,\infty,\infty,\infty,\infty)$\\
Features + prediction weights & $(2,2,2,\infty,\infty,\infty)$\\
Features + stored parameters & $(2,2,\infty,\infty,\infty,7)$\\
All six locations & $(2,2,2,2,2,7)$\\
\bottomrule
\end{tabular}
\end{center}
The temporary copy $\widetilde v_u$ is formed once per update, shared
by the $B$ predictions, and discarded afterward; the update starts
from $v_u$. Storage rounding is applied after the update, and its
result is carried to the next step. Distinct rounding locations
use independent draws, renewed each update; draws at the same location
are matched across policies and data conditions. Bit counts are
mantissa precisions as in Section~\ref{sec:setup}. Remaining arithmetic,
accumulation and evaluation use float64. No compute savings are assigned
to the four additional rounding locations.

\paragraph{Outputs and plotted risk increases.}
For every case we use $v_T^{\rm last}$ and $v_T^{\rm avg}$ from
\eqref{eq:output-weights}, without rounding the output average.
Let $v_T^{\iout,(0)}$ denote the exact-arithmetic output on the same
data and access order. For each rounded output, the plotted estimate
averages
$
 \frac12\bigl\|v_T^{\iout}-v_T^{\iout,(0)}\bigr\|_G^2,
$ for $\iout\in\{\mathrm{last},\mathrm{avg}\}$,
first over the 16 replicas within each task and then over the eight
tasks. This quantity estimates the rounding-induced risk increase in
\eqref{eq:penalty}. Risks are evaluated by \eqref{eq:risk}. Smaller SR-induced
error across data conditions need not imply smaller full risk. 

\paragraph{Integer precision allocation.}
\label{app:exp-precision}
Figure~\ref{fig:diagnostics}(c) illustrates the precision-allocation idea
in \eqref{eq:lagrange}. With $M=64$, $B=32$, $\eta=0.5/g_1$ and
$T=8869$, teacher-averaged second-moment recursions for one fixed sketch
compare time-varying integer precision with constant two-bit precision
at either feature location, keeping the other exact, for both outputs.
The plotted ratios compare SR-induced risk increases; realized bit
budgets differ, so these comparisons do not establish equal-budget
optimality.

\subsection{Neural-network rounding locations and evaluation}
\label{app:exp-nn}
Figure~\ref{fig:nn} reports final transformer validation loss and MLP
test loss, both cross-entropy using natural logarithms. Training without
SR disables only the added experimental stochastic-rounding operations.
Evaluation means making predictions and calculating validation or test loss
after training, with parameters fixed and no further updates. It disables
the added SR operations unless stated otherwise; the underlying arithmetic
described below remains unchanged. Bit counts refer
to mantissa precision, with unrestricted exponents.

\paragraph{Models and training.}
The transformer has six layers, six heads, width 384, context 256 and
30,041,856 parameters. It trains on approximately 200 million FineWeb-Edu
tokens for 12,207 updates, with batches of $64\times256$ tokens.
AdamW uses peak learning rate $6\times10^{-4}$, coefficients $(0.9,0.95)$,
weight decay $0.1$, gradient clipping at one and 100 warmup updates.
We use three seeds and 40 fixed validation batches. All transformer runs
use BF16/FP32 mixed-precision arithmetic, with BF16 autocast enabled
during both training and evaluation, including when added SR is disabled.
The CIFAR-10 MLP has architecture $3072$--$1024$--$10$ with ReLU and
FP32 computation. It trains for 25 epochs (9750 updates), batch size 128,
using SGD with momentum $0.9$, weight decay $5\times10^{-4}$, peak learning
rate $0.01$ and 200 warmup updates. We use 16 seeds and the full CIFAR-10
test set. Both models decrease the learning rate linearly during the
final $20\%$ of training and report last-iterate performance.

\paragraph{Rounding locations.}
Figure~\ref{fig:nn}(a) compares independent forward/backward activation
rounding at two bits, incoming weight-gradient signals at two bits,
temporary forward weights at two bits, and parameter storage at seven
bits. The combined case uses two bits at all four operation locations
and seven-bit storage. Optimizer states and the input-gradient product
use their original arithmetic.
Operation rounding affects linear layers except the tied language-model
head. Storage rounding affects all parameters, including embeddings,
normalization parameters and biases; disabling SR at evaluation does
not undo earlier storage rounding. These cases therefore differ in
parameter coverage as well as rounding location.

\begin{table}[htbp]\centering
\caption{Final transformer validation loss with added SR disabled at evaluation.
The table covers Figure~\ref{fig:nn}(a,b) and a three-bit storage comparison.}
\label{tab:nn-transformer}
{\small\setlength{\tabcolsep}{4pt}
\begin{tabular}{lr}
\toprule
Training configuration & Validation loss\\
\midrule
No added SR & $4.1846\pm0.0049$\\
Independent activations, 1 bit & $4.2452\pm0.0034$\\
Independent activations, 2 bits & $4.2016\pm0.0082$\\
Shared activations, 1 bit & $4.2238\pm0.0019$\\
Shared activations, 2 bits & $4.1976\pm0.0036$\\
Backpropagated gradients, 2 bits & $4.1852\pm0.0051$\\
Forward weights, 2 bits & $4.1907\pm0.0009$\\
Stored parameters, 7 bits & $4.1910\pm0.0036$\\
Combined: 2-bit computations, 7-bit storage & $4.2188\pm0.0029$\\
Stored parameters, 3 bits & $4.3555\pm0.0019$\\
\bottomrule
\end{tabular}}
\end{table}

\paragraph{Shared activation rounding.}
Figure~\ref{fig:nn}(b) uses one or two bits at both activation
computations. Independent SR rounds the activation separately for
forward and weight-gradient products; shared SR reuses the forward-rounded
activation. No other experimental SR is active. Seeds match initialization,
data order and evaluation examples. Rounding draws differ because sharing
uses fewer casts; the data stream is independent of those draws.
Each plotted loss difference subtracts same-seed training without SR.
The direct shared-minus-independent transformer differences are
$-0.02140\pm0.00311$ at one bit and $-0.00400\pm0.00502$ at two bits,
with twice the paired standard error. The two-bit comparison does not
clearly distinguish the methods. Sharing changes dependence between
rounding errors and can change the mean update, beyond
Theorem~\ref{thm:scheduled-main}.

\Needspace{22\baselineskip}
\paragraph{Training precision and rounding at test time.}
Figure~\ref{fig:nn}(c) uses one-, two- or three-bit forward activation
rounding during training. Each trained MLP makes predictions on the same
test examples twice, keeping its parameters fixed and making no updates:
with forward-activation SR disabled, and with SR enabled at its training
precision. Both tests use the MLP's underlying FP32 arithmetic.
The dashed line uses training and evaluation without SR. Accuracy in
Table~\ref{tab:nn-mlp} uses evaluation without SR.

\begin{table}[htbp]\centering
\caption{Final MLP test performance for Figure~\ref{fig:nn}(b,c).
Entries are means $\pm$ two standard errors across 16 seeds.
The two loss columns evaluate the same final parameters without or with SR;
accuracy uses evaluation without SR.}
\label{tab:nn-mlp}
{\small\setlength{\tabcolsep}{3pt}
\begin{tabular}{lrrr}
\toprule
Training configuration & \shortstack{Test loss\\SR off} & \shortstack{Test loss\\SR on} & \shortstack{Accuracy (\%)\\SR off}\\
\midrule
No added SR & $1.6198\pm0.0032$ & $1.6198\pm0.0032$ & $56.409\pm0.147$\\
Forward activations, 1 bit & $1.6237\pm0.0041$ & $1.6721\pm0.0056$ & $56.541\pm0.109$\\
Forward activations, 2 bits & $1.6210\pm0.0027$ & $1.6338\pm0.0030$ & $56.458\pm0.112$\\
Forward activations, 3 bits & $1.6212\pm0.0039$ & $1.6242\pm0.0039$ & $56.381\pm0.124$\\
Independent activations, 1 bit & $1.6208\pm0.0034$ & $1.6706\pm0.0041$ & $56.501\pm0.120$\\
Independent activations, 2 bits & $1.6223\pm0.0032$ & $1.6356\pm0.0035$ & $56.403\pm0.131$\\
Shared activations, 1 bit & $1.5581\pm0.0041$ & $1.6028\pm0.0051$ & $56.905\pm0.143$\\
Shared activations, 2 bits & $1.6081\pm0.0045$ & $1.6205\pm0.0048$ & $56.369\pm0.136$\\
\bottomrule
\end{tabular}}
\end{table}

\end{document}